\documentclass[final,onefignum,onetabnum]{siamonline190516}

\usepackage{graphicx} 
\usepackage{multirow} 
\usepackage{booktabs} 
\usepackage{array}
\usepackage{mathtools}
\usepackage{mathrsfs} 
\usepackage{color}
\usepackage{amssymb}
\usepackage{enumitem}
\usepackage{braket,amsfonts}

\usepackage{array}

\usepackage[caption=false]{subfig}
\usepackage{pgfplots}

\newsiamthm{claim}{Claim}
\newsiamremark{remark}{Remark}
\newsiamremark{hypothesis}{Hypothesis}
\crefname{hypothesis}{Hypothesis}{Hypotheses}

\usepackage{algorithmic}

\usepackage{graphicx,epstopdf}

\Crefname{ALC@unique}{Line}{Lines}

\usepackage{amsopn}

\usepackage{xspace}
\usepackage{bold-extra}
\usepackage[most]{tcolorbox}

\colorlet{texcscolor}{blue!50!black}
\colorlet{texemcolor}{red!70!black}
\colorlet{texpreamble}{red!70!black}
\colorlet{codebackground}{black!25!white!25}

\lstdefinestyle{siamlatex}{%
  style=tcblatex,
  texcsstyle=*\color{texcscolor},
  texcsstyle=[2]\color{texemcolor},
  keywordstyle=[2]\color{texemcolor},
  moretexcs={cref,Cref,maketitle,mathcal,text,headers,email,url},
}

\tcbset{%
  colframe=black!75!white!75,
  coltitle=white,
  colback=codebackground, % bottom/left side
  colbacklower=white, % top/right side
  fonttitle=\bfseries,
  arc=0pt,outer arc=0pt,
  top=1pt,bottom=1pt,left=1mm,right=1mm,middle=1mm,boxsep=1mm,
  leftrule=0.3mm,rightrule=0.3mm,toprule=0.3mm,bottomrule=0.3mm,
  listing options={style=siamlatex}
}

\newtcblisting[use counter=example]{example}[2][]{%
  title={Example~\thetcbcounter: #2},#1}

\newtcbinputlisting[use counter=example]{\examplefile}[3][]{%
  title={Example~\thetcbcounter: #2},listing file={#3},#1}

\DeclareTotalTCBox{\code}{ v O{} }
{ %fontupper=\ttfamily\color{texemcolor},
  fontupper=\ttfamily\color{black},
  nobeforeafter,
  tcbox raise base,
  colback=codebackground,colframe=white,
  top=0pt,bottom=0pt,left=0mm,right=0mm,
  leftrule=0pt,rightrule=0pt,toprule=0mm,bottomrule=0mm,
  boxsep=0.5mm,
  #2}{#1}

\patchcmd\newpage{\vfil}{}{}{}
\begin{tcbverbatimwrite}{tmp_\jobname_header.tex}
\title{Optimization Models and Interpretations for Three Types of Adversarial Perturbations against Support Vector Machines\thanks{Submitted to the editors December 18, 2021.
\funding{The work of Qingna Li is supported  by the National Natural Science Foundation of China (NSFC) grants 12071032.}}}
 
\author{Wen Su\thanks{School of Mathematics and Statistics, Beijing Institute of Technology, Beijing, 100081, China (\email{suwen019@163.com}).}
\and Qingna Li\thanks{School of Mathematics and Statistics, Beijing Key Laboratory on MCAACI/Key Laboratory of Mathematical Theory and Computation in Information Security, Beijing Institute of Technology, Beijing, 100081, China (\email{qnl@bit.edu.cn}).}
\and Chunfeng Cui\thanks{Corresponding author. LMIB of the Ministry of Education, School of Mathematical Sciences, Beihang University, Beijing, 100191, China (\email{chunfengcui@buaa.edu.cn}).}}

\headers{Optimization Models and Interpretations for Three Types of Adversarial Perturbations against Support Vector Machines}{W. Su, Q. Li and C. Cui}
\end{tcbverbatimwrite}
\input{tmp_\jobname_header.tex}

\ifpdf
\hypersetup{ pdftitle={Three Types of Adversarial Perturbations against Support Vector Machines} }
\fi

\begin{document}
\maketitle

%% ------------------------------------------------------------------
%% ABSTRACT
%% ------------------------------------------------------------------
\begin{tcbverbatimwrite}{tmp_\jobname_abstract.tex}
\begin{abstract}
 Adversarial perturbations have drawn great attentions in various deep neural networks. Most of them are computed by iterations and cannot be interpreted very well. In contrast, little attentions are paid to basic machine learning models such as support vector machines. In this paper, we investigate the optimization models and the  interpretations for three types of adversarial perturbations against  support vector machines, including sample-adversarial perturbations (sAP), class-universal adversarial perturbations (cuAP) as well as universal adversarial perturbations (uAP). For linear binary/multi classification support vector machines (SVMs),  we derive the explicit solutions for sAP, cuAP and uAP (binary case), and approximate solution for uAP of multi\mbox{-}classification. We also obtain the upper bound of fooling rate for uAP. Such results not only increase the interpretability of the three adversarial perturbations, but also provide great convenience in computation since iterative process can be avoided. Numerical results show that our method is fast and effective in calculating three types of adversarial perturbations. 
\end{abstract}


\end{tcbverbatimwrite}
\input{tmp_\jobname_abstract.tex}
%% ------------------------------------------------------------------
%% END HEADER
%% ------------------------------------------------------------------

\section{Introduction}
\label{sec:intro}

Machine learning has proved to be a powerful tool in analyzing data from different applications, such as computer vision \cite{ref1},  natural language processing \cite{ref2}, speech recognition \cite{ref3}, recommendation system \cite{ref66,ref5,ref4}, cyber security \cite{ref7,ref8}, clustering \cite{ref64,ref65} and so on. 
However, some hackers can analyze the loopholes in machine learning to launch attacks on intelligent applications.  The attackers can fool the trained machine learning system by designing input data.
%However, some hackers can often make malicious inputs by analyzing the loopholes in machine learning technology to launch attacks on intelligent applications. The attackers can fool the trained machine learning system by designing input data, making the machine learning system produce obvious errors. 
For example, in a spam detection system, attackers can confuse the machine learning detection system by adding unrelated tokens to their emails \cite{ref9}.  Correspondingly, some researchers also modify neural network structure to make it resist attacks \cite{ref60}.
The security of machine learning systems is still an important research topic. Meanwhile, how to propose effective methods to calculate adversarial perturbations and how to make interpretation are getting more and more attentions. Due to the importance of  adversarial perturbations, many researchers started to investigate the  small pertubations to a machine learning system. It can be divided into three types: perturbations for a given sample (sAP), universal adversarial perturbations (uAP) and class-universal adversarial perturbations (cuAP). Below we discuss them one by one.

sAP is one of the most important and widely used types of adversarial perturbations. It  has been studied in deep neural networks since 2014,  and  there are many methods to calculate sAP. Szegedy et al. \cite{ref10}  firstly discovered a surprising weakness of neural networks in the background of image classification, that is, neural networks are easily attacked by very small adversarial perturbations. Different from the previous input data designed by the attacker, these adversarial instances are almost indistinguishable from natural data (which means in human observation, there is no difference between adversarial instances and undisturbed input), and are misclassified by the neural  network. This leads to researchers' interest in studying sAP. Szegedy et al. \cite{ref10}   believed that the highly nonlinear nature of neural networks led to the existence of adversarial instances. However,  Goodfellow et al. \cite{ref11} put forward the opposite view.  They believed that the linear behavior of neural networks in high-dimensional space is the real reason for the existence of adversarial instances. 
The method (namely Fast Gradient Sign Method (FGSM))  proposed in \cite{ref11}  is designed based on gradient and has also become one of the mainstream  method. 
\cite{ref30} and \cite{ref29} further improved  FGSM, which normalized the gradient formula calculated by FGSM using $l_2\mbox{-}$norm and $l_{\infty}\mbox{-}$norm respectively. Different from the  above methods, DeepFool is an attack method based on linearization and separating hyperplane \cite{ref13}, which initializes with the clean image that is assumed to reside in a region confined by the decision boundaries of the classifier. At each iteration, DeepFool perturbs the image by a small vector that  takes the resulting image to the linearizing boundaries of the region within which the image resides. 

As for uAP, it is another popular type of adversarial perturbations. uAP can be generated in advance and then applied dynamically during the attack. Moosavi-Dezfooli et al. \cite{ref15} proposed a single small image perturbation that fools a state-of-the-art deep neural network classifier on all natural images. Such perturbations are dubbed universal, as they are imageagnostic. The fooling rate is the most widely adopted metric for evaluating the efficacy of uAP. Specifically, the fooling rate is defined as the percentage of samples whose prediction changes after uAP is applied, i.e., $ \dfrac{\text{The number of } \hat{k}(x+r)\neq \hat{k}(x), x \in A}{\text{The number of data in } A } $, where $A$ is a given dataset and $ \hat{k} $ is the known classifier.
%$P_{x \sim \mu}(\hat{k}(x+r)\neq \hat{k}(x))$, where $ \hat{k} $ is the classifier, $ x $ is the sample satisfying the distribution $  \mu $ and $ r $ is the perturbation vector independent of $ x $. 
The existence of uAP reveals the  important geometric correlations among the high-dimensional decision boundary of classifiers. Since \cite{ref15},  many methods have been introduced by researchers to generate uAP, both data-driven and data-independent. Miyato et al. \cite{ref50}
proposed a  data-driven method. They generated uAP by using fooling and diversity objectives along with a generative model.
Cui et al. \cite{ref20} investigated generating uAP by the active-subspace. Miyato et al. \cite{ref52} proposed a data-independent method to generate uAP by fooling the features learned at multiple layers of the network. 
Their method didn't use any information about the training data distribution of the classifier.
% Their approach of crafting perturbations did not utilize any knowledge about the data distribution under which the target model has been trained.  

In terms of cuAP, it is a modified uAP based on different applications. 
It can attack the data in a class-discriminative way, which is more stealthy.  As far as we known, less work has been done on cuAP, compared with sAP and uAP. 
Zhang et al. \cite{ref53} noticed that uAP might cause obvious misconduct, and it might make the users suspicious. Thus, they proposed  CD-UAP (class discriminative universal adversarial perturbation). CD-UAP only attacks data from a chosen group of classes, while having limited the impact on the remaining classes. Since the classifier will only misbehave when the specific data from a targeted class is encountered, cuAP will not raise obvious suspicion. Ben  et al. \cite{ref54} extended  cuAP to a targeted version, which means they  made a perturbation to fool data of the particular class toward the targeted class they pre-defined.

Recently, some researchers turn their attentions to  perturbations against  SVM. Fawzi et al. \cite{ref12} initiated the research on sAP against SVM, which provides more insights for  the relationship between robustness and design parameters. Langenberg et al. \cite{ref14}  analyzed and quantified the numerical experiments of sAP on SVM, showing that the robustness of SVM  is significantly affected by parameters which change the linearity of the models. However, there is no  mathematical derivation of sAP on SVM. 

%Inspired by \cite{ref12} and \cite{ref14} (2019), we hope to further understand the nature of adversarial perturbations in neural networks by analyzing sAP, cuAP and uAP of SVM. Meanwhile, inspired by the method of generating perturbations in DeepFool (\cite{ref13}, 2016), for SVM with known hyperplane, we  solve the optimal adversarial perturbations that perturbs the data/dataset to the hyperplane. In the linear classification problem, if we want to change the class of sample $ x $, the smallest perturbation is to move $ x $ toward the hyperplane, then the distance from the sample to the hyperplane is the least cost choice, which is the strategy to generate the simplest accurate sAP. we can further generate  the optimal cuAP and uAP based on it.
In this paper, we systematically study the optimization models and the  interpretations for three types of  adversarial perturbations (sAP, cuAP and uAP) against classification models trained by SVMs. Inspired by the idea and framework in Deepfool \cite{ref13}, we propose the optimization models  for  sAP, cuAP and uAP  against the trained SVMs and derive the explicit formulations for sAP, cuAP and uAP (binary case), and approximate formulation for uAP of multi\mbox{-}classification. For uAP, we provide an upper bound for the fooling rate. Our numerical results also verify the fast speed in finding the three types of adversarial perturbations for classification models.
The contributions of this paper are as follows. Firstly, we propose optimization models of  sAP, cuAP and uAP against linear SVMs.  Secondly, we derive the explicit solutions and the approximate solution for the three types of adversarial perturbations, avoiding iterative process. 
Based on the formulae, we increase the interpretability of against SVM classification models. 

The rest of the paper is organized as follows. In  \cref{Optimization Models for adversarial perturbations against linear SVMs}, we briefly review the  binary  and multiclass linear SVMs. 
Then we propose a general optimization framework for adversarial perturbations of linear SVMs.
In  \cref{Universal adversarial perturbations of linear binary SVMs}	and  \cref{Adversarial Perturbations for Multiclass SVMs}, we solve  sAP, cuAP and uAP for binary SVM and multiclass SVM, respectively.
In \cref{Numerical Experiments.}, we conduct numerical experiments on MNIST and CIFAR-10 dataset. Final conclusions are given in  \cref{conclusion}.

\section{Optimization Models for Adversarial Perturbations against Linear SVMs}\label{Optimization Models for adversarial perturbations against linear SVMs}  

\subsection{Training models of linear SVMs} 
 
In this part, we  briefly review the optimization models of binary linear SVMs and multiclass linear SVMs.
We use the short-hand notation $ [B]$  to denote the set $ [B]=\{1, 2, \dots, B\} $ for some integer $ B \in \mathbb{N} $.
For the binary classification problem, we assume that the training  dataset  $ D=\{(x_1,y_1), (x_2,y_2), \dots, (x_n,y_n)\} $ is given, with $ x_i \in \mathbb{R}^p $ and $ y_i \in \{-1,1\} $ is the label of the corresponding $ x_i $. SVM aims to find the hyperplane $ H \triangleq \{x: w^{\top} x+b=0\}  $ such that the training data in $D$ are separated as much as possible. A typical soft-margin training model of the  binary SVMs is the $ L_1\mbox{-}$loss SVM model  \cite{ref33}, which is given as follows 
\begin{equation}\label{equation1}
\begin{aligned}
\min_{w \in \mathbb{R}^p, \, b\in \mathbb{R},\,  \xi\in \mathbb{R}^n} &\quad\frac{1}{2}||w||^2+C\sum_{i=1}^{n}\xi_i  \\
\hbox{s.t.}&\quad y_i(w^{\top} x_i+b)\geq 1-\xi_i, {i \in [n],}\\
&\quad \xi_i \geq 0, {i \in [n]}.
\end{aligned}
\end{equation}
Let $ (w^*,b^*)\in \mathbb{R}^p \times \mathbb{R} $ be the solution obtained by solving the $L_1\mbox{-}$loss SVM \cref{equation1}. We show the separating hyperplane (decision boundary)  in \cref{fig:figure1}.
%\begin{equation}  
%H \triangleq \{x: w^{\top} x+b=0\} 
%\end{equation}
%and the decision function is
For a test sample  $ x \in \mathbb{R}^p $, the decision function is given by \begin{equation}\label{equation3}
\hat{k}(x)=sign(({w^*})^{\top} x+b^*).
\end{equation}
We refer to \cite{ref34} for other binary SVM models with different loss functions. 
%图片  
\begin{figure}[h] 	
	\centering 	{\includegraphics[width=0.39\linewidth]{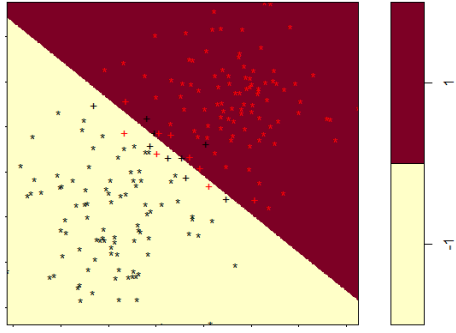}} 	
	\caption{The decision boundary of binary linear SVM, `$* $' indicates  data points, and  `$ +$' indicates support vectors.} 	
	\label{fig:figure1} 
\end{figure}

In terms of multiclass linear SVMs, we assume that the training dataset is given by $ D=\{(x_1,y_1), (x_2,y_2), \dots, (x_n,y_n)\} $, with $ x_i \in \mathbb{R}^p $ and $ y_i \in { [c]} $  is the label of the corresponding $ x_i $. Crammer and  Singer \cite{ref6} proposed an approach for multiclass problems by solving a single optimization problem. The idea of Crammer$\mbox{-}$Singer multiclass SVM is to divide the space $\mathbb{R}^p $ into $ c $ regions directly by several hyperplanes, and each region corresponds to the input of a class. Specifically, \cite{ref6} solves the following problem 
\begin{equation}\label{equation2}
\begin{aligned}
\min\limits_{w\in \mathbb{R}^{p\times c},\,  \xi\in \mathbb{R}^n}\quad&\quad\frac{1}{2}\sum_{l=1}^{c}||w_l||^2_2 + C \sum_{i=1}^{n}\xi_i\\
\hbox{s.t.}\quad&\quad w_{y_i}^{\top} x_i- w_l^{\top} x_i+\delta_{y_i,l}\geq 1-\xi_i, {i \in [n]}, {l \in [c]}\\
&\quad \xi_i \geq 0, {i \in [n]},
\end{aligned}
\end{equation}
where $ \delta_{y_i,l} =
\begin{cases}
1,  & \hbox{if} \; y_i=l, \\
0, & \hbox{if} \; y_i \neq l
\end{cases} $ and $ w $ is a matrix of size $ p\times c $, $ w_l $ is the $ l\mbox{-}th $ column of $ w $. 

Let $w^*\in \mathbb{R}^{p\times c}$, $ \xi^* \in \mathbb{R}^n $ be the solution obtained by solving \cref{equation2}. For a test data $ x\in \mathbb{R}^p $, the decision function is given by
%We obtain the separating hyperplane (decision boundary):
%\begin{equation}
%H \triangleq \{x: w_l^{\top} x-w_l^{\top} x=0\}
%\end{equation}
%and the decision function is
%图片
\begin{equation}\label{equ2_3}
\hat{k}(x)=\mathop{\arg\max} \{{(w^*)}_l^{\top} x \, | \, l \in [c] \}.
\end{equation}
The value of the inner-product of the $ l\mbox{-}th $ column of $ w $ with the instance $ x $ (i.e., $ w_l^{\top} x $)     is referred to as the confidence and the similarity score for the $ l\mbox{-}th $ class. Therefore, according to the definition above, the predicted label is the index of the column attaining the highest similarity score with $ x $. This can be viewed as a generalization of linear binary classifiers.

\subsection{General optimization framework for adversarial perturbations of linear SVMs}\label{Optimization Problem of Universal Adversarial Perturbations}

We first make the following assumption.

\begin{assumption}\label{assumption1} %给线性二分类和多分类用的sAP,cuAP 
	Assume that the dataset is given by $T_x=\{x_1, x_2, \dots, x_n\}$, $ x_i \in \mathbb{R}^p $, $ {i \in [n]} $.  There are $ c $ class of output labels, and the proportion of each class is $ \theta_l \in (0,1) $, $ \sum_{l=1}^c \theta_l=1 $, $ {l \in [c]}$. $\Omega$ is a subset of $ T_x $.  $ \hat{k} $ is a linear SVM classification model trained on dataset $ T_x $ and it  maps an input sample $ x\in T_x $ to an estimated label $ \hat{k}(x) $.
	Specifically, for the linear binary SVM classification model  \cref{equation3}, there is $ c=2 $ and  $ \hat{k}(x_i) \in \{-1,1\} $, $ {i \in [n]} $. 
	For the  linear multiclass SVM classification model  \cref{equ2_3}, we have  $ c \geq 3 $  is an integer, and $ \hat{k}(x_i) \in { [c]} $, ${i \in [n]}$. 		 	 
\end{assumption}

The general optimization model for generating adversarial perturbations against linear SVMs is to look for a perturbation  $ r \in \mathbb{R}^p $ with smallest length, such that the data $ x \in \Omega $ can be misclassified. That is,
\begin{subequations}\label{equ2}
	\begin{align}
	\min_{r \in \mathbb{R}^p} \  & \  \|r\|_2 \label{equ2_1} \\
	\hbox{s.t.} \  & \  \hat{k}(x+r)\neq \hat{k}(x),\ \forall \, x \in \Omega. \label{equ2_2}
	\end{align}
\end{subequations}

Different choices of $ \Omega $ lead to different types of adversarial perturbations, which are given below. 
\begin{itemize} 	
	\item $ |\Omega|=1 $. It means that one only needs to misclassify a single sample. It is basically sAP.
	\item  $ \Omega=T_x^l:=\{ x\, |\, \hat{k}(x)= l\} $. If $ \hat{k} $ is trained by \cref{equation3}, $ l \in \{-1,1\} $; if $ \hat{k} $ is trained by \cref{equ2_3}, $ l $ is a specific value in  $ {[c]} $. For such cases, model \cref{equ2} aims to misclassify one specific class of data, which is actually cuAP.
	\item $ \Omega=T_x $. It means that we aim to misclassify all data  in $ T_x $. It is actually uAP.
\end{itemize}

Notice that for sAP, the original form of model \cref{equ2} is proposed in \cite{ref13}. Actually, sAP is defined as the minimum perturbation  $ r $ that is sufficient to change the estimation label $ \hat{k}(x) $.

For cuAP and uAP, it may be difficult to find a nonzero $ r $ to satisfy condition \cref{equ2_2}. In other words, problem \cref{equ2} may admit only zero feasible solution for cuAP and uAP.
Therefore, the following optimization model is also proposed to calculate cuAP and uAP:
\begin{subequations}\label{equ3}
	\begin{align}
	\max_{r \in \mathbb{R}^p} \  & \  E_\Omega\left(1_{\hat{k}(x+r)\neq \hat{k}(x)}\right) \label{equ3_1} \\
	\hbox{s.t.} \  & \  \|r\|_2 \le \xi, \label{equ3_2}
	\end{align}
\end{subequations}
where $ 1_{A} $ equals one if  $ A $ is true and zero otherwise.
Similarly, $ \Omega=T_x^l $ and $ \Omega=T_x $ correspond  to  cuAP and uAP, respectively. The primitive expectation formula of \cref{equ3_1} is proposed in \cite{ref20}. $ \xi >0 $ is given, which controls the magnitude of the perturbation vector $ r $.
Below we give the  explanation  of the problem \cref{equ3}.

Given event $ A $ as $ x $ is successfully fooled, i.e., $ \hat{k}(x+r)\neq \hat{k}(x)$, and define the function:
$$ I_A(\omega)= \begin{cases} 1, & \omega \in A\\ 0, & \omega \notin A \end{cases}.
$$
$ I_A(\omega) $ is the indicator function of event $ A $. In the above definition, $ I_A(\omega) $ is a random variable, $ \omega $ is the sample point, $ \omega \in A $ indicates that event A occurs. At this time, the value of random variable $ I_A(\omega) $ is 1. And $ \omega \notin A $ indicates that event A does not occur after experiment, at this time, the value of random variable  $ I_A(\omega) $ is 0.
We abbreviate  $ I_A(\omega) $ as $ 1_{A} $,   then the expectation of random variable $ I_A(\omega) $ is $ EI_A(\omega) = E_\Omega\left(1_{\hat{k}(x+r)\neq \hat{k}(x)}\right) $.  
$  EI_A(\omega)$  reflects the average value of random variable $ I_A(\omega) $.   
Therefore, the constrained optimization problem \cref{equ3}  maximizes the average value of $ 1_{\hat{k}(x+r)\neq \hat{k}(x)} $ on the dataset $ \Omega $ when the perturbation $ r $ is small enough. 

Further, we try to convert the  expectation of $ I_A(\omega) $ into the probability of event $ A $, so as to simplify the calculation.  
We let  $ P_{\Omega}(\hat{k}(x+r)\neq \hat{k}(x)) $ represent the probability of $ \hat{k}(x+r)\neq \hat{k}(x)$, all $ x $ in $ {\Omega} $ are distributed independently. That is, $ P_{\Omega}(\hat{k}(x+r)\neq \hat{k}(x)) $ represent  the probability of event $ A $.

%Given sample space $ S $=\{$ x \in \Omega $ can be misclassified, $ x \in \Omega $ cannot be misclassified\}. For event $ A \subset S$, $ A $=\{$ x \in \Omega $ can be misclassified\}, we have
%\begin{equation}
%I_A(\omega)= \begin{cases}
%1, & \omega \in A\\
%0, & \omega \notin A
%\end{cases}.
%\end{equation}

%If we set the random variable $ Z=I_A(\omega) $,   then the expectation of random variable $ Z $ is $ EZ=EI_A(\omega) = E_\Omega[1_{\hat{k}(x+r)\neq \hat{k}(x)}] $. 
%$ EZ $  reflects the average value of random variable $ Z $.  
%Therefore, the constrained optimization problem \eqref{equ3}  maximizes the average value of $ 1_{\hat{k}(x+r)\neq \hat{k}(x)} $ on the dataset $ \Omega $ when the perturbation $ r $ is small enough. 

%改成了概率形式
%Further, we try to convert the  expectation of the indicator function of event A into the probability of event A, so as to simplify the calculation.  
%We let  $ P_{\Omega}(\hat{k}(x+r)\neq \hat{k}(x)) $ represent the probability of $ \hat{k}(x+r)\neq \hat{k}(x)$, all $ x $ in $ {\Omega} $ are distributed independently. That is, $ P_{\Omega}(\hat{k}(x+r)\neq \hat{k}(x)) $ represent  the probability of event A.

\begin{proposition}\label{proposition2}
	Given $ \Omega $,  it holds that $E_{ \Omega}\left(1_{\hat{k}(x+r)\neq \hat{k}(x)}\right)= P_{{\Omega}}(\hat{k}(x+r) \neq \hat{k}(x)) $.
	%线性二分类器上的通用对抗扰动的最优方向是固定的（or determined），且指向分类超平面。
\end{proposition}

\begin{proof} 
	Using random variable $I_A(\omega) $, we have 
	\begin{align*}
	E_{\Omega}\left(1_{\hat{k}(x+r)\neq \hat{k}(x)}\right)&=EI_A(\omega)\\&=1 \cdot P\{ \omega \in A\}+0 \cdot P\{ \omega \notin A\}\\&=P\{ \omega \in A\}
	\\&=P_{\Omega}(\hat{k}(x+r)\neq \hat{k}(x)).
	\end{align*}
	The proof is completed.
\end{proof}

Due to  \cref{proposition2}, the constrained optimization problem \cref{equ3} can also be written equivalently  as the following model
\begin{subequations}\label{equ4}
	\begin{align}
	\max_{r \in \mathbb{R}^p} \  & \  P_{\Omega}(\hat{k}(x+r)\neq \hat{k}(x)) \label{equ3_3} \\
	\hbox{s.t.} \  & \  \|r\|_2 \le \xi. \label{equ3_4}
	\end{align}	
\end{subequations}	
It establishes a connection with the model of the universal adversarial perturbations in \cite{ref15}.  
In the rest of the paper, we will mainly solve the problem \cref{equ4} and denote the perturbation rate $ P_{\Omega}(\hat{k}(x+r)\neq \hat{k}(x)) $  as $ G_{\Omega, \hat{k}} $.

\begin{remark}\label{remark0}
	In practice, the constraint $ \xi $ of  the size of uAP is usually selected through experimental verification, which is related to specific datasets. 
\end{remark}

\section{Adversarial Perturbations for Binary Linear SVMs}\label{Universal adversarial perturbations of linear binary SVMs}	

In this section, we solve sAP, cuAP and uAP for binary linear SVM, and give their explicit solutions and interpretations respectively.
%先全改f(x)为

\begin{assumption}\label{assumption3}
	Assume that \cref{assumption1} holds and the linear binary classifier is trained by model \cref{equation3}.
\end{assumption}

\subsection{The case of sAP} 
\begin{theorem}\label{theorem1}
	Under \cref{assumption3}, the  optimal sAP of  optimization problem  \cref{equ2} has the following closed form expression
	\begin{equation}\label{eq2}
	r=-sign(w^{\top} x + b) \cdot \dfrac{|w^{\top} x + b|}{\|w\|_2^2} \cdot  w.
	\end{equation}	 
\end{theorem}
%\begin{proposition}\label{proposition1}	
%	The optimal adversarial perturbation \eqref{equ2} on the linear binary classifier has closed form solution and it can be written as $ r=-sign(f(x)) \cdot \dfrac{|w \cdot x + b|}{\|w\|_2^2} \cdot  w $.
%	%线性二分类器上的通用对抗扰动的最优方向是固定的（or determined），且指向分类超平面。
%\end{proposition}

\begin{proof}
	Notice that $r$ satisfies \cref{equ2_2}. If $ \hat{k}(x)=sign(w^{\top}x+b)=1 $, then $ \hat{k}(x+r)=-1 $.  There is $ w^{\top}x+b \ge 0 $ and $ w^{\top}x+b+w^{\top}r \le 0 $. It gives that 	
	$$
	w^{\top}r \le -|w^{\top}x+b|.
	$$
	Similarly, if $ sign(w^{\top}x+b)=-1 $, we have
	$$
	w^{\top}r \ge |w^{\top}x+b|.
	$$
	Overall, \cref{equ2_2} can be rewritten as $$
	sign(w^{\top}x+b) w^{\top}r \le -|w^{\top}x+b|.$$
	Then the optimization problem  \cref{equ2} is equivalent to	
	\begin{subequations}\label{equ0} 	
		\begin{align} 	
		\min_{r \in \mathbb{R}^p} \  & \  \|r\|_2 \label{equ0_1} \\ 	
		\hbox{s.t.} \  & \  sign(w^{\top}x+b) w^{\top}r \le -|w^{\top}x+b|, \ \forall x \in \Omega. \label{equ0_2} 	
		\end{align} 
	\end{subequations}	
	
	Notice that in sAP, there is only one single point in $ \Omega $. 
	Through the constraint condition \cref{equ0_2}, we obtain that the feasible region of $ r $ is the closed half-space $ \{ r\, |\, \langle sign(w^{\top}x+b) w, r \rangle \le -|w^{\top}x+b| \} $ that does not contain the origin. 
	In \cref{equ0_1}, since $ \|r\|_2 $ represents the Eucliden distance between the origin and the vector  $ r $ in the feasible region,
	the optimal solution of optimization problem \cref{equ0} is the shortest distance from the origin to hyperplane $ \{ r\, | \, \langle sign(w^{\top}x+b) w, r \rangle = -|w^{\top}x+b| \} $, and the direction  is opposite to the normal vector of hyperplane. The  distance is calculated by $ \dfrac{|w^{\top} x + b|}{\|w\|_2} $ and the direction  is $ -sign(w^{\top} x + b) \cdot \dfrac{w}{\|w\|_2} $.
	%When the classifier is linear binary, it is obvious that the distance between the sample $ x $ and the separating hyperplane $H $ can be calculated by
	%% In the linear binary classifier model,
	%\begin{equation}
	%d=\dfrac{|w^{\top} x + b|}{\|w\|_2}. \label{eq1}
	%\end{equation}
	%The direction of moving $ x $ in class $\hat{k}(x)$ toward the hyperplane is 
	%\begin{equation}
	% -sign(w^{\top} x + b) \cdot  \dfrac{ w}{\|w\|_2}.
	%\end{equation}
	Therefore, the optimal sAP of linear binary classifier can be written as  \cref{eq2}.
\end{proof}

\begin{remark}\label{remark1} 
	In this case, the solution provided by \cref{eq2} leads to the fact that $ sign(w^{\top}(x+r) + b)=0 $. 
	However, in practice, it is difficult to determine the class of the data that just lie on the hyperplane. Usually we should move the sample toward the hyperplane, and  make it slightly pass across the hyperplane. Mathematically speaking, we can add a small enough $ \varepsilon > 0$ to the perturbation vector to make sAP  become the following form 
	$$
	r=-sign(w^{\top}  x + b) \cdot \dfrac{|w^{\top}  x + b|+\varepsilon}{\|w\|_2^2} \cdot  w.
	$$
\end{remark}

\subsection{The case of cuAP}

\begin{theorem}\label{theorem2}
	Under  \cref{assumption3}, the  optimal cuAP of  optimization problem  \cref{equ4} has a closed form expression, which is  $ r=-sign(w^{\top} x + b)\cdot  \dfrac{\xi w}{\|w\|_2} $.
\end{theorem}

\begin{proof}
	Without loss of generality, let $ \Omega=T_x^1 $, that is, $ \Omega $ is the  dataset  with labels $ l=1 $. For any sample $ x \in \Omega $,  according to  \cref{theorem1}, the optimal $ r $  of sAP is given by
	$$
	r_x=- \dfrac{|w^{\top} x + b|}{\|w\|_2^2} \cdot  w.
	$$
	The direction of $ r_x $ is $ -  \dfrac{ w}{\|w\|_2} $. Therefore, under the constraint \cref{equ3_4},  we hope 
	that the data could be fooled as much as possible in the positive data. The optimal cuAP of positive  dataset  is 
	$
	r=-  \dfrac{\xi w}{\|w\|_2}.
	$
	Similarly, if $ \Omega $ is the negative  dataset, the optimal cuAP is $ r=\dfrac{\xi w}{\|w\|_2} $. In summary, the optimal cuAP  of linear binary classifier can be written as  
	\begin{equation}\label{equation5}
	r=-sign(w^{\top} x + b)\cdot  \dfrac{\xi w}{\|w\|_2}.
	\end{equation}
	The proof is finished.
\end{proof}

%\begin{remark} 	
%	\sout{
%	If we denote the decision function \eqref{equation3} of the linear binary classifier  as $ \hat{k}(x)=sign(f(x)) $, then $ \nabla f(x)=w $. The optimal cuAP given in Theorem \ref{theorem2} can be further written as 	}	
%\begin{equation}\label{equ3_21} 	
%\sout{r= -sign(f(x))\cdot  \dfrac{\xi \nabla f(x)}{  \|\nabla f(x)\|_2} .}
%\end{equation}  	
%%To simplify the calculation, we will not use the gradient-related formula when actually generating  cuAP, 
%\sout{The formula in \eqref{equ3_21} generating  cuAP
%establishes a connection with some previous papers that used the gradient information of the classifiers to establish adversarial perturbations (\cite{ref11}, 2015; \cite{ref13}, 2016; \cite{ref14}, 2019; \cite{ref20}, 2020). }
%\end{remark}  
%用不用详细说是什么联系？？

\subsection{The case of uAP}

For the case of uAP, we have the following result.

\begin{theorem}\label{theorem3} 	
	Assume that \cref{assumption3} holds. Let $ \theta_1 $ be the ratio of positive data over all the sample data.
	\begin{itemize}[itemindent=3em]  	 	
		\item 	Suppose $ \xi $ is sufficiently large, the  optimal uAP of  optimization problem  \cref{equ4} takes the following form:
		\begin{equation}\label{the1}	
		r=\begin{cases}
		- \dfrac{\xi w}{\|w\|_2}, & if\  \theta_1>\dfrac{1}{2}, \\
		\dfrac{\xi w}{\|w\|_2}, & otherwise. \\
		\end{cases} 
		\end{equation}
		\item The upper bound of $ G_{\Omega, \hat{k}} $ on the linear binary classifier can  reach $ \max(\theta_1, 1-\theta_1)$, that is, $ G_{\Omega, \hat{k}} \leq \max(\theta_1, 1-\theta_1)$. 
	\end{itemize} 
\end{theorem}

\begin{proof} 
	We have
	\begin{align*}
	P_\Omega(\hat{k}(x+r)\neq \hat{k}(x))=&P_{T_x}(sign(w^{\top}(x+r)+b) \neq sign(w^{\top}x+b))\nonumber
	\\=&P_{T_x}( (w^{\top}(x+r)+b)(w^{\top}x+b)<0  )\nonumber
	\\=&P_{T_x}( {(w^{\top}x+b)}^2+(w^{\top}x+b)w^{\top}r <0)\nonumber
	\\=&\theta_1  P_{T_x^1}({(w^{\top}x+b)}^2+(w^{\top}x+b)w^{\top}r <0)+(1-\theta_1) \nonumber 
	\\& P_{T_x^{-1}}({(w^{\top}x+b)}^2+(w^{\top}x+b)w^{\top}r <0).
	\end{align*}
	We abbreviate $ P_{T_x^1}({(w^{\top}x+b)}^2+(w^{\top}x+b)w^{\top}r <0) $ and $ P_{T_x^{-1}}({(w^{\top}x+b)}^2+(w^{\top}x+b)w^{\top}r <0) $ as $P_{T_x^1}$ and $ P_{T_x^{-1}}$.
	
	Firstly, we prove that $ r $ will not satisfy both $ P_{T_x^1}>0 $ and $ P_{T_x^{-1}}>0 $, i.e.,  uAP can fool only one class of data at the same time. We discuss two situations. 	If $ P_{T_x^1}>0 $, that is,  there exists positive data $ x_0 \in  {T_x^1}$ such that $ {(w^{\top}x_0+b)}^2+(w^{\top}x_0+b)w^{\top}r <0  $ holds. Because $ (w^{\top}x_0+b)w^{\top}r <- {(w^{\top}x_0+b)}^2 < 0 $ and $ w^{\top}x_0+b>0 $, we have $ w^{\top}r<- (w^{\top}x_0+b)<0 $. In this case, for any negative data $ x $, there is  $ (w^{\top}x+b)(w^{\top}x+b)+(w^{\top}x+b)w^{\top}r>0 $. Thus, uAP cannot mislead the negative data.	That is, $ P_{T_x^{-1}}=0 $. If $ P_{T_x^{-1}}>0 $, we have $ P_{T_x^1}=0 $. Therefore, uAP of  optimization problem  \cref{equ4} can fool one class of data at most,  and the upper bound of $ G_{\Omega, \hat{k}} $ on the linear binary classifier is $ \max(\theta_1, 1-\theta_1) $.

	Next, we prove the explicit formula for the optimal direction of uAP. If uAP can only fool positive data, the optimization problem \cref{equ4} can be simplified to a cuAP problem with $\Omega=T_x^1 $. According to \cref{equation5} of  \cref{theorem2}, we obtain that the optimal solution for uAP is
	$ r=-  \dfrac{\xi w}{\|w\|_2} $.
	If uAP can only fool negative data, we obtain that the optimal uAP is $ r=\dfrac{\xi w}{\|w\|_2} $. 
	Because $\xi$ is sufficiently large, when adding $ r $ to positive data or negative data, all data will be mislead, so we only need to select the class with more data to attack. Finally, we can reach $ \cref{the1} $.
	%Considering comprehensively, we hope to maximize the fooling rate of the total data, that is, select a larger  proportion of data to add attacks. If $ \theta_1\geq \theta_2 $, the optimal uAP is \eqref{equation6}. Otherwise, the optimal uAP is \eqref{equation7}.	
\end{proof}

To proceed our discussion about the relationship between $ G_{\Omega, \hat{k}} $  and $ \xi $, we need the following assumption on dataset $ T_x $.

\begin{assumption}\label{assumption2} %高斯分布	
	Let \cref{assumption3} hold. Assume that the data in $ T_x $  satisfies Gaussian mixture distribution, {with} the probability density function of the data  $ p(x\, | \, \theta_1, \theta_2, \mu_+,  \mu_-, \Sigma_+,  \Sigma_-)\\ = \theta_1 p(x\, | \, \mu_+,\Sigma_+) + \theta_2 p(x\, |\, \mu_-,\Sigma_-) $, $ x \in \mathbb{R}^p $. Here we  notice that $ p(x\, | \, \mu_+,\Sigma_+) = $ $\dfrac{1}{{(2 \pi)}^\frac{p}{2} {|\Sigma_+|}^\frac{1}{2}}\cdot\\exp \left\{-\dfrac{1}{2}{(x-\mu_+)}^{\top} \Sigma_+^{-1}(x-\mu_+)\right\}$ is the Gaussian distribution density function of positive data, $ \mu_+ $ and $ \Sigma_+ $ represent the expectation and variance of positive data, and $ p(x\, | \, \mu_-,\Sigma_-)$ $ = $ $\dfrac{1}{{(2 \pi)}^\frac{p}{2} {|\Sigma_-|}^\frac{1}{2}} exp \left\{-\dfrac{1}{2}{(x-\mu_-)}^{\top}\Sigma_-^{-1}(x-\mu_-)\right\} $ is the Gaussian distribution density function of negative data, $ \mu_- $ and $ \Sigma_- $ represent the expectation and variance of negative data.	
	%	The data in $ T_x $   satisfies Gaussian mixture distribution with expectation $ \mu $ and variance $ \Sigma $, $ p(x\ | \mu,\Sigma)=\theta_1 p(x|\mu_+,\Sigma_+) + \theta_2 p(x|\mu_-,\Sigma_-) $, 
	%		where $ \mu_+ $ and $ \Sigma_+ $ represent the expectation and variance of positive data, $ \mu_- $ and $ \Sigma_- $ represent the expectation and variance of negative data. 
	%where the proportion of positive data $x_+ \sim (\mu_+,\Sigma_+) $ is $ \theta_1 $, and the proportion of negative data $ x_- \sim (\mu_-,\Sigma_-) $ is $ \theta_2 $. $ \Omega$ is the subset of $ T $. $ \hat{k} $ is a linear SVM classification model with known parameters.
\end{assumption}

With the above assumption and  \cref{theorem3}, we have the following result.

\begin{corollary}
	Assume that \cref{assumption2} holds. Since $ \|r\|_2 \le \xi $, we have
	\begin{equation}\label{the2}	 		
	P_\Omega(\hat{k}(x+r)\neq \hat{k}(x))=\begin{cases} 		
	\theta_1 F_{Y_+}(\xi), & if\  \theta_1>\theta_2, \\ \theta_2 F_{Y_-}(\xi), & otherwise, \\ 		
	\end{cases}  		
	\end{equation}
	where $ Y_+ $ and $ Y_- $  {are} defined by $ Y_+=\dfrac{w^{\top}x+b}{\|w\|_2} $ and $ Y_-=\dfrac{-w^{\top}x-b}{\|w\|_2} $, $ F_{Y_+}(\xi) $ and $ F_{Y_-}(\xi) $ are the
	cumulative distribution function of   $ Y_+ $ and $ Y_- $, respectively defined by 
	$ F_{Y_+}(\xi) \triangleq P(Y_+<\xi) $ and
	$ F_{Y_-}(\xi) \triangleq P(Y_-<\xi) $.
	%and the lirate of uAP has a positive correlation with the setting of constraint condition $ \xi $. If $ \theta_1\geq  \theta_2 $, the correlation can be expressed as $ P_x(\hat{k}(x+r)\neq \hat{k}(x)) =\theta_1 F_{Y_+}(\xi) $, otherwise, $ P_x(\hat{k}(x+r)\neq \hat{k}(x)) =\theta_2 F_{Y_-}(\xi) $, where $ F$ represents the cumulative distribution function,  $Y_+=\dfrac{w^{\top}x+b}{\|w\|_2} $ and $Y_-=\dfrac{-w^{\top}x-b}{\|w\|_2} $. When $ \xi $ increases, the fooling rate also increases until reaches the upper bound. 
	%数据的扰动率与通用对抗扰动的L2范数具有正相关关系。near binary classifier is \eqref{equation3}. The fooling 
\end{corollary}	
\begin{proof}
	Rewrite  $ G_{\Omega, \hat{k}} $ in \cref{equ3_3} as	
	%数据的扰动率可写作
	\begin{align*}
	P_\Omega(\hat{k}(x+r)\neq \hat{k}(x))
	=&P_{T_x}( (w^{\top}(x+r)+b)(w^{\top}x+b)<0  )\\
	=&\theta_1  P_{T_x^1}(w^{\top}(x+r)+b<0)+\theta_2  P_{{T_x^{-1}}}(w^{\top}(x+r)+b>0).
	\end{align*}
	If $ \theta_1\geq \theta_2 $,  according to  \cref{theorem3},
	then the optimal uAP is $ r=-  \dfrac{\xi w}{\|w\|_2} $. We have the maximal perturbation rate  {as follows}
	\begin{align*}
	&\theta_1  P_{T_x^1}\left( w^{\top}x+b- \dfrac{\xi w^{\top} w}{\|w\|_2} <0\right)+\theta_2  P_{T_x^{-1}}\left( w^{\top}x+b- \dfrac{\xi w^{\top} w}{\|w\|_2} >0\right)\\
	=&\theta_1  P_{T_x^1}( w^{\top}x+b<\xi \|w\|_2 )+\theta_2  P_{T_x^{-1}}(w^{\top}x+b>\xi  \|w\|_2 )\\
	=&\theta_1 P_{T_x^1}( w^{\top}x+b <\xi \|w\|_2)+ \theta_2 \times 0.
	\end{align*}
	We  {denote} the Gaussian distribution with  expectation $ \mu $ and variance $ \Sigma $ as $N(\mu,\Sigma) $.
	Because the positive data $ x $ $ \sim $  
	$N(\mu_+,\Sigma_+) $,  {we have} the random variable $Y_+=\dfrac{w^{\top}x+b}{\|w\|_2} $ $ \sim $  $N\left(\dfrac{w^{\top} \mu_+ +b}{\|w\|_2},\dfrac{w^{\top}}{\|w\|_2}\Sigma_+\dfrac{w}{\|w\|_2}\right)$.  {Therefore, we can obtain that} 
	%Because the positive data satisfy $x_+ \sim (\mu_+,\Sigma_+) $ and the negative data satisfy $ x_- \sim (\mu_-,\Sigma_-) $. Let's set the random variables $Y_+=\dfrac{w^{\top}x_++b}{\|w\|_2} $ and  $Y_-=\dfrac{-w^{\top}x_--b}{\|w\|_2} $ ,i.e., $ Y_+ \sim (\dfrac{w^{\top} \mu_+ +b}{\|w\|_2},\dfrac{w^{\top}}{\|w\|_2}\Sigma_+\dfrac{w}{\|w\|_2})$ and $Y_- \sim (\dfrac{-w^{\top} \mu_- -b}{\|w\|_2},\dfrac{-w^{\top}}{\|w\|_2}\Sigma_+\dfrac{-w}{\|w\|_2}) $, then 
	$$
	P_x(\hat{k}(x+r)\neq \hat{k}(x)) 
	=\theta_1 P_{T_x^1}\left(\dfrac{w^{\top}x+b}{\|w\|_2}< \xi \right)
	=\theta_1 F_{Y_+}(\xi),
	$$
	where $ F_{Y_+} $ is the cumulative distribution function of $ Y_+ $.
	If $ \theta_1< \theta_2 $, the optimal uAP $ r=\dfrac{\xi w}{\|w\|_2} $.
	There is $Y_-=\dfrac{-w^{\top}x-b}{\|w\|_2} $ $\sim$  $ N \left(\dfrac{-w^{\top} \mu_- -b}{\|w\|_2},\dfrac{-w^{\top}}{\|w\|_2}\Sigma_+\dfrac{-w}{\|w\|_2}\right) $. Therefore, we have
	\begin{align*} 
	&\theta_1  P_{T_x^1}\left( w^{\top}x+b+ \dfrac{\xi w^{\top} w}{\|w\|_2} <0\right)+\theta_2  P_{T_x^{-1}}\left( w^{\top}x+b+ \dfrac{\xi w^{\top} w}{\|w\|_2} >0\right)\\ =&\theta_1 \times 0+\theta_2  P_{T_x^{-1}}\left(-w^{\top}x-b<\xi  \|w\|_2 \right) \\=&\theta_2  F_{Y_-}(\xi),
	\end{align*}
	where $ F_{Y_-} $ is the cumulative distribution function of $ Y_- $.
	The proof is completed.
	%Obviously, the fooling rate $P_x(\hat{k}(x+r)\neq \hat{k}(x)) $ has a positive correlation with the setting of constraint condition $ \xi $. 
	%通用对抗扰动的L2范数与数据的扰动率具有正相关关系。
	%and the relationship between them can be drawn as Fig.~\ref{fig:figure2} 
	%\begin{figure}[h]
	%	\centering
	%	\includegraphics[width=0.48\linewidth]{Figure/figure2}
	%	\caption{The relationship between the norm  of the  class-universal adversarial perturbations and the fooling rate.}
	%	\label{fig:figure2}
	%\end{figure}
\end{proof}

%\begin{corollary}
%	In the linear binary classifier model, the setting of constraint condition $ \xi $ will directly affect the fooling rate of  class-universal adversarial perturbations.
%	%约束条件\xi的设置，直接影响通用对抗扰动的愚弄率。
%\end{corollary}
%\begin{proof}
% Suppose $ X=\{x_1, \cdots, x_m\} $ be a set from the distribution $ \mu $  and satisfy the following relationship,
%\begin{align*}
%\dfrac{|w^{\top}x_1+b| }{\|w\|_2} \le \dfrac{|w^{\top}x_2+b| }{\|w\|_2} \le \ dots \le \dfrac{|w^{\top}x_m+b| }{\|w\|_2} 
%\end{align*}
%First, if $ \xi $ is so small that $\xi \le   \dfrac{|w^{\top}x_1+b| }{\|w\|_2} $, then the fooling rate $ P_x(\hat{k}(x+r)\neq \hat{k}(x))=E_x[1_{\hat{k}(x+r)\neq \hat{k}(x)}]=\frac{1}{m}\sum_{i=1} 1_{\hat{k}(x_i+r)\neq \hat{k}(x_i)}=0 $. Second, if $ \xi $ satisfies $ \dfrac{|w^{\top}x_i+b| }{\|w\|_2} < {\xi} \le  \dfrac{|w^{\top}x_{i+1}+b| }{\|w\|_2}$, $ i=1, \ dots,m-1 $, then the fooling rate $ P_x(\hat{k}(x+r)\neq \hat{k}(x))=\dfrac{i}{m} $. Finally, if $ \xi $ is large enough  that $ \dfrac{|w^{\top}x_m+b| }{\|w\|_2} < {\xi}  $，then the fooling rate $ P_x(\hat{k}(x+r)\neq \hat{k}(x))=1$. \end{proof}

Notice that in human observation, it is expected that adversarial examples have no difference from undisturbed inputs. 	
Through formula \cref{the2}, we find that if we attack the data under this premise that uAP is small enough, $ G_{\Omega, \hat{k}} $ may not reach the upper bound. Therefore, excessive pursuit of high fooling rate will lose the good property of  uAP, so it is important to set an appropriate $\xi $ to trade off between a large $ G_{\Omega, \hat{k}} $ and the disturbance of data.
Another remark is that the conclusions in this section are also valid  to general machine learning models with the same decision function as SVM.

\section{Adversarial Perturbations for Multiclass  Linear SVMs}\label{Adversarial Perturbations for Multiclass SVMs}

In this section, we investigate sAP, cuAP and uAP for multiclass SVM, and derive the explicit solutions of sAP, cuAP  and the approximate solution of uAP  respectively. 
\begin{assumption}\label{assumption4} 	
	Let  \cref{assumption1} hold and let the linear multi\mbox{-}classifier be trained by  \cref{equ2_3}.
\end{assumption}

\subsection{The case of sAP}
\begin{theorem}\label{theorem5} 	
	Let \cref{assumption4} hold. The  optimal sAP of  optimization problem  \cref{equ2} is given by
	\begin{equation}\label{equ5}
	r=\dfrac{(w_{\hat{k}(x)}-w_{l^*})^{\top} x}{\|w_{\hat{k}(x)}-w_{l^*}\|_2^2}\cdot (w_{l^*}-w_{\hat{k}(x)}),
	\end{equation}
	where   
	$ l^*=\mathop{\arg\max}\limits_{{l \in [c]},\, l \neq \hat{k}(x)} \;   \alpha_l $, and 
	$ \alpha_l= \arccos\left(\dfrac{(w_{\hat{k}(x)}-w_{l})^{\top} x}{\|w_{\hat{k}(x)}-w_{l}\|_2 \|x\|_2}\right)$, ${l \in [c]} $, $ l \neq \hat{k}(x) $.

\end{theorem} 

\begin{proof}	
	For linear multi\mbox{-}classifier, \cref{equ2_2} can be written as
	$$  
	\left\{\max\limits_{{l \in [c]}, \, l \neq \hat{k}(x)} w^{\top}_{l}(x+r)\right\}\ge w^{\top}_{\hat{k}(x)}(x+r), \ x \in \Omega,
	$$
	%Further, \eqref{equ2_2} is  $ \max\limits_{{l \in [c]} } (w_{l}-w_{\hat{k}(x)})^{\top}(x+r)\ge 0, \ x \in \Omega$, 
	Therefore, we have
	$$
	\left\{ \min\limits_{{l \in [c]}, \, l \neq \hat{k}(x)} (w_{\hat{k}(x)}-w_{l})^{\top}(x+r)\right\}\le 0, \ x \in \Omega.
	$$
	The optimization problem  \cref{equ2} is equivalent to	
	\begin{align*} 	 	
	\min_{r \in \mathbb{R}^p} \  & \  \|r\|_2  \\ 	 	\hbox{s.t.} \  &  \ \left\{\min\limits_{{l \in [c]}, \, l \neq \hat{k}(x)} (w_{\hat{k}(x)}-w_{l})^{\top}(x+r)\right\}\le 0, \ x \in \Omega. 	 	\end{align*}
	In the multiclass model \cref{equ2_3}, the data $ x $ with label $ \hat{k}(x) $ are located at the intersection of  
	closed half-space $ \{ z\, |\, (w_{\hat{k}(x)}-w_{l})^{\top} z \ge 0 \} $, $ {l \in [c]} $, $ l \neq \hat{k}(x) $. The hyperplanes separating the data  $ x $  of class $ \hat{k}(x) $  from the data of  other classes are $ \{ z\, |\, (w_{\hat{k}(x)}-w_{l})^{\top} z = 0 \} $, denoted as $ H_{\hat{k}(x),l} $, $ {l \in [c]} $, $ l \neq \hat{k}(x) $.
	The distance between $ x $ and $ H_{\hat{k}(x),l} $ are $ d_{l}=\dfrac{(w_{\hat{k}(x)}-w_{l})^{\top} x}{\|w_{\hat{k}(x)}-w_{l}\|_2} $, $ {l \in [c]} $, $ l \neq \hat{k}(x) $. 
	Notice that we  change the class of a particular sample $ x $, the smallest perturbation is to move $ x $ toward the nearest hyperplane. It is obvious that the shortest distance between the sample $ x $ of class $\hat{k}(x)$ and the nearest separating hyperplane is
	% In the linear binary classifier model, 
	$$
	d_{l^*}=\dfrac{(w_{\hat{k}(x)}-w_{l^*})^{\top} x}{\|w_{\hat{k}(x)}-w_{l^*}\|_2},
	$$
	where
	$$
	l^*=\mathop{\arg\min}\limits_{{l \in [c]}, \, l \neq \hat{k}(x)} \; d_{l}=\mathop{\arg\min}\limits_{{l \in [c]}, \,  l \neq \hat{k}(x)} \; \|x\|_2 \cdot \cos \alpha_l=\mathop{\arg\min}\limits_{{l \in [c]}, \,  l \neq \hat{k}(x)} \;  \cos \alpha_l=\mathop{\arg\max}\limits_{{l \in [c]}, \,  l \neq \hat{k}(x)} \;   \alpha_l,
	$$
	with $ \alpha_l= \arccos\left(\dfrac{(w_{\hat{k}(x)}-w_{l})^{\top} x}{\|w_{\hat{k}(x)}-w_{l}\|_2 \|x\|_2}\right)$, i.e., the angle between $ x $ and $ w_{\hat{k}(x)}-w_{l} $, $ {l \in [c]} $, $ l \neq \hat{k}(x) $. In other words, since  $ w_{\hat{k}(x)}-w_{l} $ is the normal vector of the {hyperplane} $ H_{\hat{k}(x),l} $, we choose $ l^* $ by finding $ H_{\hat{k}(x),l^*} $ with the smallest angle with $ x $.
	The direction of moving $ x $ toward the $ H_{\hat{k}(x),l^*} $ is $  -\dfrac{w_{\hat{k}(x)}-w_{l^*}}{\|w_{\hat{k}(x)}-w_{l^*}\|_2} $. Overall,  the optimal sAP of linear multi\mbox{-}classifier can be given by \cref{equ5}. The proof is completed. 
\end{proof}

\begin{remark} 
	Similar to  \cref{remark1}, sAP in  practice takes the form of  ($ \varepsilon >0 $)
	%In the real world or experiment, because it is difficult to determine the class of the data that just fall on the hyperplane, usually we should not only move the sample to the hyperplane, but also make it slightly cross the hyperplane, that is, we can add a small enough $ \varepsilon $ to the perturbation vector to make  sAP  become
	$$
	r=\dfrac{(w_{\hat{k}(x)}-w_{l^*})^{\top} x+\varepsilon}{\|w_{\hat{k}(x)}-w_{l^*}\|_2^2}\cdot (w_{l^*}-w_{\hat{k}(x)}).
	$$ 
\end{remark}

\subsection{The case of cuAP}

\begin{theorem}\label{theorem6} 	 	
	Under  \cref{assumption4},  cuAP of  optimization problem  \cref{equ4}  is	
	\begin{equation}\label{equ7}
	r= \xi \cdot \dfrac{w_{{l_c^*}}-w_{\hat{k}(x)}}{  \|w_{{l_c^*}}-w_{\hat{k}(x)}\|_2},
	\end{equation}
	where   $ \hat{k}(x) $ is the class of data $ \Omega $  being attacked,
	$$
	l_c^*\in \mathop{\arg\max}  \{\gamma_l \, | \, l \in [c], \, l \neq \hat{k}(x)\},
	$$
	and $ \gamma_l $ is the ratio of $ x \in  \Omega$ that has the largest angle with $ w_{\hat{k}(x)}-w_{l} $.
\end{theorem}

\begin{proof}
	For any sample $ x \in \Omega $, according to  \cref{theorem5}, the optimal sAP is  given by
	$$	
	r_x=\dfrac{(w_{\hat{k}(x)}-w_{l^*})^{\top} x}{\|w_{\hat{k}(x)}-w_{l^*}\|_2^2}\cdot (w_{l^*}-w_{\hat{k}(x)}),
	$$
	and the  optimal direction of sAP is 
	$	
	\dfrac{w_{l^*}-w_{\hat{k}(x)}}{\|w_{\hat{k}(x)}-w_{l^*}\|_2},
	$
	where  $ l^*=\mathop{\arg\max}\limits_{{l \in [c]}, \,  l \neq \hat{k}(x)} \;   \alpha_l $, and $ \alpha_l$ are the angles between $ x $ and $ w_{\hat{k}(x)}-w_{l} $, $ {l \in [c]} $, $ l \neq \hat{k}(x) $.
	Since the class $\hat{k}(x) $ of all $ x \in \Omega $ is the same, we need to find the unique $ l_c^* $, i.e., the same direction $ \dfrac{w_{l_c^*}-w_{\hat{k}(x)}}{\|w_{l_c^*}-w_{\hat{k}(x)}\|_2}$ of $ r_x $, for all the data $ x $ in  $ \Omega $. 
	%At this time, the direction of the optimal perturbation is unique, which is $ \dfrac{w^{\top}_{l_c^*}-w^{\top}_{\hat{k}(x)}}{\|w_{\hat{k}(x)}-w_{l_c^*}\|_2} $.
	%The choice of $ l_c^* $ follows the following rules.
	Suppose that the ratio of data, which has the largest angle with $ w_{\hat{k}(x)}-w_{l} $, is $ \gamma_l $, ${l \in [c]} $, $ l \neq \hat{k}(x) $. That is,  
	\begin{align*}
	\gamma_l&=\dfrac{\text{The number of } x \text{ which} \text{ has the largest angle with } w_{\hat{k}(x)}-w_{l}, x  \in  \Omega }{\text{The number of data in }  \Omega }\\&=\dfrac{\text{The number of } x \text{ corresponding to  } l^*=l, \ x  \in  \Omega }{\text{The number of data in }  \Omega  }, {l \in [c]},  l \neq \hat{k}(x),
	\end{align*}
	and  $ \sum_{{l \in [c]}, \, l \neq \hat{k}(x)} \gamma_l=1 $. 
	To maximize  $ G_{\Omega, \hat{k}} $, we need to choose $ {l_c^*} $ to make as many $ x $ as possible to obtain the optimal solution in \cref{theorem5}.
	Therefore, we choose $ l_c^* $ as
	$
	l_c^*\in \mathop{\arg\max}  \{\gamma_l \, | \, l \in [c], \, l \neq \hat{k}(x)\} .
	$
	By doing so, we can find one
	direction of the optimal perturbation  which is $ \dfrac{w_{l_c^*}-w_{\hat{k}(x)}}{\|w_{\hat{k}(x)}-w_{l_c^*}\|_2} $.
	Under the constraint \cref{equ3_4}, the maximum length of the cuAP is limited by $ \xi $.  Thus, cuAP of the linear multi\mbox{-}classifier on $ \Omega $ is given by \cref{equ7},
	where $ \xi $ is a certain small value which limits the norm  of cuAP,  and  $\hat{k}(x)$ is the class of data being attacked.
\end{proof}

\subsection{The case of uAP}

\begin{theorem}\label{theorem7}  	
	Assume that  \cref{assumption4} holds. \begin{itemize}[itemindent=3em]
		\item 	Suppose $ \xi $ is sufficient large, uAP of  optimization problem  \cref{equ4} can be written as: 		 			\begin{equation}\label{the5}	 		 			r= \dfrac{\xi w_{{l^*_u}}}{  \|w_{{l^*_u}}\|_2}, \  l^*_u =\mathop{\arg\min} \{ \theta_{l} \, | \, l \in [c] \}.	 			\end{equation} 		 			\item  $ G_{\Omega, \hat{k}}$  can  be bounded by $1-\theta_{l^*_u}$. That is, $ G_{\Omega, \hat{k}} \le 1-\theta_{l^*_u}$.
	\end{itemize}	 		
\end{theorem}  

\begin{proof} 
	We calculate $ G_{\Omega, \hat{k}} $ on all data according to the division of different classes $ \hat{k}(x) \in [c]$ of data $ x $. We have
	\begin{equation}\label{equ3_15}
	\begin{aligned}		
	P_\Omega(\hat{k}(x+r)\neq \hat{k}(x))
	=&\sum_{\hat{k}(x) \in [c]}  \theta_{\hat{k}(x)}P_{T_x^{\hat{k}(x)}}(\hat{k}(x+r)\neq \hat{k}(x)) \\
	=&\sum_{\hat{k}(x) \in [c]}  \theta_{\hat{k}(x)}P_{T_x^{\hat{k}(x)}}\left(\left\{\max\limits_{l \in [c], \, l \neq \hat{k}(x)} w^{\top}_{l}(x+r)\right\}\ge w^{\top}_{\hat{k}(x)}(x+r)\right). 	
	\end{aligned} 
	\end{equation}
	We abbreviate $ P_{T_x^{\hat{k}(x)}} \left(\left\{ \max\limits_{l \in [c], \, l \neq \hat{k}(x)} w^{\top}_{l}(x+r)\right\} \ge w^{\top}_{\hat{k}(x)}(x+r)\right) $  as $\widetilde{P}_{T_x^{\hat{k}(x)}}$.

	Firstly, we prove that $ r $ will not satisfy  $ \widetilde{P}_{T_x^{\hat{k}(x)}}>0 $,  for all $ \hat{k}(x) \in [c] $ at the same time.  Notice that for a fixed class $ \hat{k}(x) \in [c] $, denoted as $ q $, if $ \widetilde{P}_{T_x^{q}}>0 $, then  there exists  data $ x_0 \in  T_x^{q}$ such that $\left\{\max\limits_{{l \in [c]}, \, l \neq q} w^{\top}_{l}(x_0+r)\right\}- w^{\top}_{q}(x_0+r)\ge 0$ holds. Because $ x_0 \in  T_x^{q}$, we have $ w^{\top}_{q} x_0 - w^{\top}_{l} x_0>0$,  $ l \in [c], \, l \neq q$.  Thus,  $ \left\{\max\limits_{l \in [c],\ l \neq q}(w_l^{\top}r-w^{\top}_{q}r )\right\}> 0 $ must be satisfied, and $ \widetilde{P}_{T_x^{q}} \le  P_{T_x^{q}} \left(\left\{\max\limits_{l \in [c],\ l \neq q}(w_l^{\top}r-w^{\top}_{q}r )\right\}> 0\right)$. That is, there exists $ l \in [c],\ l \neq q $, such that $ (w_l-w_{q})^{\top} r> 0 $. 
	Let the cone generated by vectors $ w_l-w_{q}$, $ l \in [c] $, $ l \neq q $ denoted by
	$$ C_{q}= cone\{w_l-w_{q}, \, l \in [c], \,  l \neq q \},$$ and  the cone  is shown in the green area in  \cref{Figure/figure3_1} (for $ c=3 $ and $ q=1 $). 
	The possible $ r $  is given as follows
	\begin{equation*}\label{equu0}
	R_q 	\coloneqq \{ r \, |\,   \langle r, v \rangle>0, \  \exists \ v \in   C_{q} \}.
	\end{equation*}
	It is shown in the blue area in  \cref{Figure/figure3_2}. Obviously, when $ \widetilde{P}_{T_x^q}>0 $ for all $ q \in [c] $, we have $ \mathop{\cap}\limits_{q \in [c]} R_q= \emptyset $.
	Thus, uAP cannot fool all classes of data.
	%\begin{figure}[h] 		 		
	%	\centering 		\includegraphics[width=0.7\linewidth]{Figure/figure3} 		 		
	%	\caption{When $ c=5 $, $\hat{k}(x)=1 $, the  vectors $w_{\hat{k}(x)}$, $ w_l$ and the range of $ r $ in the blue area.} 		 		
	%	\label{fig:figure3} 	 	
	%\end{figure} 
	
%	\begin{figure}[H]
%		\centering
%		\subfloat[$C_q$ (the green area).]{
%			\label{Figure/figure3_1} 
%			\includegraphics[scale=0.35]{figure/figure3_1.eps}}
%		\hspace{0.1in} % 两图片之间的距离
%		\subfloat[$ R_q $ (the blue area).]{
%			\label{Figure/figure3_2} 
%			\includegraphics[scale=0.35]{figure/figure3_2.eps}}
%		\caption{When $ c=3 $, $q=1 $, the  vectors $w_{q}$, $ w_l$, $C_q$ and $ R_q $.}
%		\label{fig:figure3} 
%	\end{figure}

%\begin{figure}[tbhp]   \centering   \subfloat[$\epsilon_{\max}=5$]{\label{fig:a}\includegraphics{lexample_fig1}}   \subfloat[$\epsilon_{\max}=0.5$]{\label{fig:b}\includegraphics{lexample_fig2}}   \caption{Example figure using external image files.}   \label{fig:testfig} \end{figure}

\begin{tcbverbatimwrite}{tmp_\jobname_fig1.tex} 
\begin{figure}[H] 		\centering 		\subfloat[$C_q$ (the green area).]{ 			\label{Figure/figure3_1}  			\includegraphics[scale=0.35]{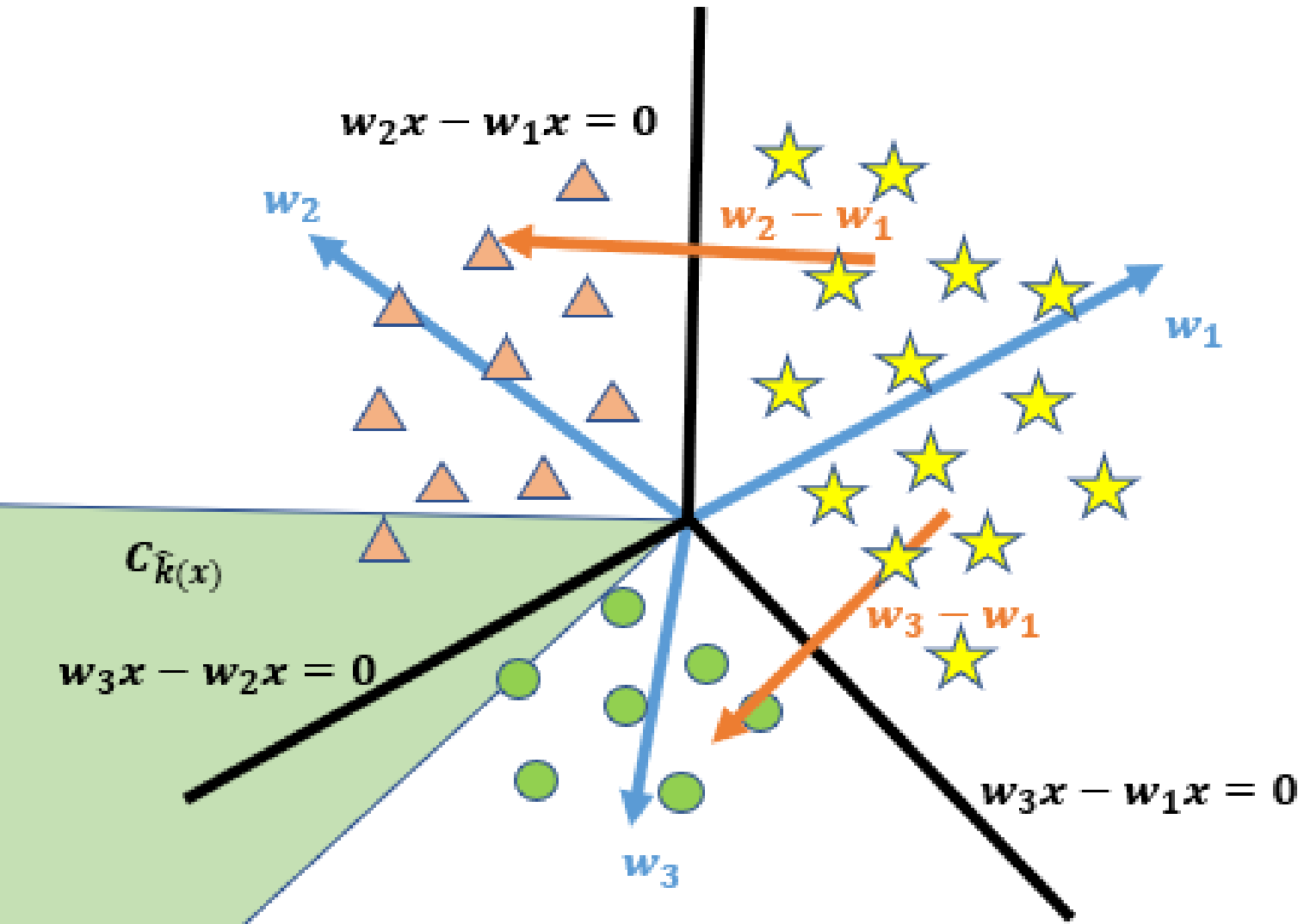}} 		\hspace{0.1in} % 两图片之间的距离 		
	\subfloat[$ R_q $ (the blue area).]{ 			\label{Figure/figure3_2}  			\includegraphics[scale=0.35]{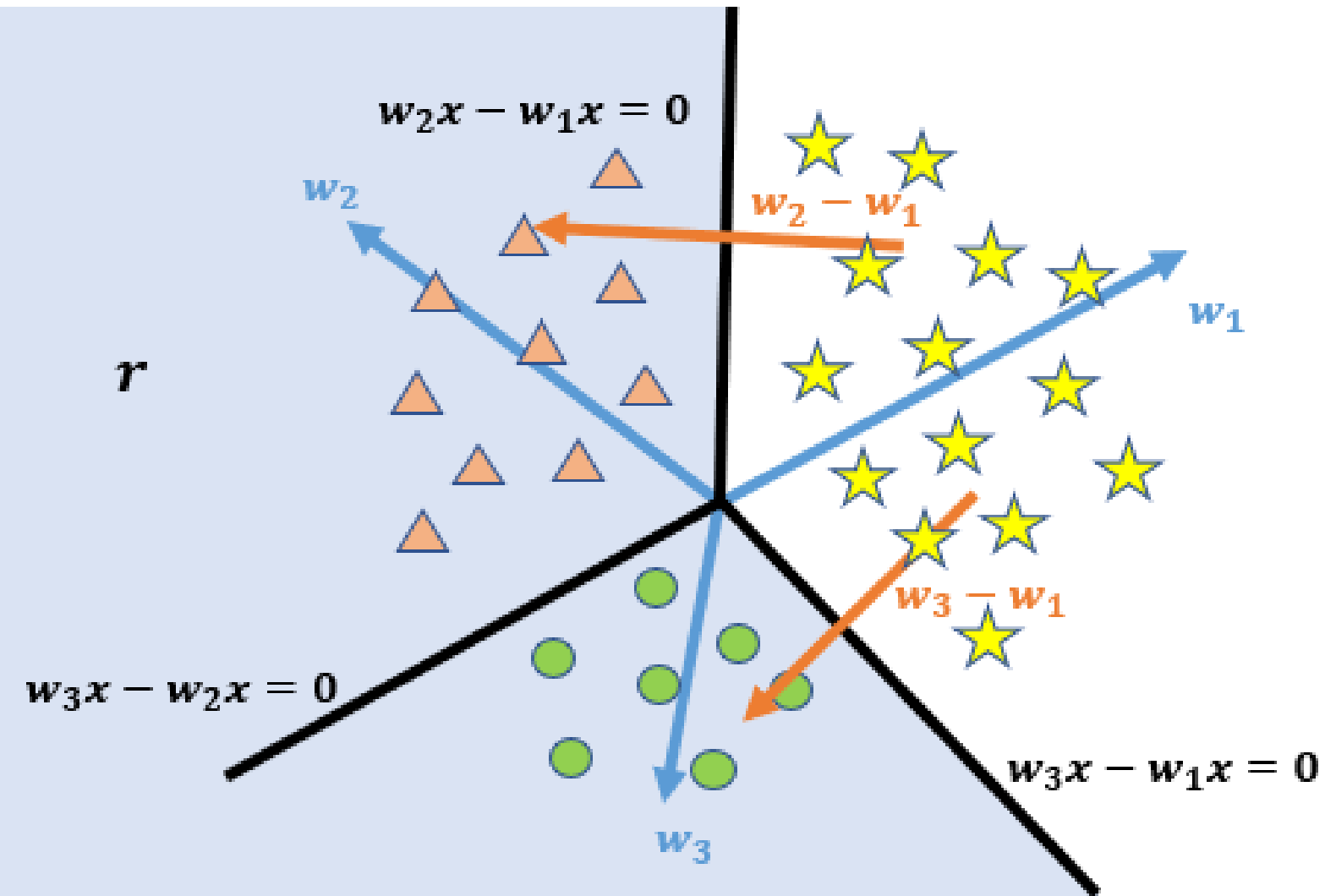}} 		\caption{When $ c=3 $, $q=1 $, the  vectors $w_{q}$, $ w_l$, $C_q$ and $ R_q $.} 		\label{fig:figure3}  	\end{figure}
\end{tcbverbatimwrite}
	
\input{tmp_\jobname_fig1.tex}

	Then we will prove that $ r $ can fool $ c-1 $ classes of data at the same time.
	For a particular class $ \hat{k}(x)=q$, we denote that polar cone of  $C_q$ as $R_q^c= \{ r \, |\,   \langle r, v \rangle \leq 0, \  \forall \ v  \in  C_q \}$.
	For $ R_q^c $, since $w^{\top}_q r \ge w^{\top}_l r$, for all $l \in [c]$, $ l \neq q $, we know that $ R_q^c  $
	is the range where the data of class $q$ is located.
	$ R_q $ is the range where other $ c-1 $ classes data except  class $ q$ are located.
	Suppose $ r \in   R_{\tilde{l}}^c$ for a fixed class $ \tilde{l} $,  so there is  $  \widetilde{P}_{T_x^{\tilde{l}}} \leq P_{T_x^{\tilde{l}}}\left(\left\{\max\limits_{l \in [c],\ l \neq {\tilde{l}}}(w_l^{\top}r-w^{\top}_{{\tilde{l}}}r )\right\}> 0\right)=0 $, i.e.,   the data of class $ {\tilde{l}} $ cannot be fooled.  
	In this case, for all $ q \in [c]$, $ q \neq {\tilde{l}}$,  we have $(w_{\tilde{l}} -w_{q})^{\top} r> 0$,
	and since $ r $ is sufficiently large,  then $ \widetilde{P}_{T_x^{q}}>0 $,  $ r \in R_{q} $. Further, we have $ r \in R_{\tilde{l}}^c  \subset  \mathop{\cap}\limits_{q \in [c],\, q \neq {\tilde{l}}} R_{q}$. That is, except for the data of class $ \tilde{l} $, other  data can be fooled. Thus, we know that uAP of  optimization problem  \cref{equ4} can fool $ c-1 $ class of data at most.

	%方向 大小
	% Therefore, the upper bound of  $ G_{\Omega, \hat{k}}$  can  reach $1-\theta_{l^*_u}$. 
	Since $ r $ is sufficiently large, we get that  all $ \widetilde{P}_{T_x^q} $, $q \in [c]$  can reach the upper bound.  Then to maximize \cref{equ3_15}, we select the label  $$ l^*_u =\mathop{\arg\min} \{ \theta_{l} \, | \, l \in [c] \}, $$  such that $ \widetilde{P}_{T_x^{l^*_u}} =0 $ and $ \widetilde{P}_{T_x^{q}} >0 $, $ q \in [c]$, $ q \neq {l^*_u}$, i.e., $ r \in R_{{l^*_u}}^c$, where $ R_{{l^*_u}}^c $  is the  distribution range of  data with class $ l^*_u$. The upper bound of $ G_{\Omega, \hat{k}} $ on the linear multi\mbox{-}classifier  can  reach $1-\theta_{l^*_u}$. Approximately, we choose the vector $ \dfrac{ w_{{l^*_u}}}{  \|w_{{l^*_u}}\|_2} $ that must belong to $ R_{{l^*_u}}^c $ as the direction of  uAP.
	Under the constraint \cref{equ3_4}, uAP of the linear multi\mbox{-}classifier on $ \Omega $ is given by \cref{the5}.	
\end{proof}  

%separating hyperplane
%{theorem7}{theorem3} 
%
%Through the results of Theorem 3 and theorem 2, we realize that for SVM model, the data after adding UAP always falls in the region of a class (determined by the classification hyperplane), so UAP can never deceive all data.
%
%Through the results of Theorem 3 and Theorem 2, we realize that for the SVM model, the data after adding uAP always falls into the area of one class (determined by the classification hyperplane), so using uAP can't fool all the data.

Through \cref{theorem3} and \cref{theorem7}, we realize that for SVM models, the data after adding uAP always falls into the region of a specific class (determined by the separating hyperplane), so uAP can not deceive all classes of data.

\section{Numerical Experiments}\label{Numerical Experiments.}
%In this section, we present our experimental results on the 10-class MNIST handwritten digit dataset \cite{ref17}. This dataset contains 55000 gray-scale images of dimension $ 28 \times 28 $ for training and 10000 for testing. Images are vectorized to 784-dimensional vectors. All classifiers are trained using the LIBSVM \cite{ref18} and LIBLINEAR \cite{ref19} implementation.
In this section, we conduct extensive numerical test to verify the efficiency of our method.  First, we introduce the datasets used in the experiment. The rest of the content is divided into two parts, and experiments are carried out on the adversarial perturbations in binary  and multiclass linear SVMs.
All experiments are tested in Matlab R2019b in Windows 10 on a HP probook440 G2 with an Intel(R) Core(TM) i5-5200U CPU at 2.20 GHz and of 8 GB RAM.
All classifiers are trained using the LIBSVM \cite{ref18} and LIBLINEAR \cite{ref19} implementation, which can be downloaded from https://www.csie.ntu.edu.tw/$ \sim $cjlin/libsvm and {https://www.csie.ntu.edu.tw/$ \sim $cjlin/liblinear}. Other recent progress in SVM can be found in \cite{ref63,ref62,ref61}.

We  test  adversarial perturbations against SVMs on MNIST \cite{ref17} and CIFAR-10 \cite{ref21} image classification datasets. MNIST and CIFAR-10 are currently the most commonly used datasets, their settings are as follows:   
\begin{itemize}
	\item MNIST: The complete MNIST dataset has a total of 60,000 training samples and 10,000 test samples, each of which is a vector of 784 pixel values and can be restored to a $ 28*28 $ pixel gray-scale handwritten digital picture. The value of the recovered handwritten digital picture ranges from 0 to 9, which exactly corresponds to the 10 labels of the dataset. 
	\item CIFAR-10: CIFAR-10 is a color image dataset closer to universal objects. The complete CIFAR-10 dataset has a total of 50,000 training samples an 10,000 test samples, each of which is a vector of 3072 pixel values and can be restored to a $ 32* 32 *3 $ pixel RGB color picture. There are 10 categories of pictures, each with 6000 images. The picture categories are airplane, automobile, bird, cat, deer, dog, frog, horse, ship and truck, their labels correspond to $ \{ 0, 1, 2, 3, 4, 5, 6, 7, 8, 9 \} $ respectively.
	%	\item  CIFAR-10: CIFAR-10 is a color image dataset closer to universal objects. The complete CIFAR-10 dataset has a total of 50,000 training samples and 10,000 test samples, each of which is a vector of 3072 pixel values ​​and can be restored to a $32 * 32 * 3 $ pixel RGB color picture. There are 10 categories of pictures, each with 6000 images. The picture categories are airplane, automobile, bird, cat, deer, dog, frog, horse, ship and truck, their labels correspond to $ \{ 0, 1, 2, 3, 4, 5, 6, 7, 8, 9 \} $ respectively. 
	%In the experiment in this section, due to the limitation of the computational complexity of building model, we only randomly select two-thirds of the total samples as the training samples, and the remaining samples as the test samples. 
	%In the experiment in this section, due to the limitation of the computational complexity of building model, we only randomly select 20,000 training samples and 3,000 test samples from the original dataset.
\end{itemize}

\subsection{Numerical experiments of binary linear SVMs}\label{Linear Binary Classification Numerical Experiments}

We present our experiments on the MNIST dataset and CIFAR-10 dataset on  binary linear classification model. For MNIST, we extract the data with class 0 and 1 to form a new binary dataset, with a total of 12,665 training data and 2,115 test data. For CIFAR-10, due to the limitation of the computational complexity of training model, we only select a part of data with class dog and truck  to form a new binary dataset,  with a total of 3,891 training data and 803 test data. 
First, we use LIBSVM to build a binary linear SVM on the training set, and obtain the parameters $ w $ and $ b $ of the classifier. 
Then we use formulas \cref{eq2}, \cref{equation5} and \cref{the1} to generate sAP, cuAP and uAP  respectively. 
Finally, we calculate $ G_{\Omega, \hat{k}} $ of the uAP.

\subsubsection{Numerical experiments of sAP}

In \cref{label_newfigure1}, we give an example to compare the original image, the image that has been misclassified after being attacked, and the  image of sAP on the MNIST dataset and CIFAR-10 dataset. By selecting the average of 10 repeated experiments on the same MNIST dataset, we get that the CPU time to train the binary classifier model is $ 1.37s $, and  the average CPU time to generate sAP is only $5.33 \times 10^{-4}s $. Similarly, on the CIFAR-10 dataset, we get the above CPU time as $ 53.49s $ and $ 3.04 \times 10^{-2}s $ respectively.
%在图~\ref{figure5_2_1}中，我们给出了一个例子，当MNIST数据集上的$| | r | | | u 2=2$时，比较原始图像、被攻击后被错误分类的图像和普遍对抗扰动的图像。通过选择相同MNIST测试数据集上10个重复实验的平均值，我们得到训练二元分类器模型的时间为$1.3658s$，而产生这种普遍对抗性扰动的时间仅为$1.06乘以10^{-5}s$。

\begin{tcbverbatimwrite}{tmp_\jobname_fig2.tex} 
\begin{figure}[bhtp]
	\centering
	\subfloat[On the MNIST dataset, the original image class is 0 and the perturbed  image class is 1.]{
		\includegraphics[width=0.45\textwidth,height=0.135\textwidth]{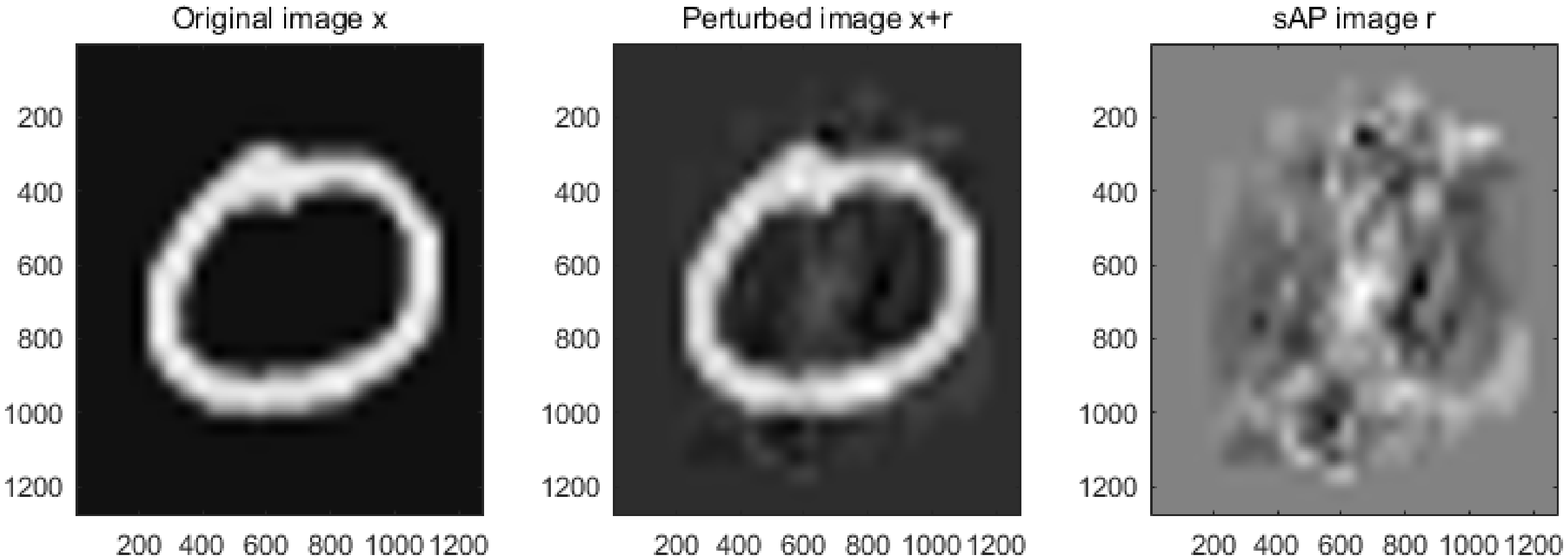}
		\label{label_newfigure1_1}
		%\caption{fig1}
	}
	\quad
	\subfloat[On the MNIST dataset, the original image class is 1 and the perturbed  image class is 0.]{
		\includegraphics[width=0.45\textwidth,height=0.135\textwidth]{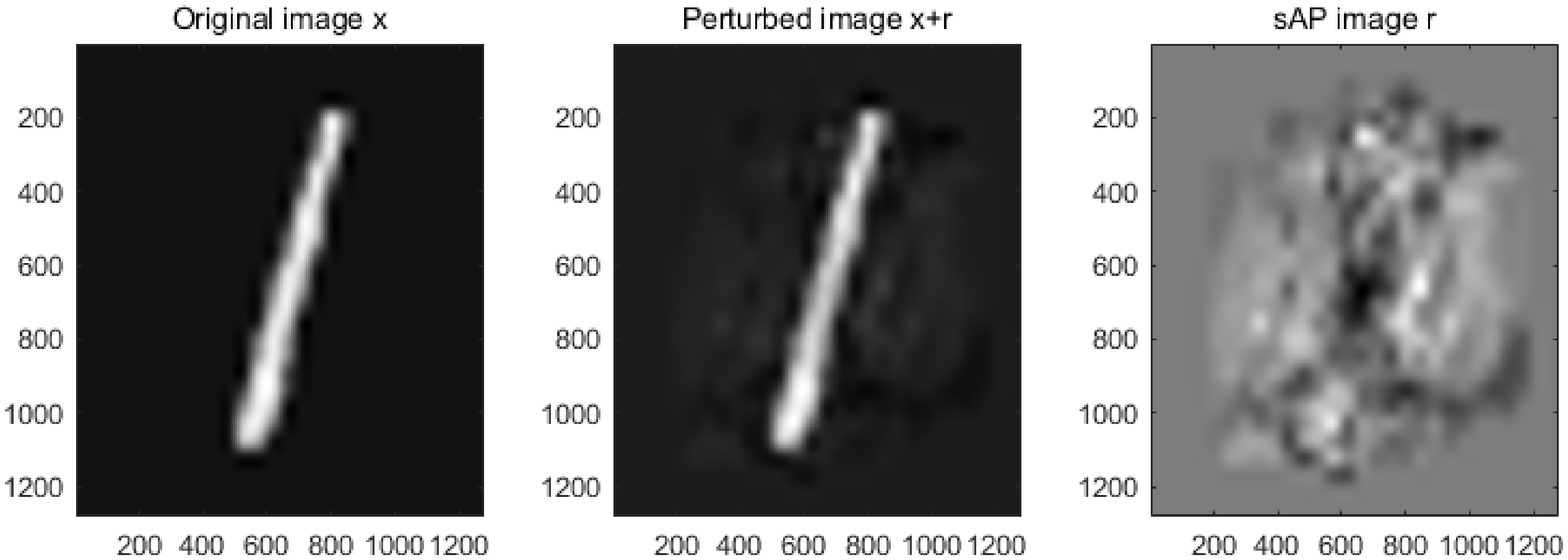}\label{label_newfigure1_2}
	}
	\quad
	\subfloat[On the CIFAR-10 dataset, the original image class is dog and the perturbed  image class is truck.]{
		\includegraphics[width=0.45\textwidth,height=0.135\textwidth]{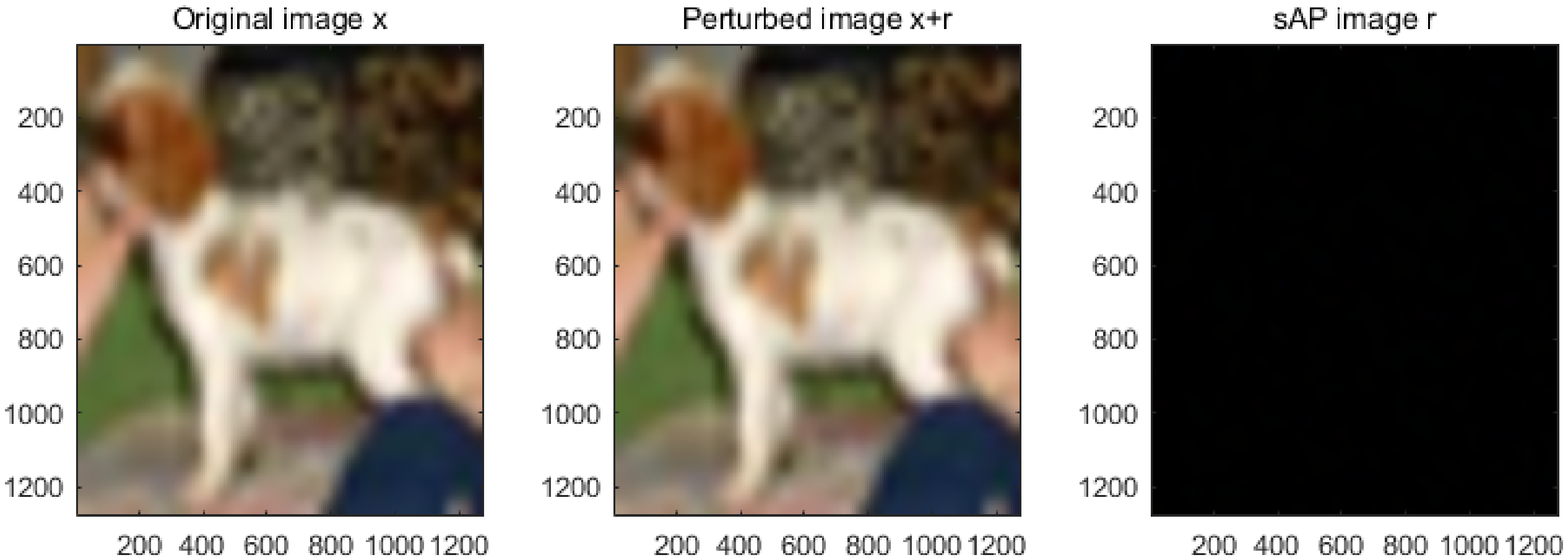}\label{label_newfigure1_3}
	}
	\quad
	\subfloat[On the CIFAR-10 dataset, the original image class is truck and the perturbed image class is dog.]{
		\includegraphics[width=0.45\textwidth,height=0.135\textwidth]{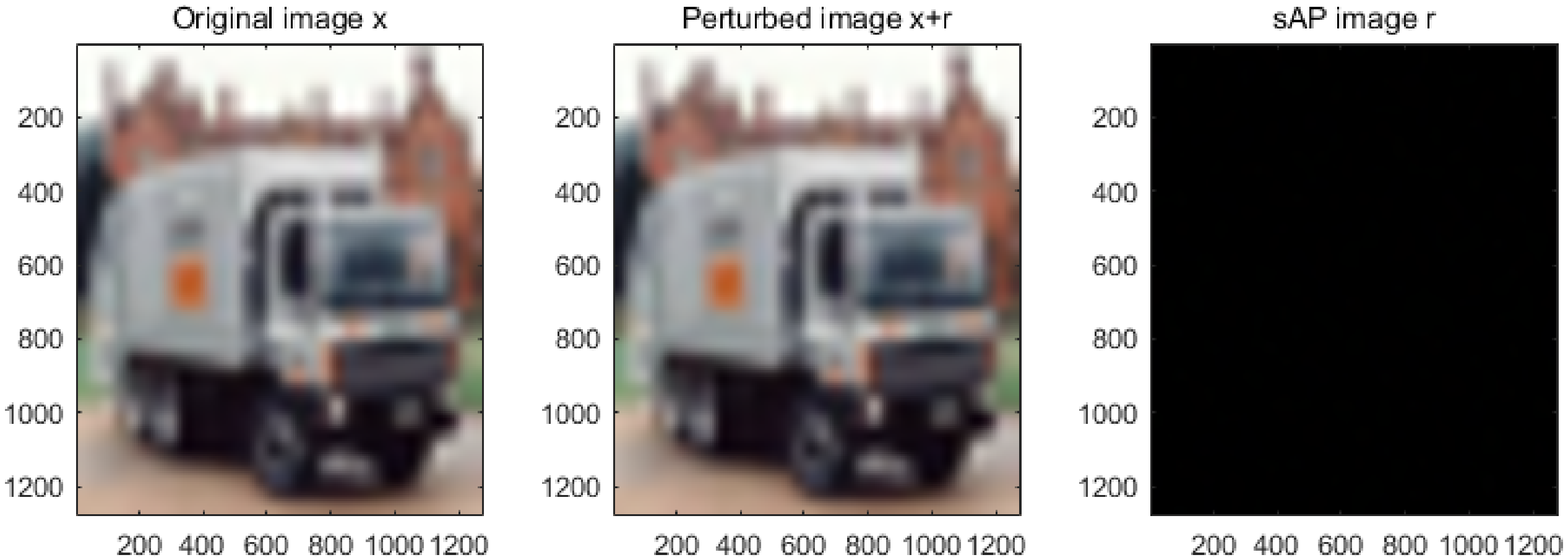}\label{label_newfigure1_4}
	}
	\caption{The original image, the image that has been misclassified after being attacked, and the  image of sAP on the MNIST and CIFAR-10 dataset.}
	\label{label_newfigure1}
\end{figure}
\end{tcbverbatimwrite}
\input{tmp_\jobname_fig2.tex}	
In \cref{label_newfigure1_1} and \cref{label_newfigure1_2}, we show the effect of adding sAP to handwritten digital images with original classes of 0 and 1, respectively. We obtain that the average norm of the data in the MNIST  dataset is 8.79, the average norm of sAP is 1.74, and the signal to noise ratio (SNR)  is 14.39. Although in human eyes, the classes of perturbed images do not change, but under the decision of the classifier, the classes of new images become 1 and 0 respectively. 
For linear binary SVM, the directions of sAP  of different classes of images are opposite, which can be reflected by the gray scale of sAP images in \cref{label_newfigure1_1} and \cref{label_newfigure1_2}.
%By comparing the gray scale of sAP images in Fig.~\ref{label_newfigure1}(a) and Fig.~\ref{label_newfigure1}(b), we can observe that for linear binary SVM, the directions of sAP  of different classes of images are opposite.
In \cref{label_newfigure1_3} and \cref{label_newfigure1_4}, we show the effect of adding sAP to the images in CIFAR-10 dataset  with original classes of dog and truck, respectively. The classes of perturbed images  become truck and dog.
We obtain that the average norm of the data in the CIFAR-10  dataset is 30.45, the average norm of sAP is 0.13, and SNR is 50.03.
Comparing \cref{label_newfigure1_1} and \cref{label_newfigure1_2} on the MNIST dataset and \cref{label_newfigure1_3} and \cref{label_newfigure1_4} on the CIFAR-10 dataset, we find that on the binary classification model, the sAP generated on the dataset (CIFAR-10) with a larger number of features are less likely to be observed, and the human eye can hardly perceive the sAP image.

\subsubsection{Numerical experiments of cuAP}

In \cref{label_newfigure2}, we illustrate the effect of adding cuAP to the dataset when $ \Omega $ is selected as the data with class 0 of MNIST dataset. 
The class of all data in \cref{label_newfigure2_1} is 0. 
If they are attacked by the same cuAP shown in \cref{label_newfigure2_2}, the perturbed new images will be displayed in the corresponding position in  \cref{label_newfigure2_3}, and the classifier will misclassify them as the number 1. In this case, the  norm of cuAP is 2, and SNR is 12.57.
The CPU time for calculating the cuAP in \cref{label_newfigure2_2} is only $1.06 \times 10^{-4}s $. Because the process of generating cuAP does not need iteration.  
\begin{tcbverbatimwrite}{tmp_\jobname_fig4.tex} 
\begin{figure}[bhtp]
	\centering
	\subfloat[The original images with class  0.]
	{
		\begin{minipage}[b]{.23\linewidth}
			\centering
			\includegraphics[scale=0.12]{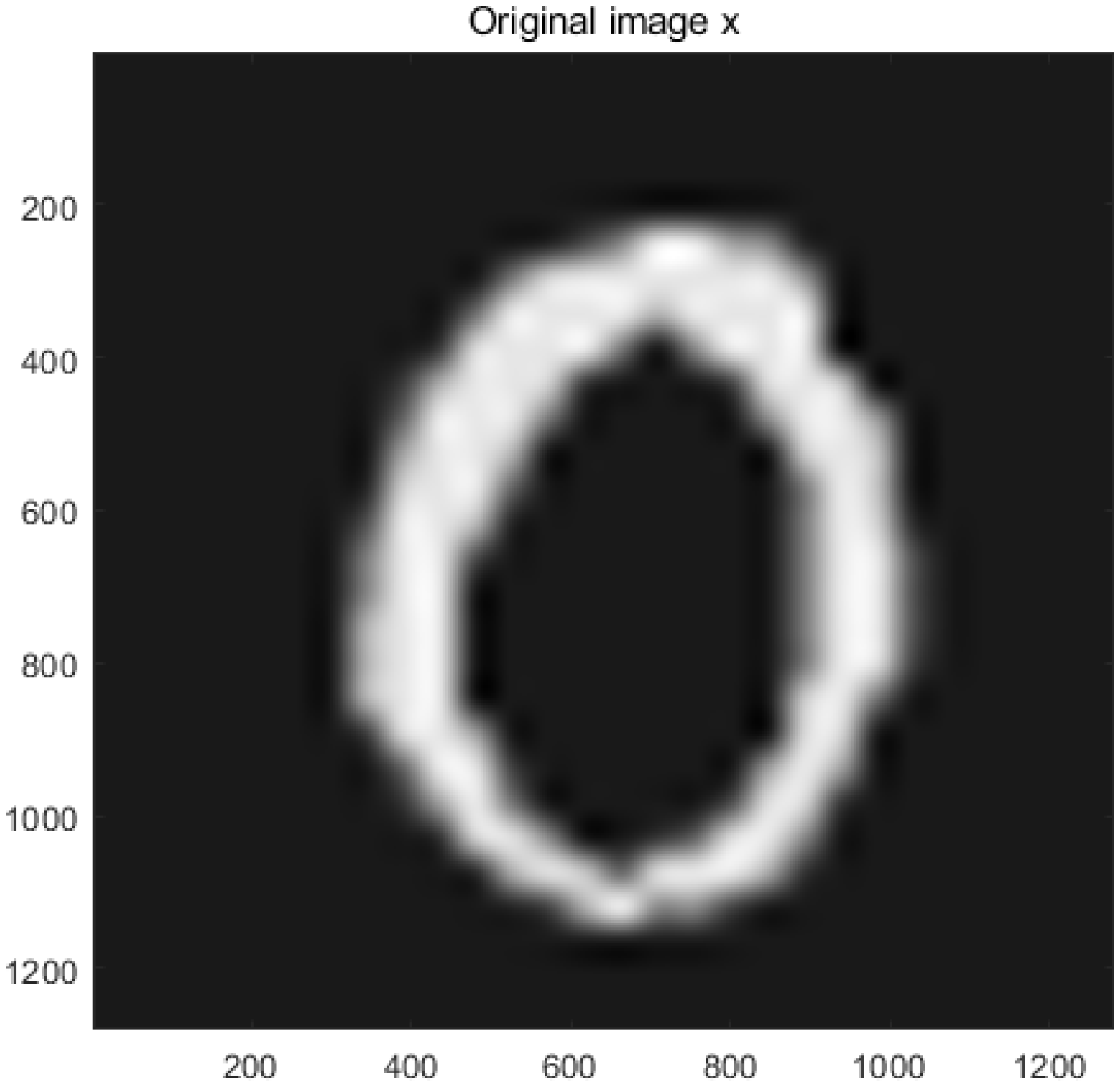} \\
			\includegraphics[scale=0.12]{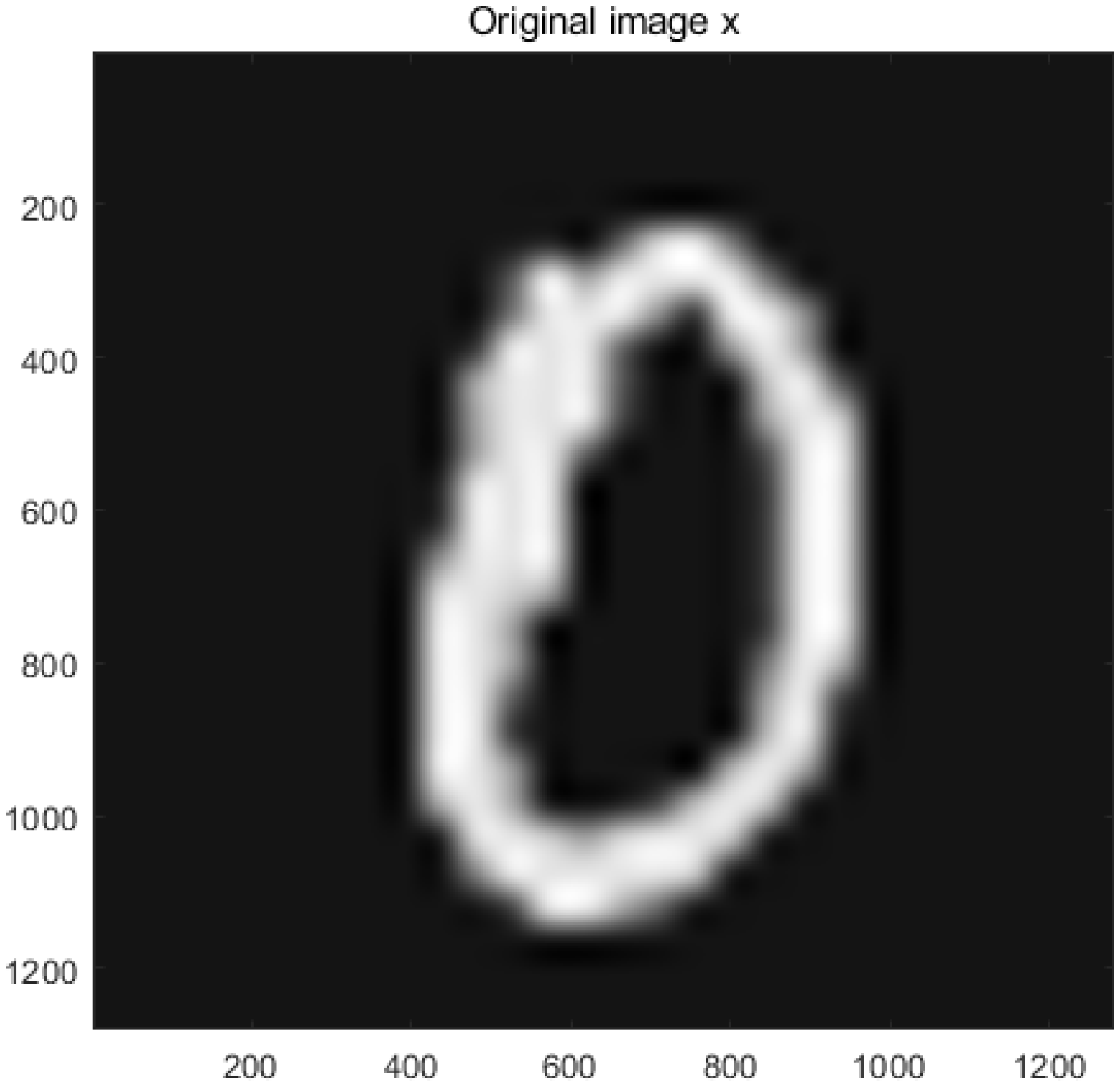} \\
			\includegraphics[scale=0.12]{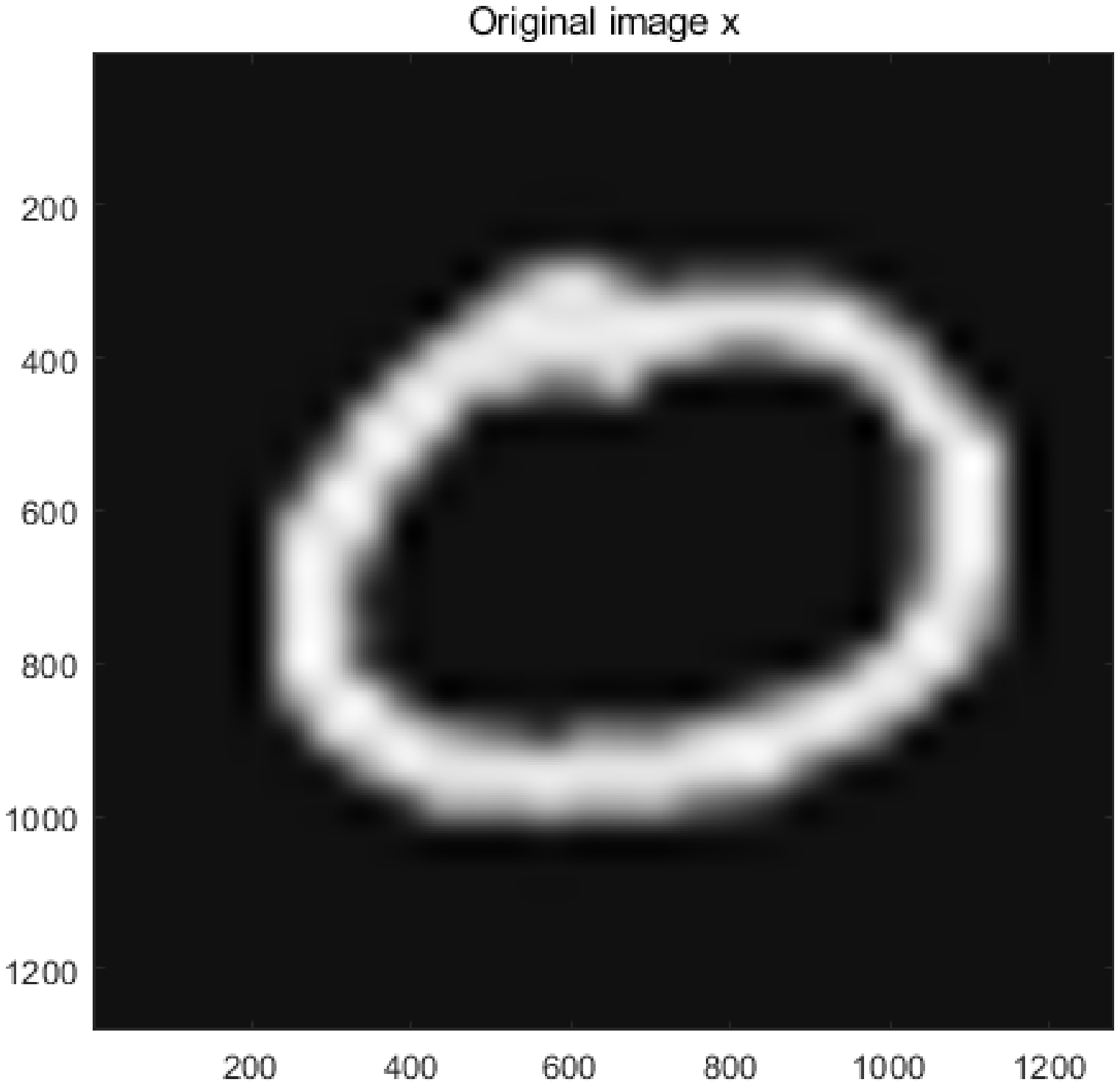}
		\end{minipage}
	\label{label_newfigure2_1}
	}
	\subfloat[The image of the unique cuAP on  the subset of MNIST dataset with class 0.]
	{
		\begin{minipage}[b]{.3\linewidth}
			\centering
			\includegraphics[scale=0.15]{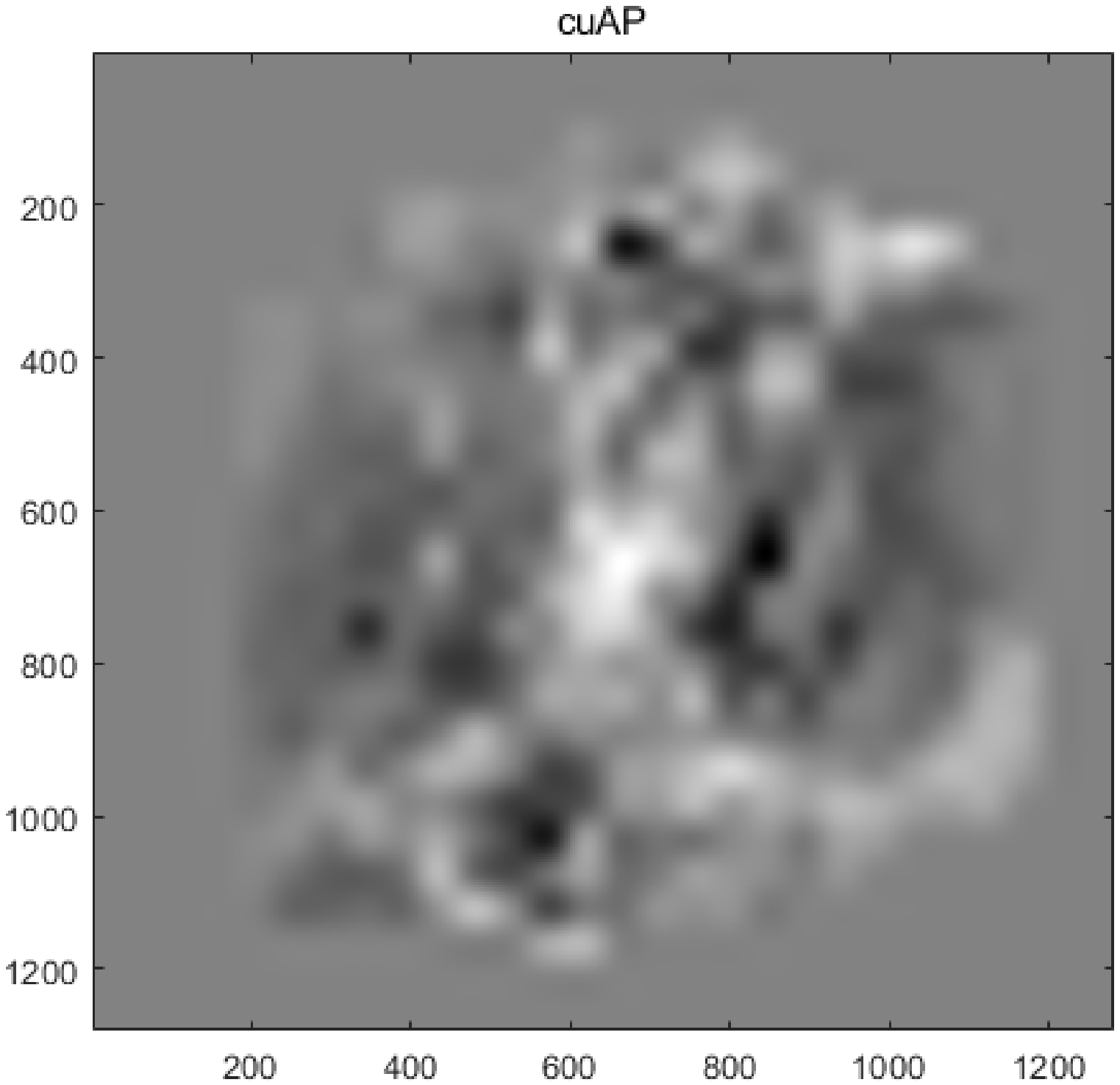}\vspace{20mm}
		\end{minipage}
	\label{label_newfigure2_2}
	}
	\subfloat[The  perturbed  images with class 1.]
	{
		\begin{minipage}[b]{.23\linewidth}
			\centering
			\includegraphics[scale=0.12]{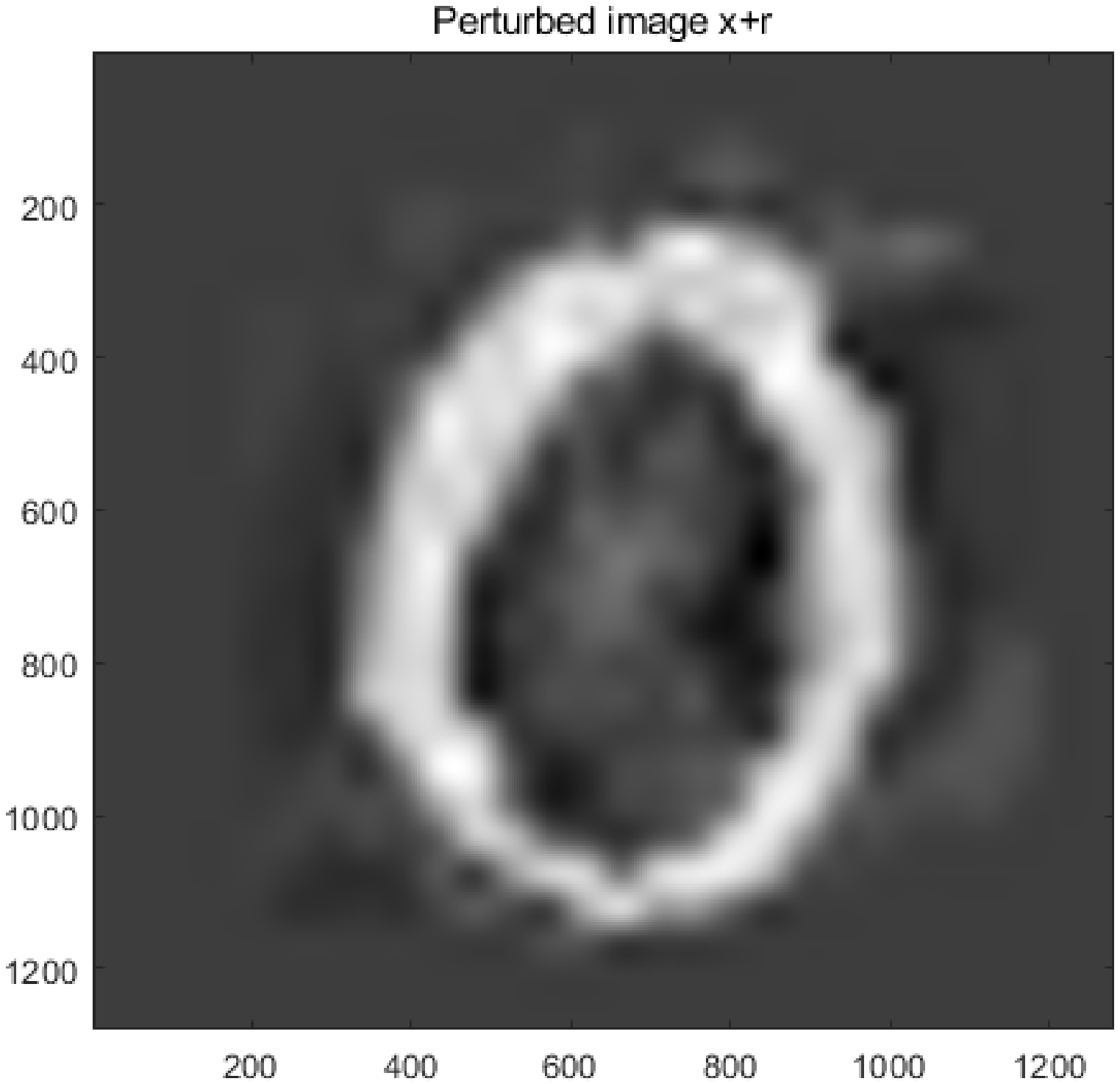} \\
			\includegraphics[scale=0.12]{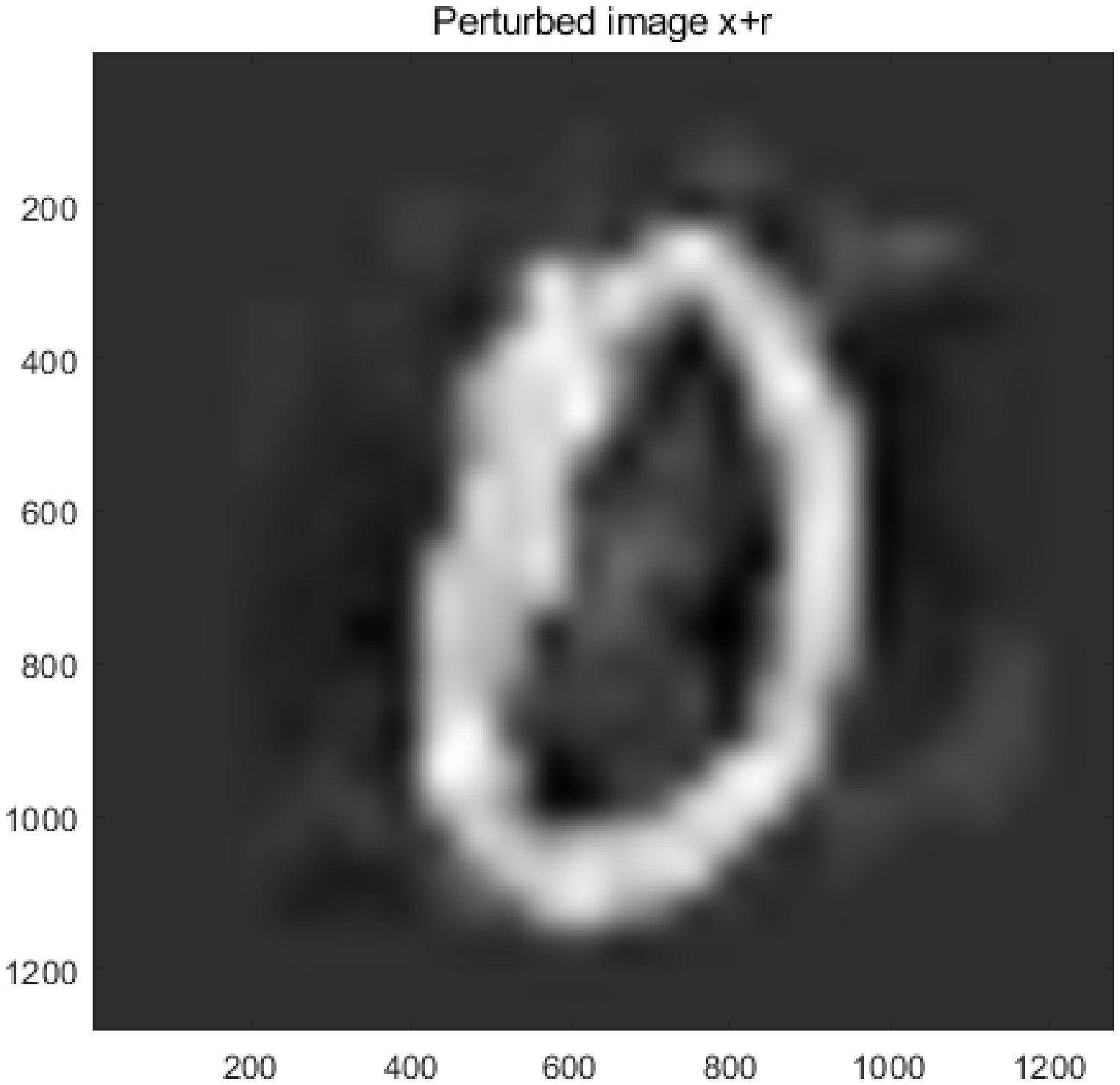} \\
			\includegraphics[scale=0.12]{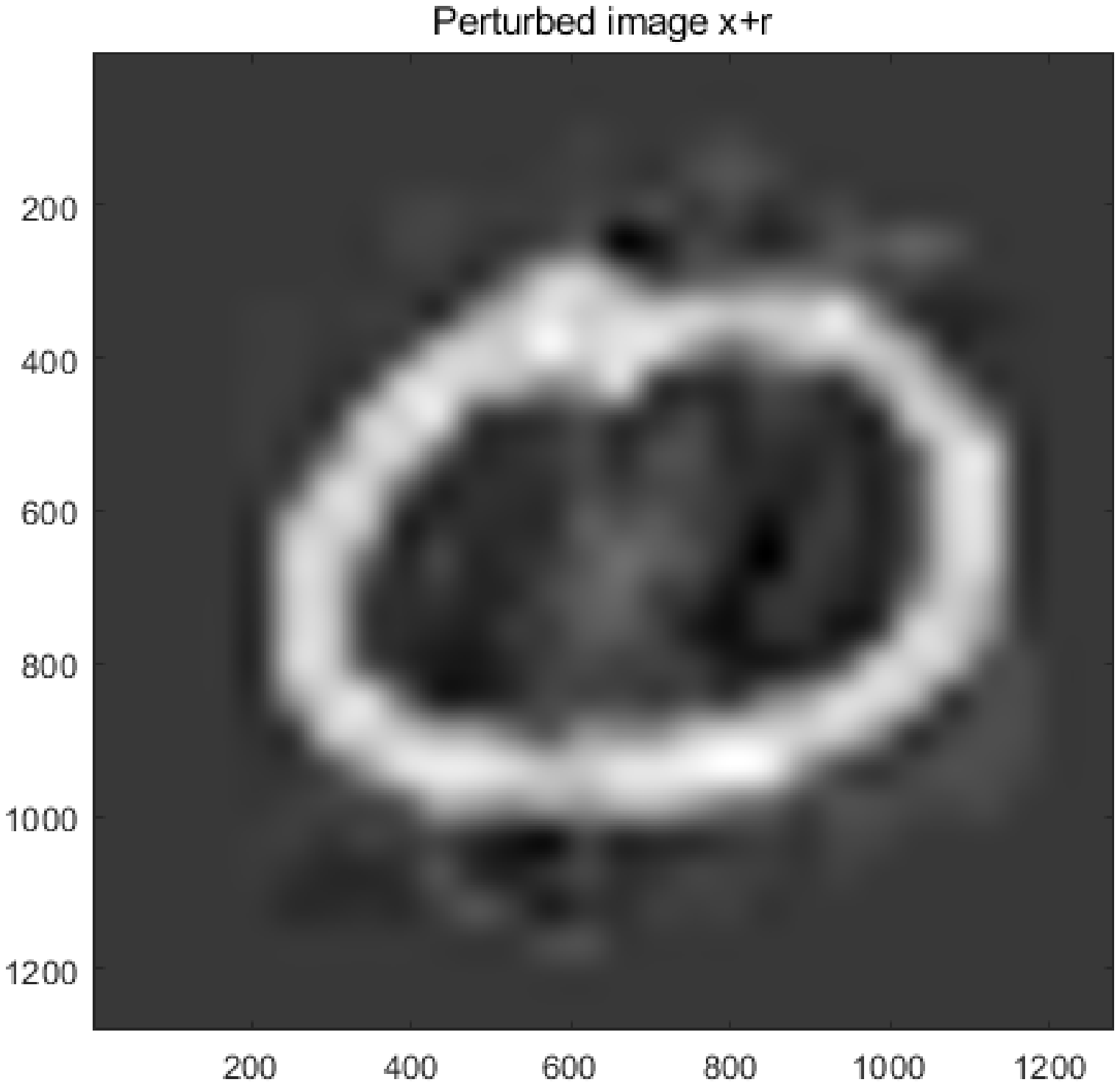}
		\end{minipage}
	\label{label_newfigure2_3}
	}
	\caption{The original images, the  image of cuAP, and the images that have been misclassified after being attacked, when $ \xi=2 $ on the MNIST dataset.}
	\label{label_newfigure2}
\end{figure}
\end{tcbverbatimwrite}
\input{tmp_\jobname_fig4.tex}
\begin{figure}[bhtp]
	\centering
	\subfloat[The original images with class truck.]
	{
		\begin{minipage}[b]{.23\linewidth}
			\centering
			\includegraphics[scale=0.12]{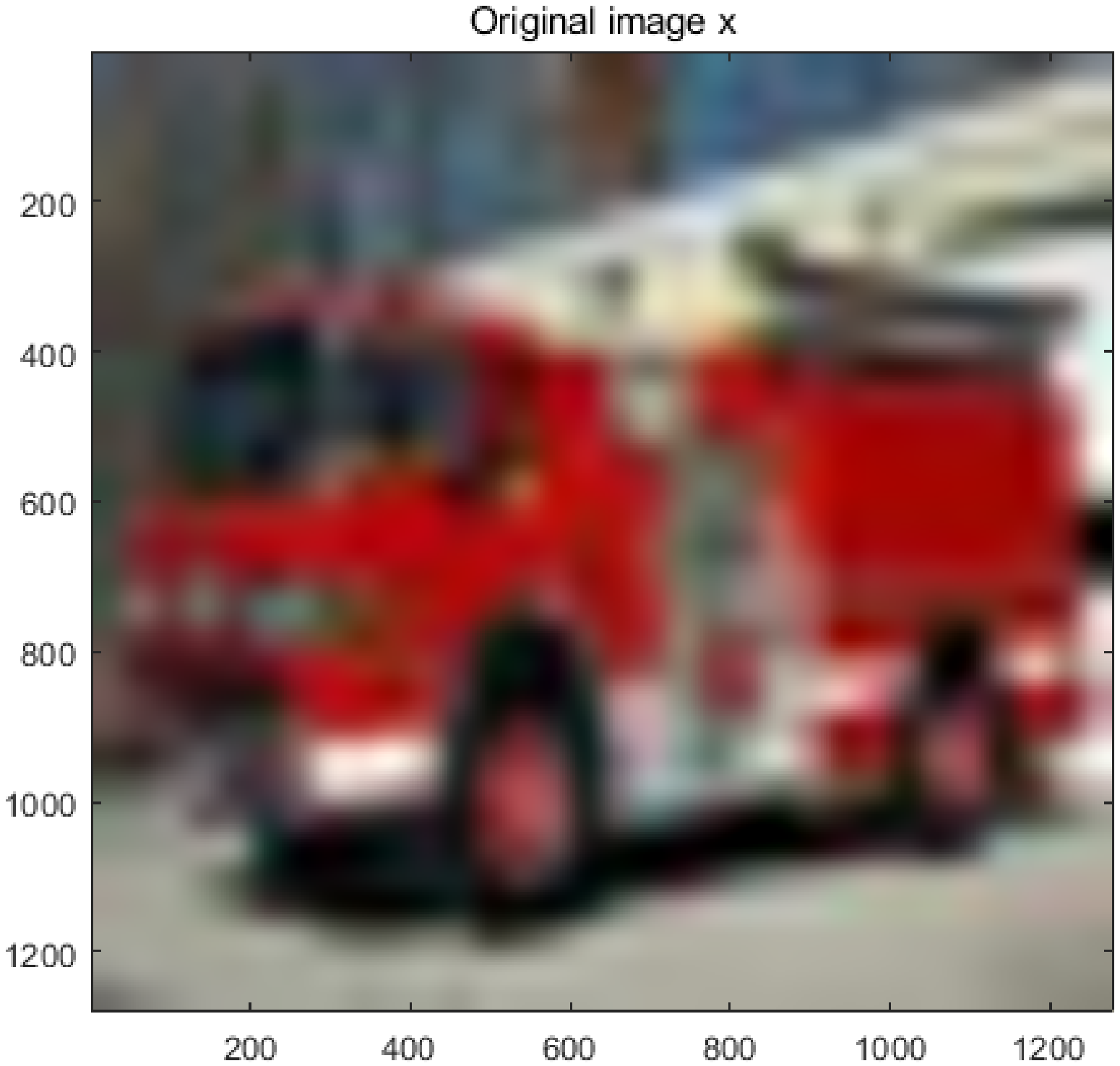} \\
			\includegraphics[scale=0.12]{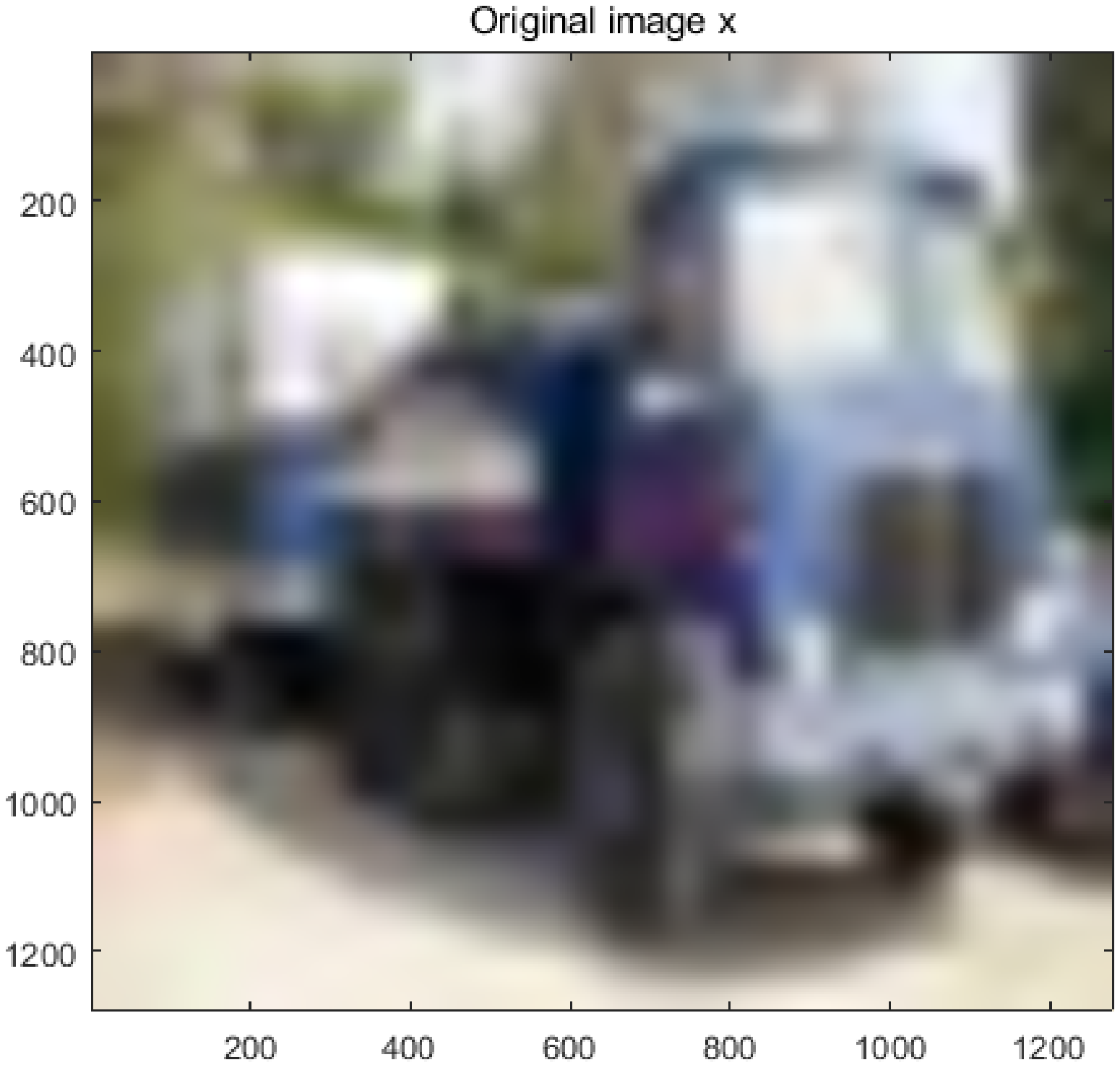} \\
			\includegraphics[scale=0.12]{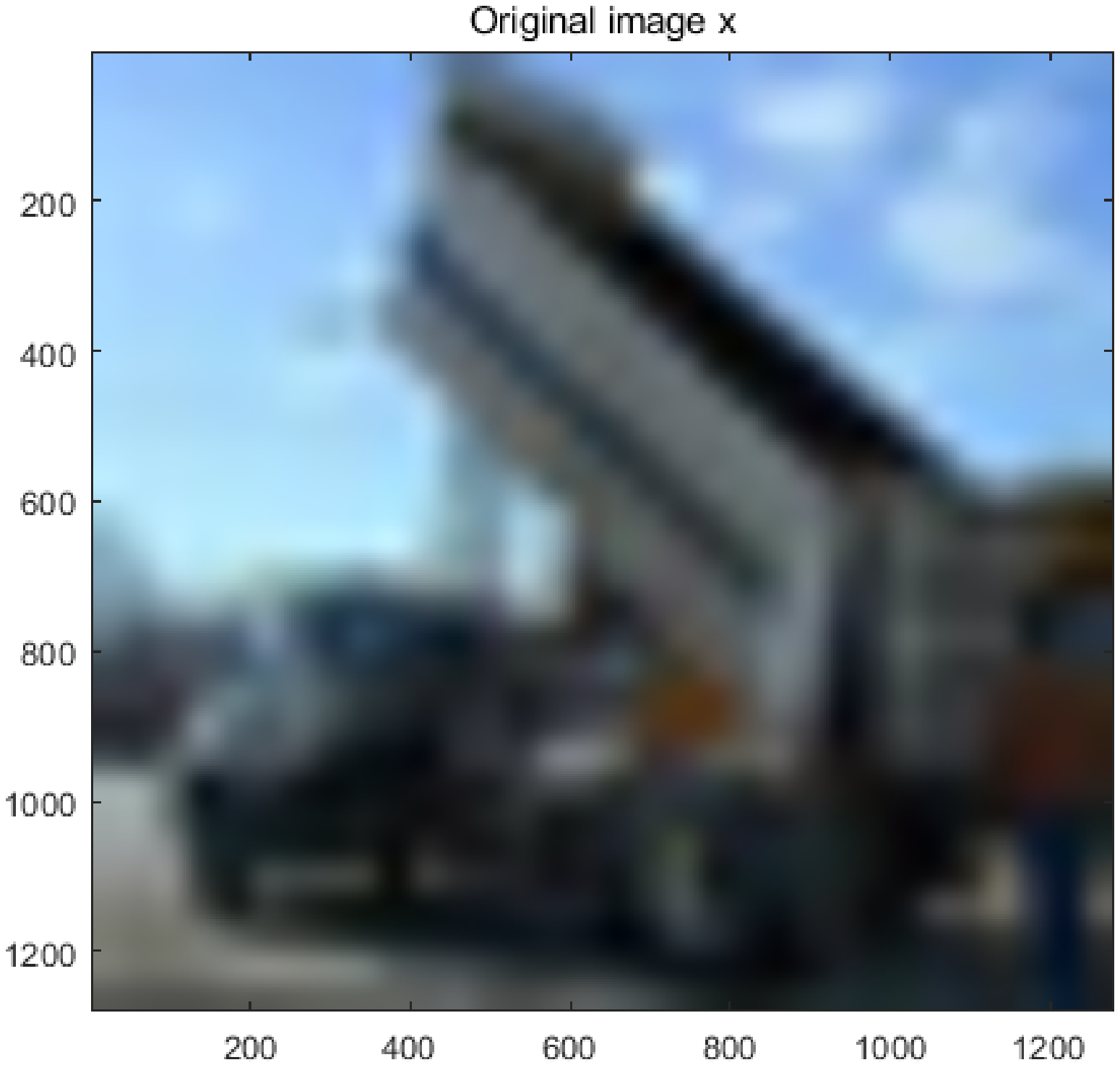}
		\end{minipage}
	}
	\subfloat[The image of the unique cuAP on  the subset of CIFAR-10 dataset with class truck.]
	{
		\begin{minipage}[b]{.3\linewidth}
			\centering
			\includegraphics[scale=0.15]{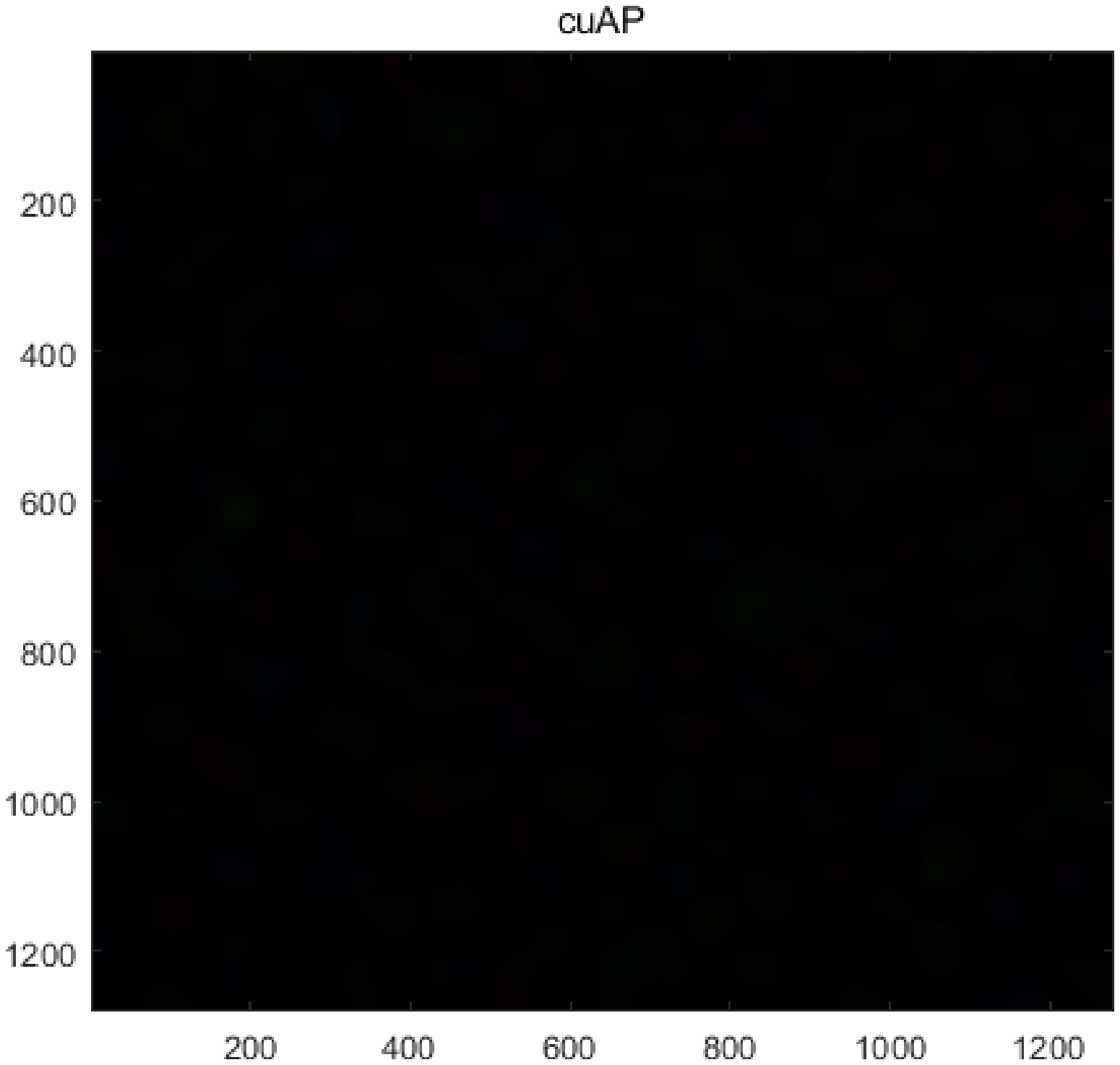}\vspace{20mm}
		\end{minipage}
	}
	\subfloat[The  perturbed  images with class dog.]
	{
		\begin{minipage}[b]{.23\linewidth}
			\centering
			\includegraphics[scale=0.12]{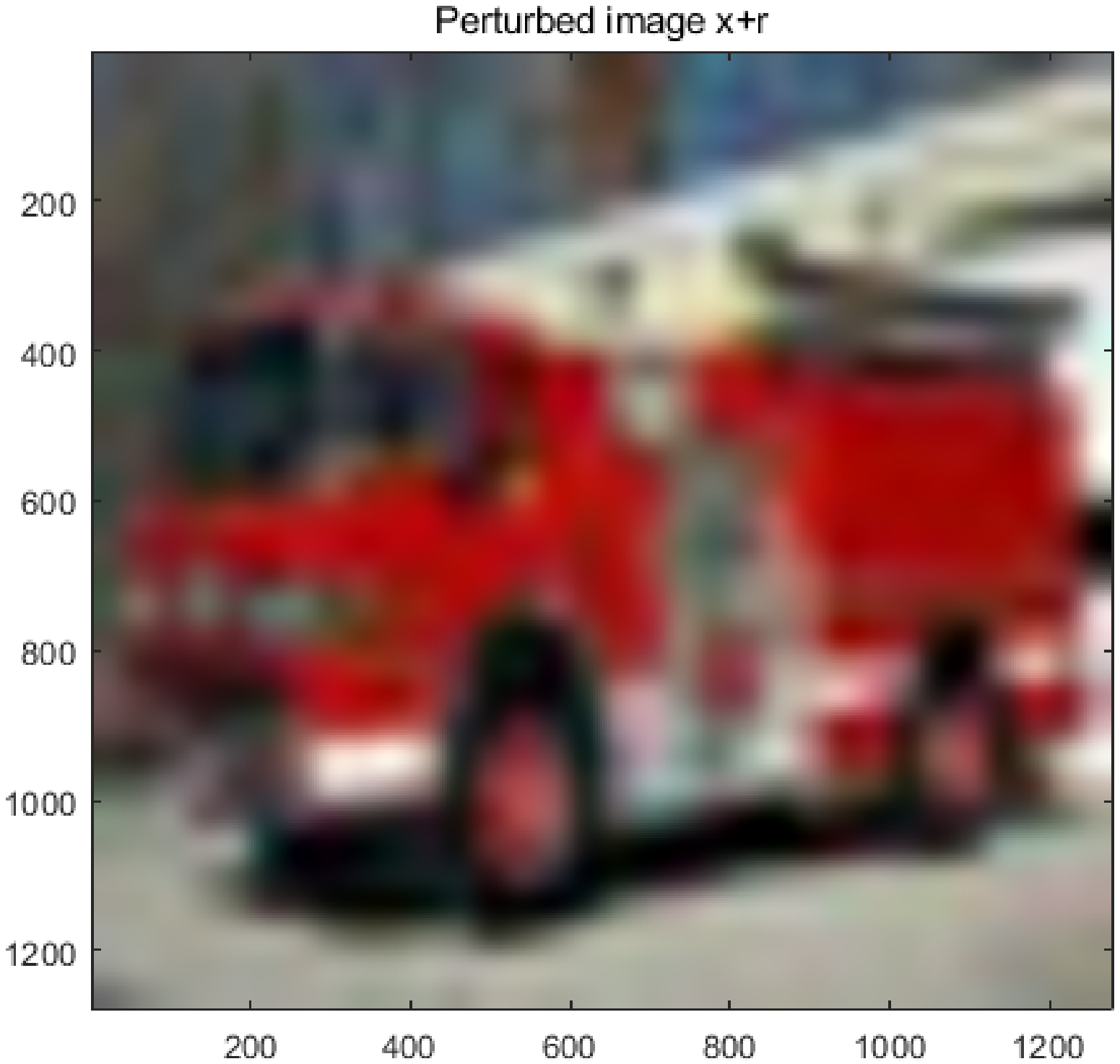} \\
			\includegraphics[scale=0.12]{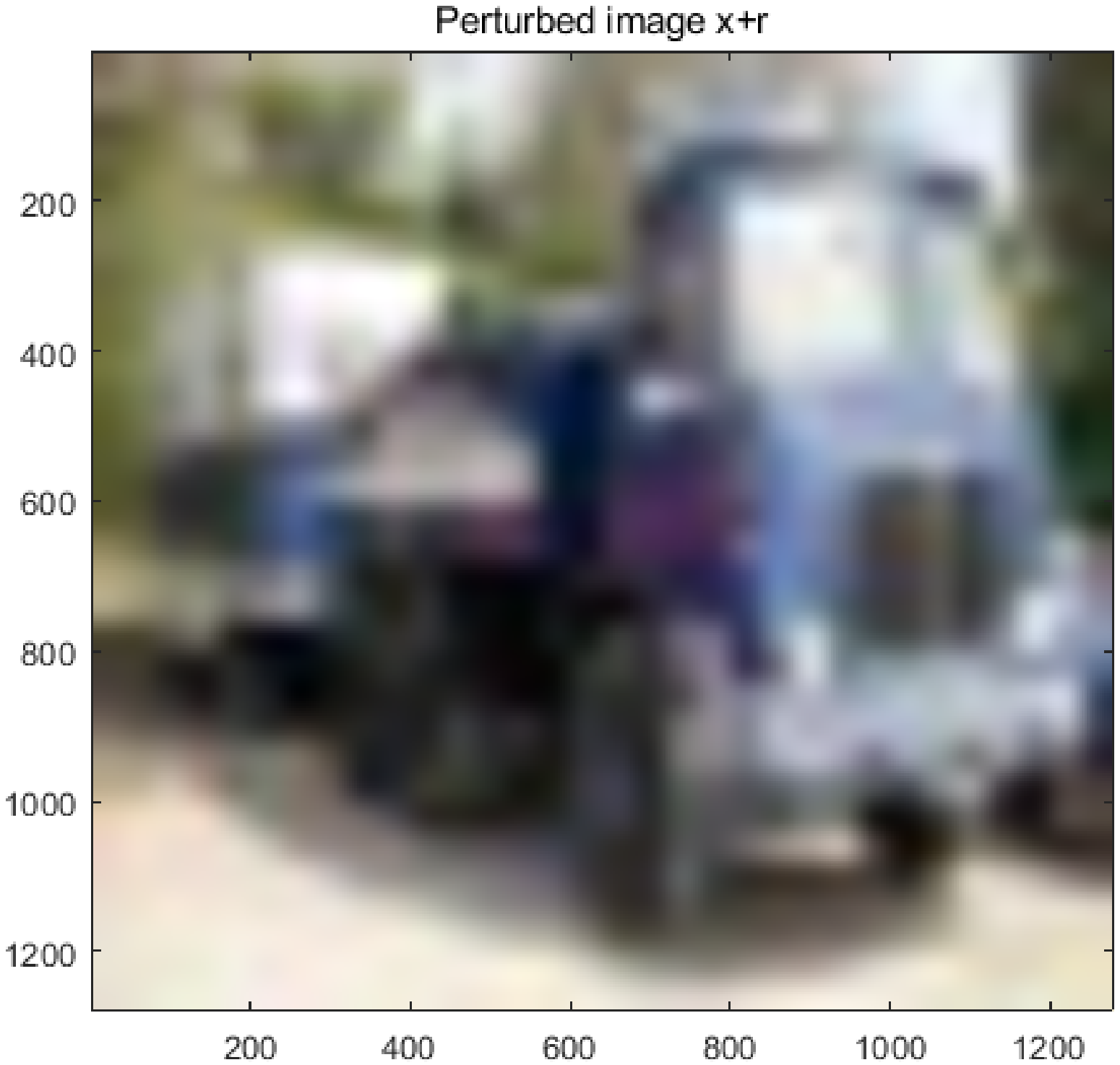} \\
			\includegraphics[scale=0.12]{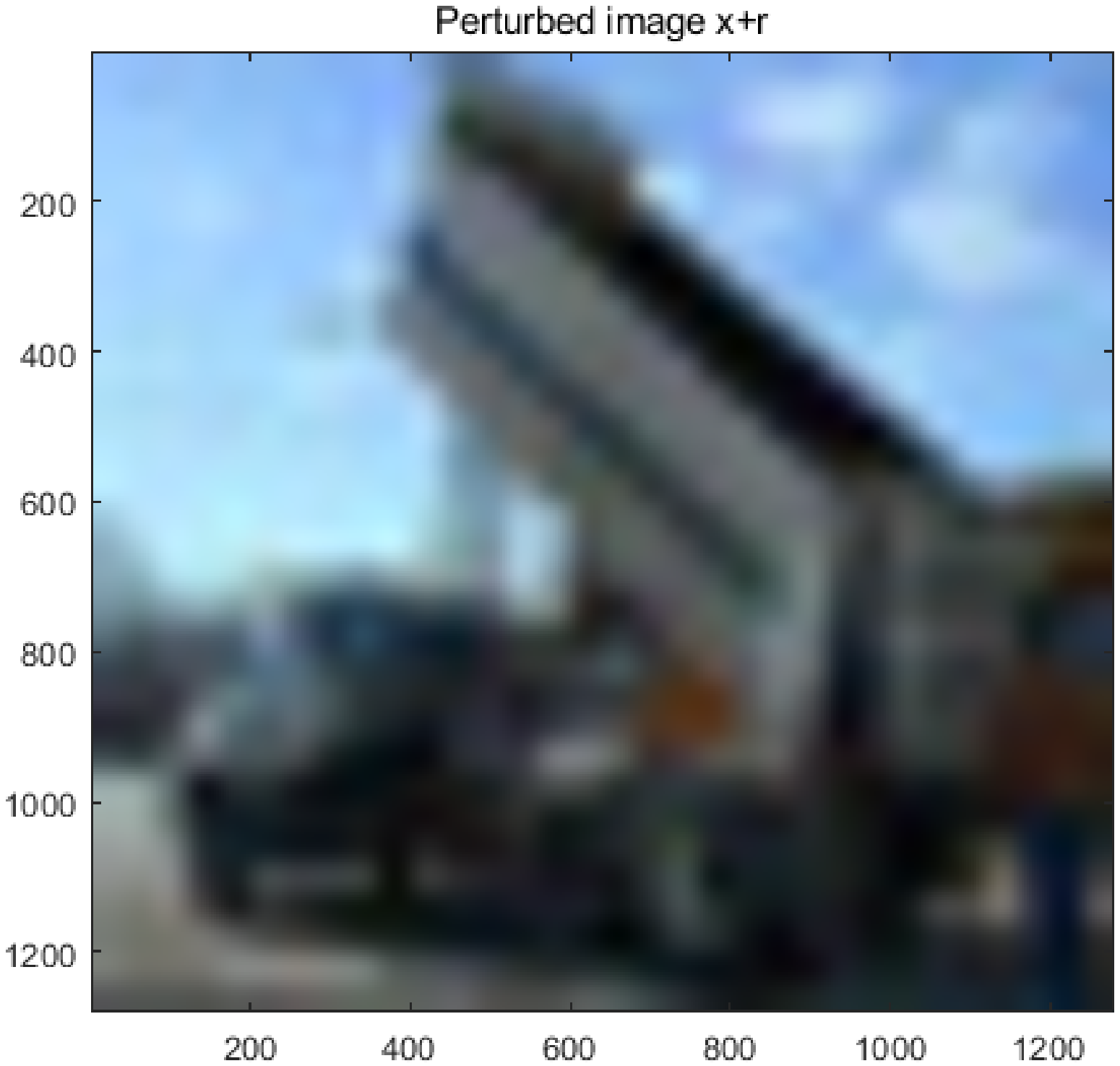}
		\end{minipage}
	}
	\caption{The original images, the  image of cuAP, and the images that have been misclassified after being attacked, when $ \xi=0.5 $ on the CIFAR-10 dataset.}
	\label{label_newfigure3}
\end{figure}

Similarly, in \cref{label_newfigure3}, we give the original images with class truck, the image of cuAP and the perturbed images with class dog, when $ \xi=0.5 $ on the CIFAR-10  dataset. We get that the CPU time to generate cuAP is $7.25 \times 10^{-3}s $ and SNR is 35.55.
Comparing \cref{label_newfigure2} and \cref{label_newfigure3}, we find that  cuAP generated on the dataset (CIFAR-10) with a larger number of features is less likely to be observed, and the perturbed images are almost the same as the original images by human observations.

\subsubsection{Numerical experiments of uAP}

%
%uAP扰动更高比例的
%与cuAP相同
\cref{label_newfigure4} shows the relationship that the fooling rate  with  the size of uAP on the MNIST and CIFAR-10 real image data test set. Obviously, in \cref{fig:figure4_1_0}, we find that when the  norm of uAP on the MNIST reaches 3, the fooling rate is almost $ 53.66\% $ and in \cref{fig:figure4_1_1}, when the  norm of uAP on the CIFAR-10  reaches 0.5, the fooling rate is almost $ 52.93\% $. Comparing the results of different datasets on the same linear  binary classifier in \cref{label_newfigure4}, we find that  CIFAR-10  dataset is more likely  to be fooled by very small uAP. 
\begin{figure}[bhtp]
	\centering
	\subfloat[The relationship  on the MNIST dataset.]{
		\label{fig:figure4_1_0} 
		\includegraphics[scale=0.2]{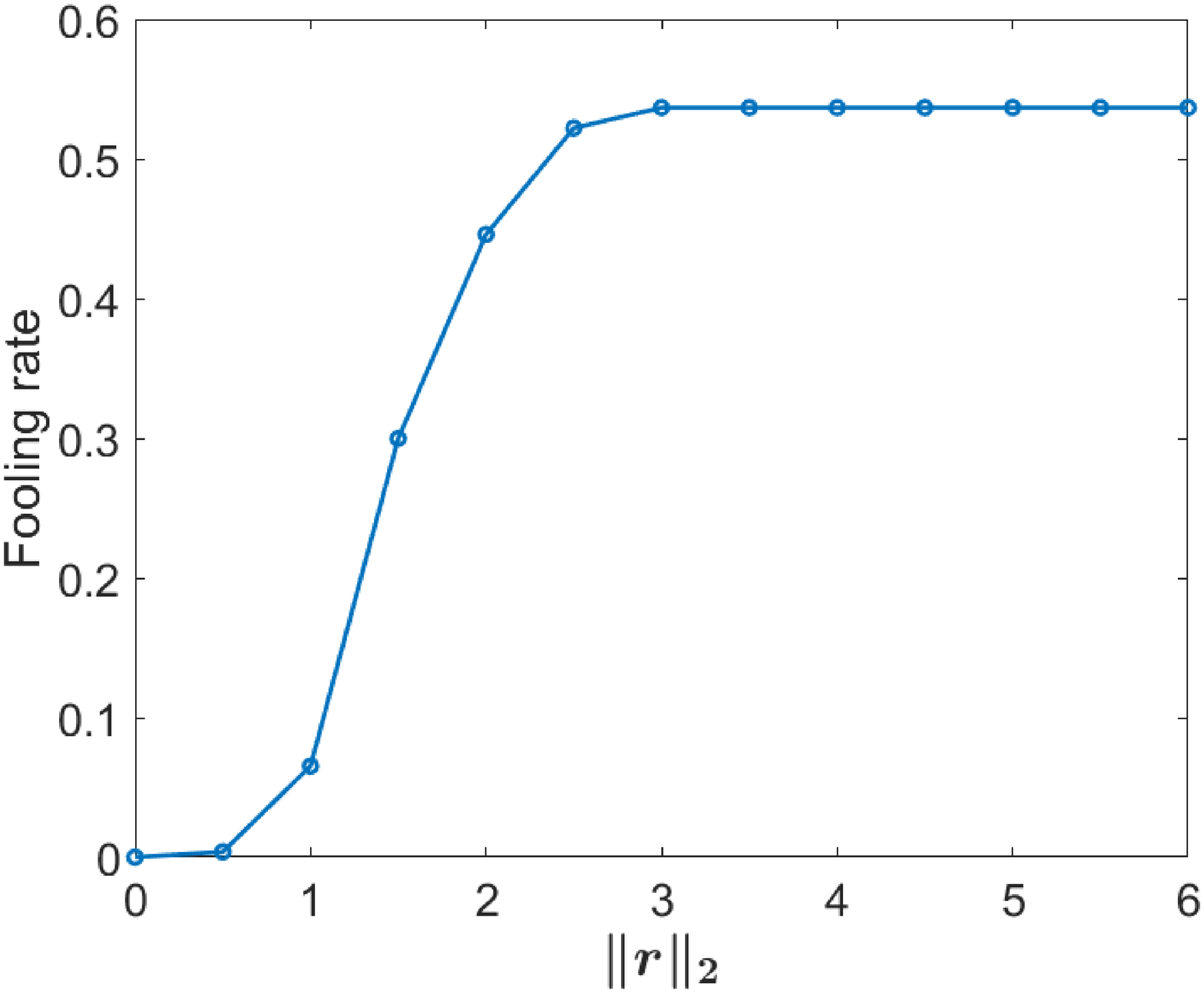}}
	\hspace{0.1in} % 两图片之间的距离
	\subfloat[The relationship  on the CIFAR-10 dataset.]{
		\label{fig:figure4_1_1} 
		\includegraphics[scale=0.2]{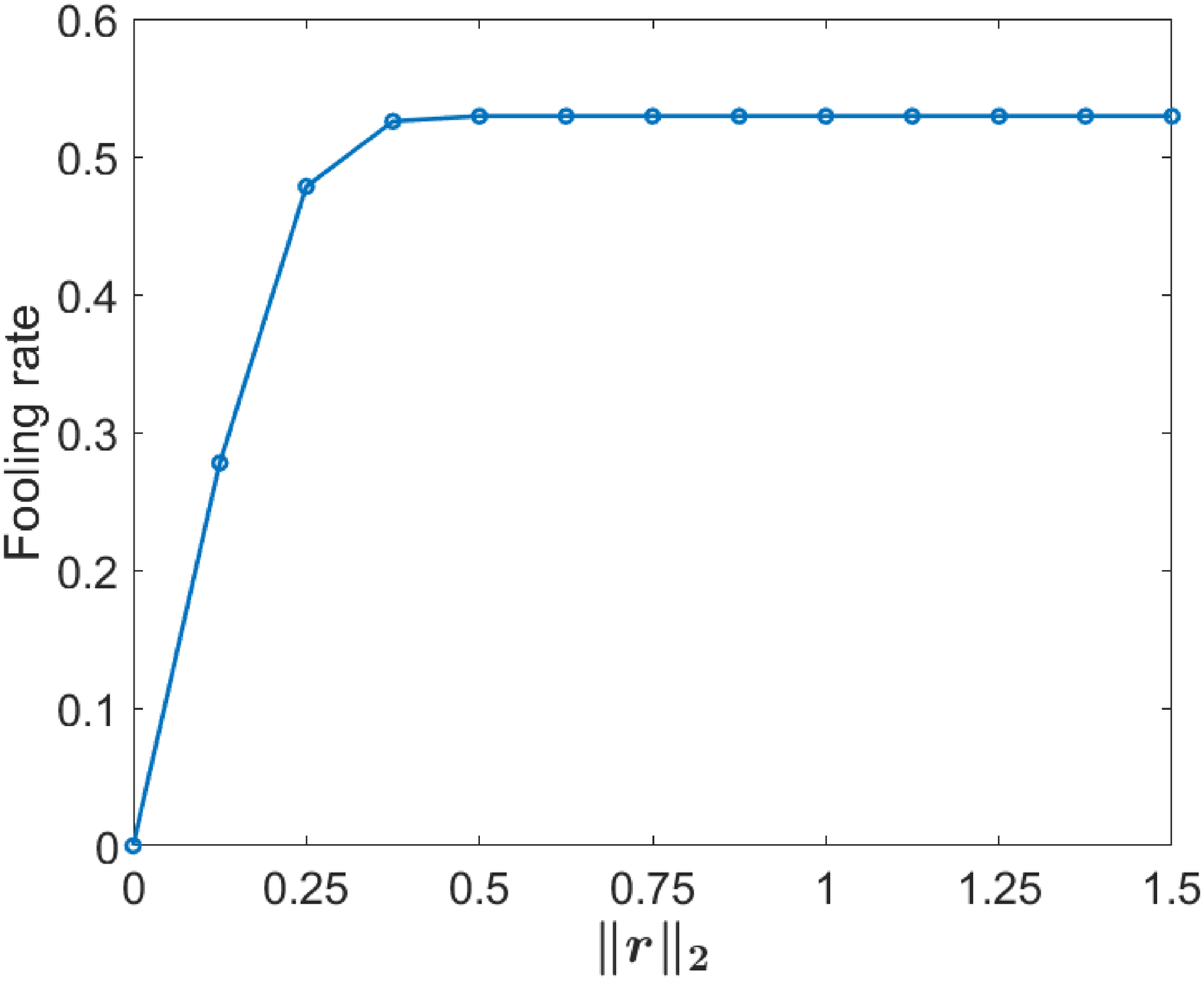}}
	\caption{The relationship between the fooling rate and the norm  of uAP  obtained on the linear  binary SVMs on the MNIST and  CIFAR-10 dataset.}
	\label{label_newfigure4} 
\end{figure}
\begin{figure}[H]
	\centering
	\subfloat[The original images with class  1.]
	{
		\begin{minipage}[b]{.23\linewidth}
			\centering
			\includegraphics[scale=0.12]{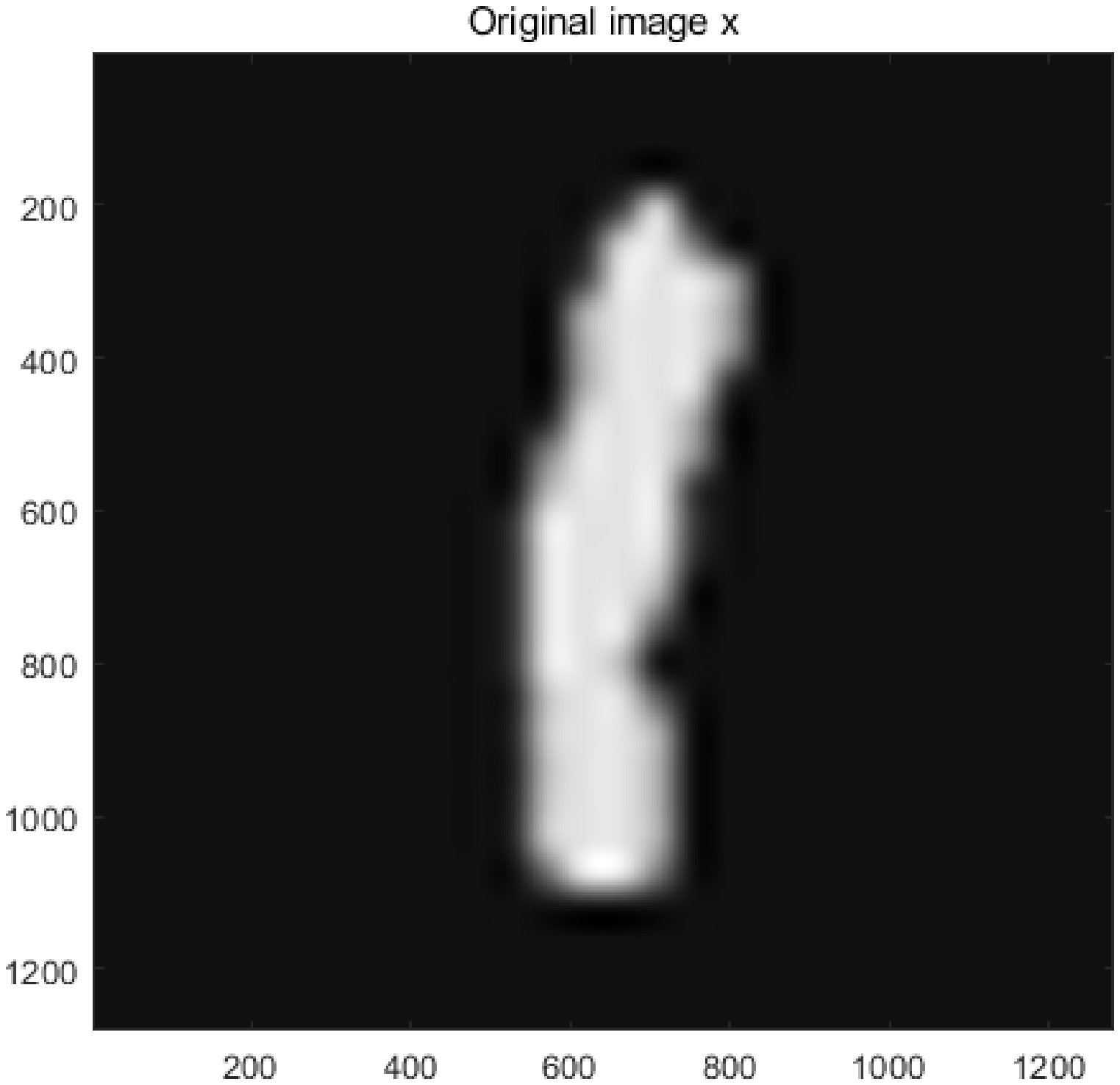} \\
			\includegraphics[scale=0.12]{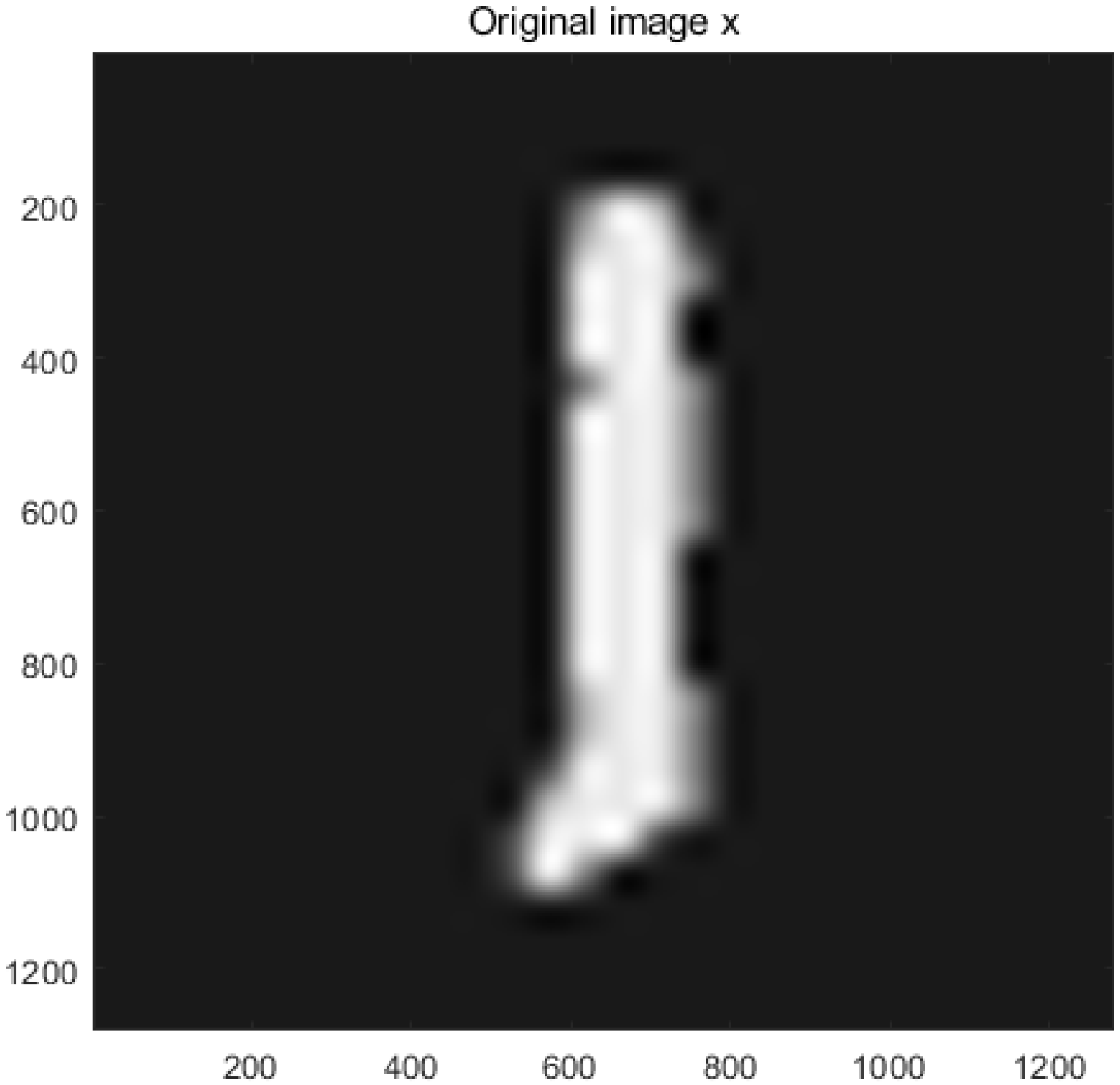} \\
			\includegraphics[scale=0.12]{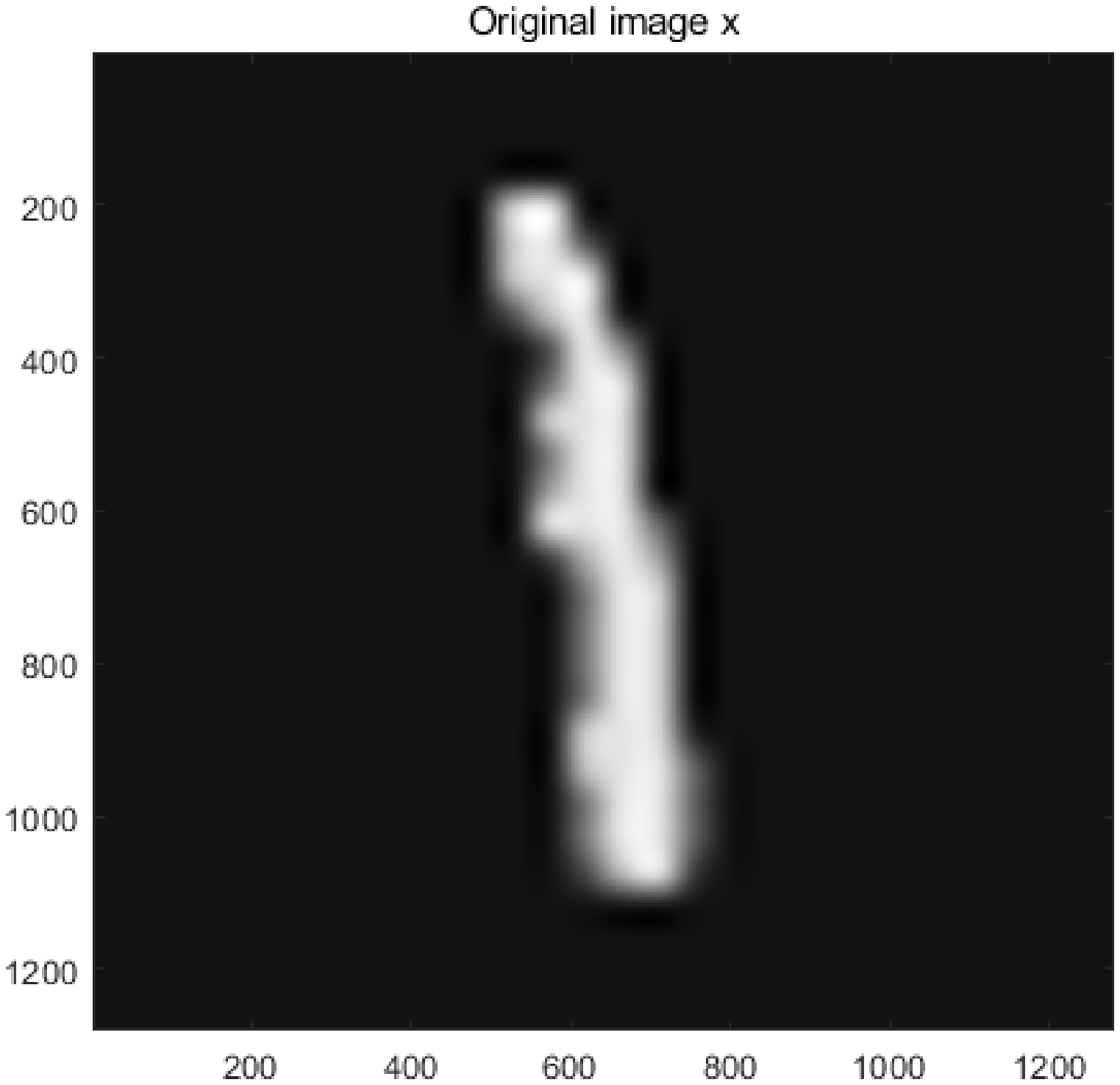}
		\end{minipage}
	\label{label_newfigure5_1}
	}
	\subfloat[The image of the unique uAP on  the  MNIST dataset.]
	{
		\begin{minipage}[b]{.3\linewidth}
			\centering
			\includegraphics[scale=0.15]{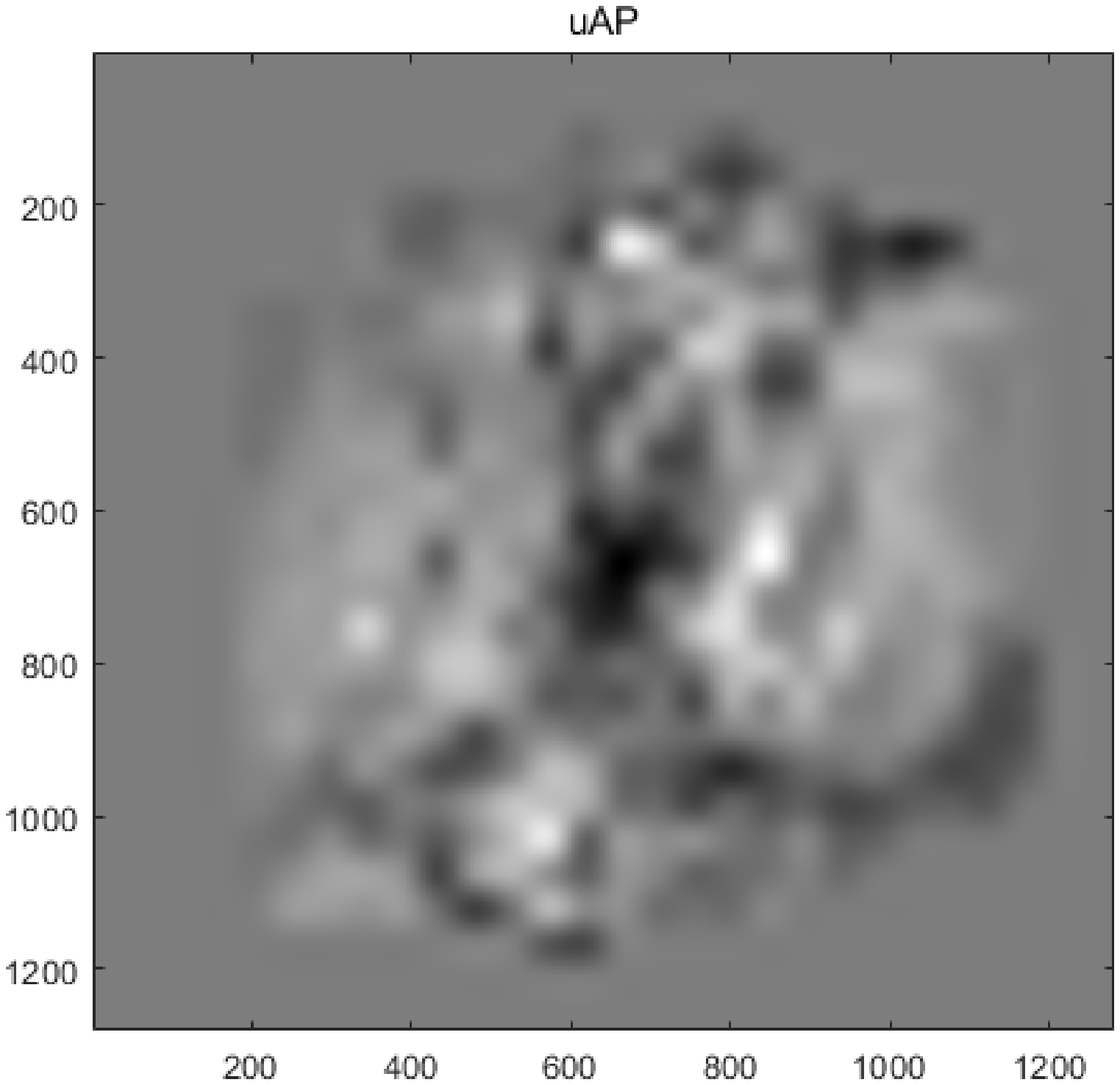}\vspace{20mm}
		\end{minipage}
	\label{label_newfigure5_2}
	}
	\subfloat[The  perturbed  images with class 0.]
	{
		\begin{minipage}[b]{.23\linewidth}
			\centering
			\includegraphics[scale=0.12]{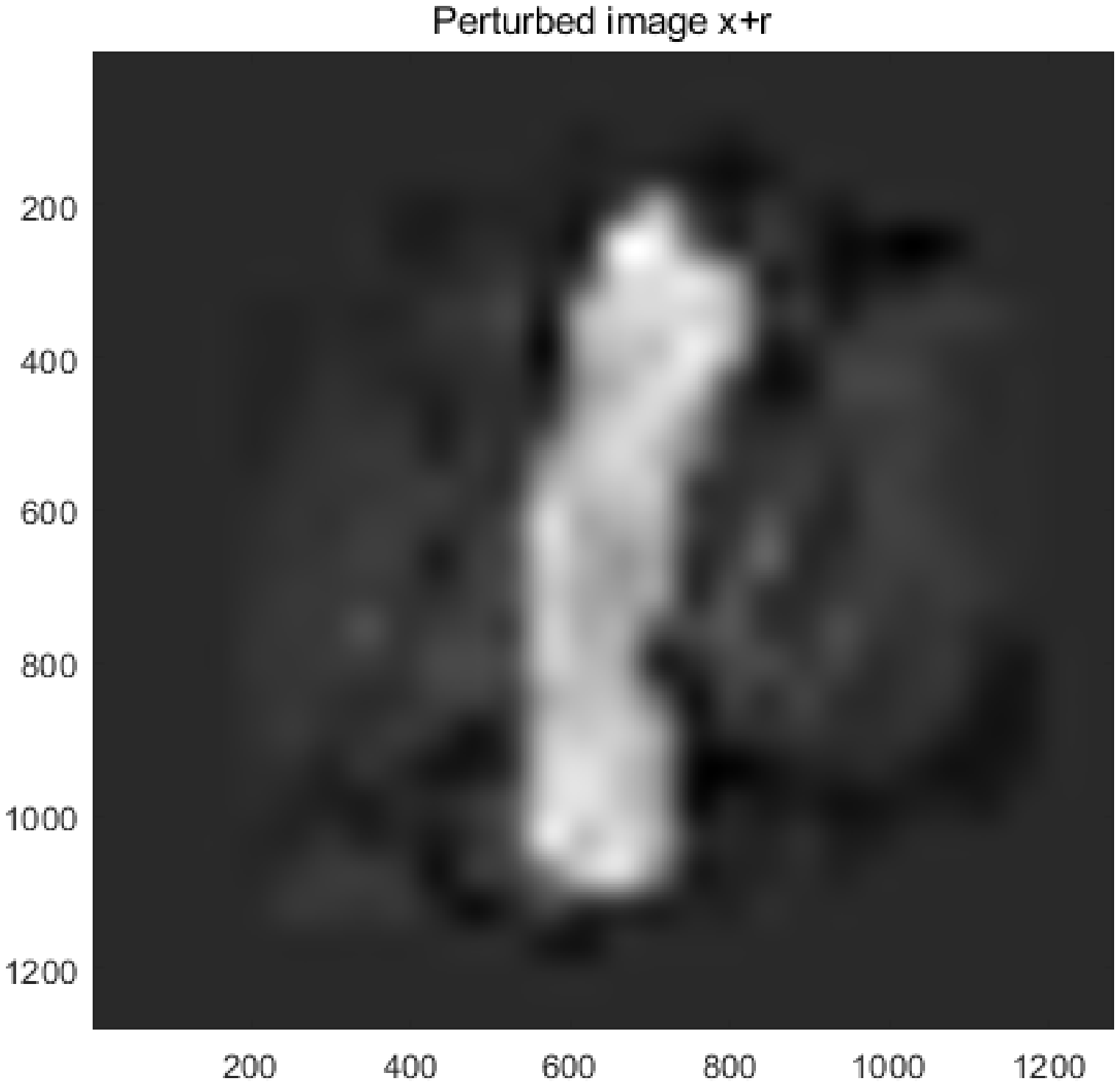} \\
			\includegraphics[scale=0.12]{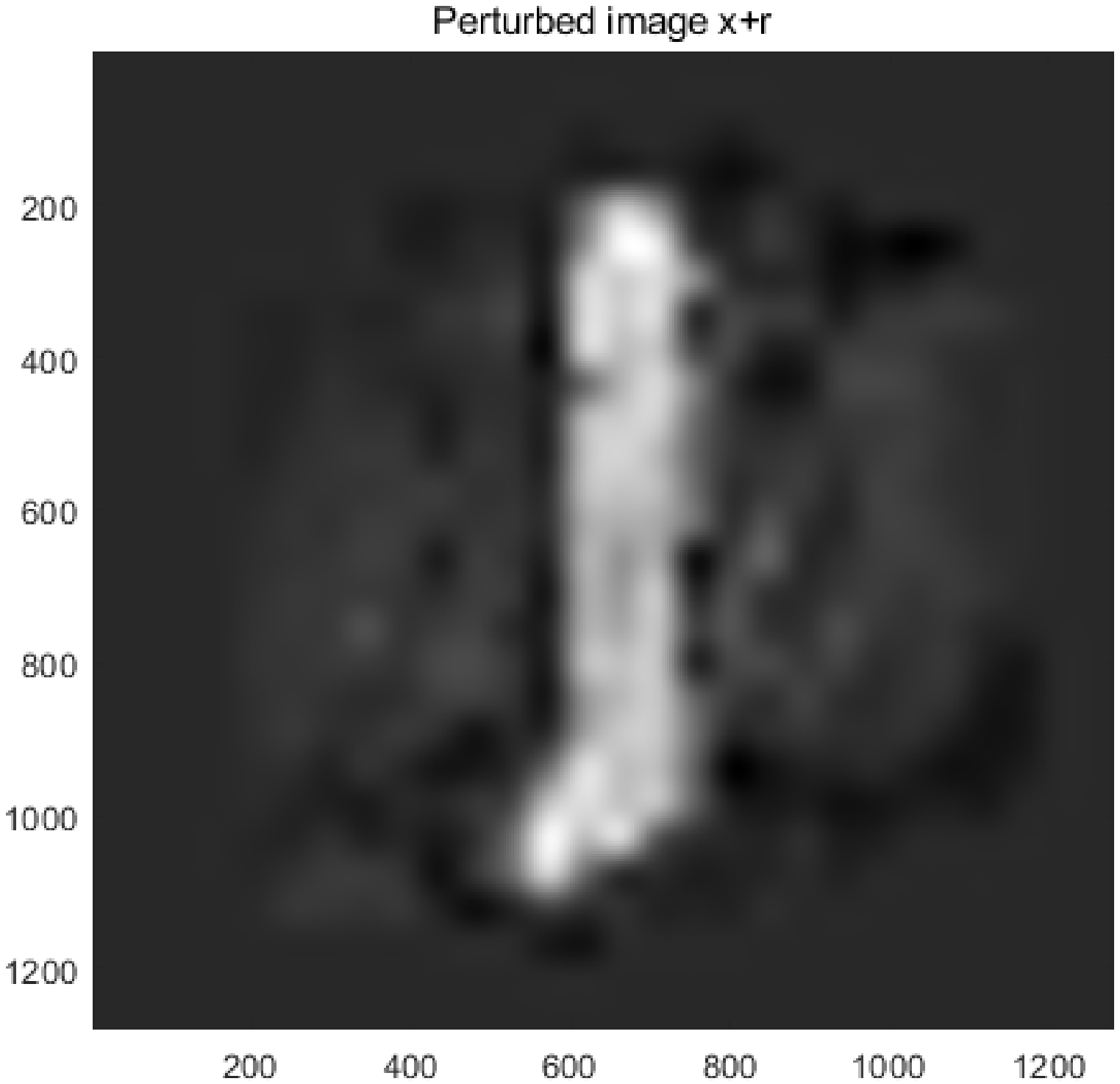} \\
			\includegraphics[scale=0.12]{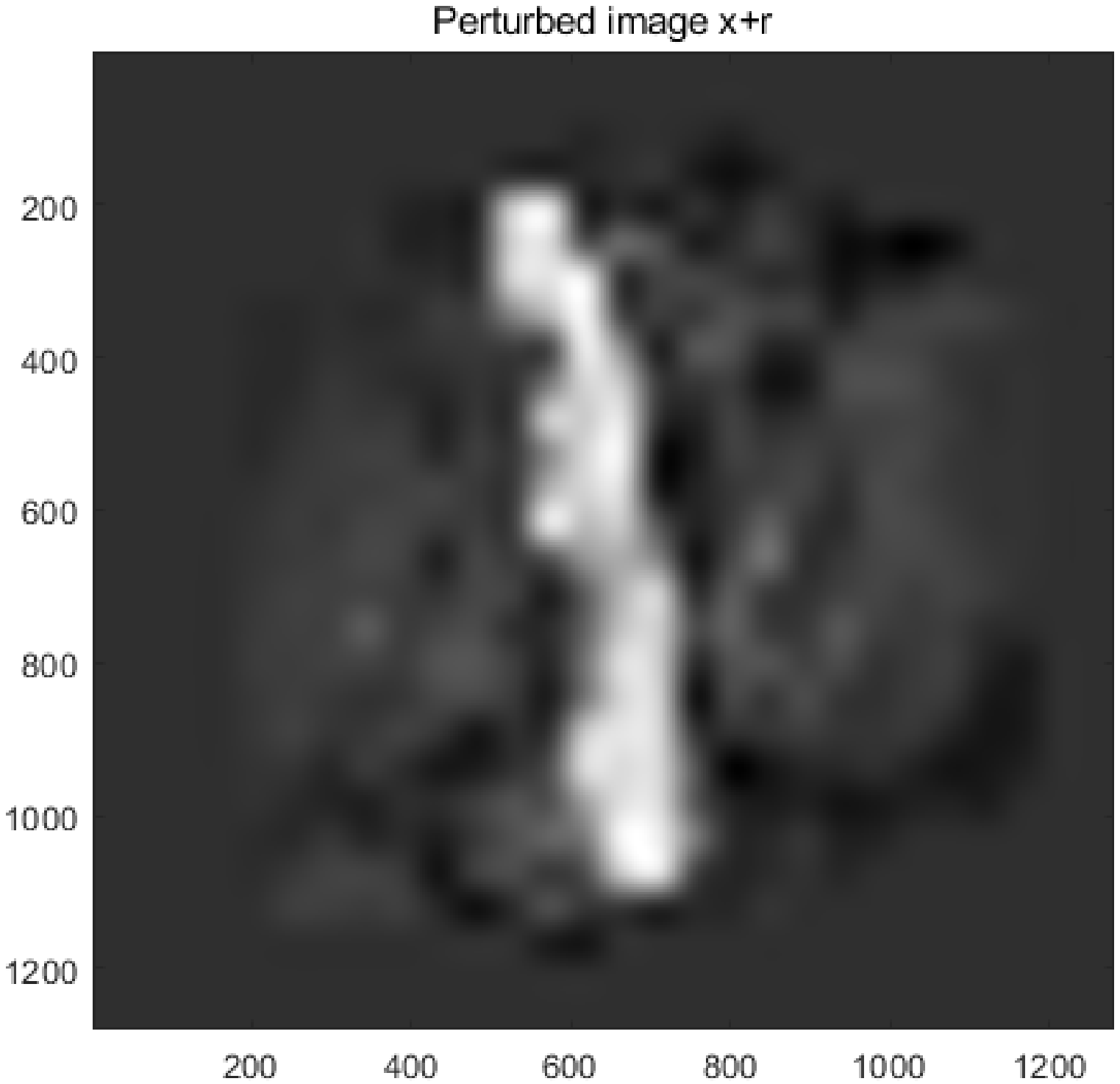}
		\end{minipage}
	\label{label_newfigure5_3}
	}
	\caption{The original images, the  image of uAP, and the images that have been misclassified after being attacked, when $ \xi=2 $ on the MNIST dataset.}
	\label{label_newfigure5}
\end{figure}

\begin{figure}[H]
	\centering
	\subfloat[The original images with class dog.]
	{
		\begin{minipage}[b]{.23\linewidth}
			\centering
			\includegraphics[scale=0.12]{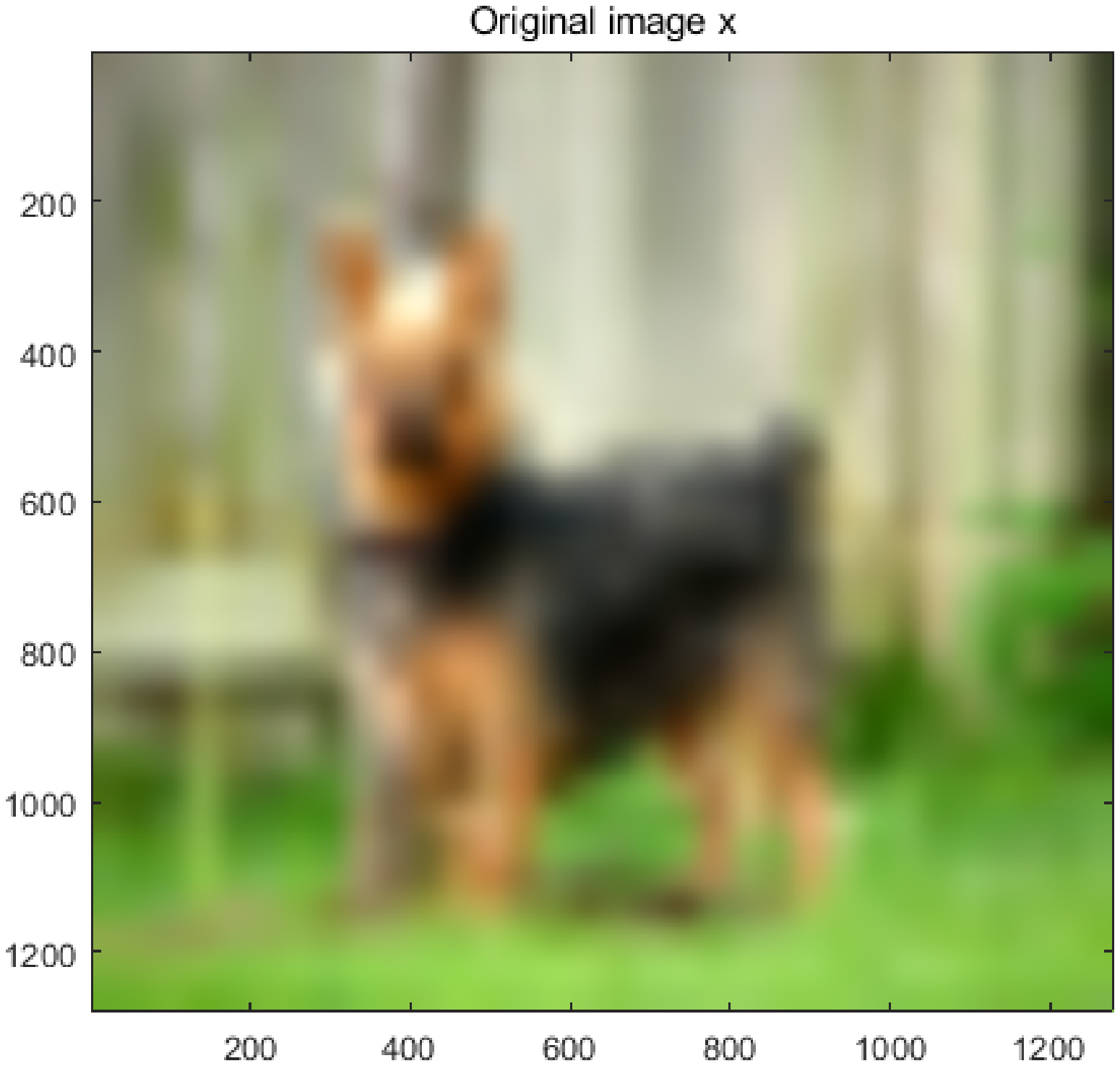} \\
			\includegraphics[scale=0.12]{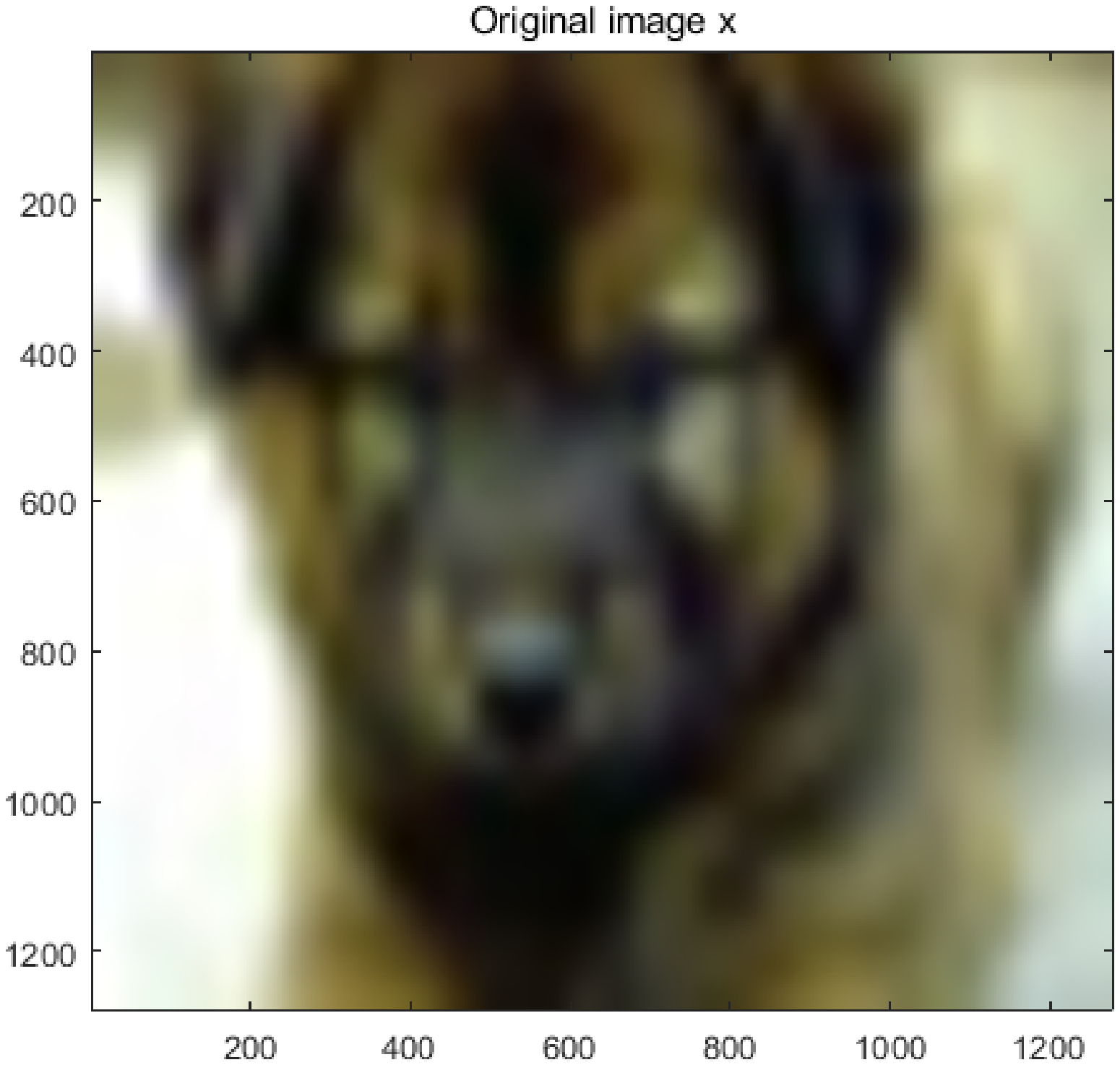} \\
			\includegraphics[scale=0.12]{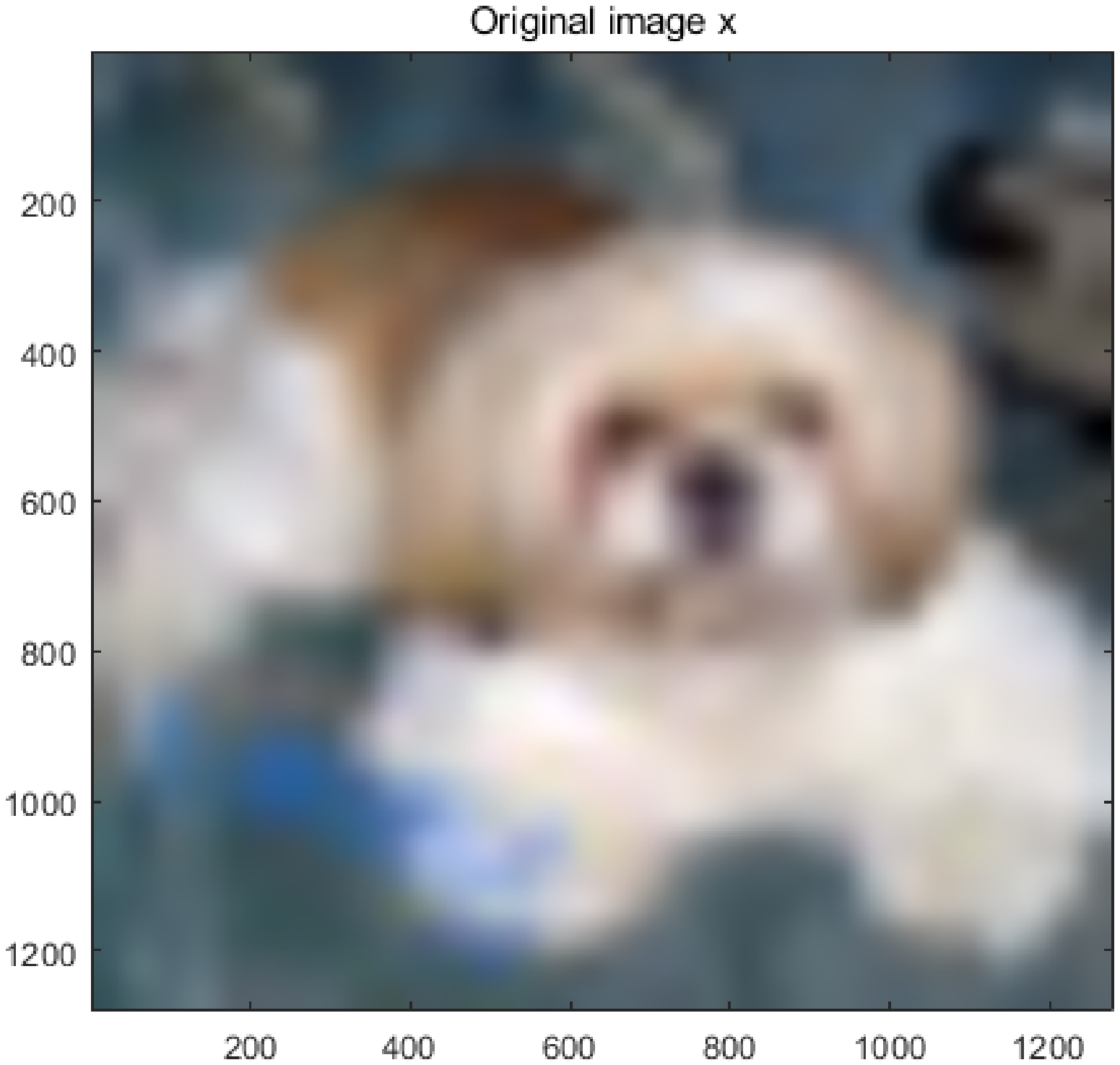}
		\end{minipage}
	}
	\subfloat[The image of the unique uAP on  the  CIFAR-10 dataset.]
	{
		\begin{minipage}[b]{.3\linewidth}
			\centering
			\includegraphics[scale=0.15]{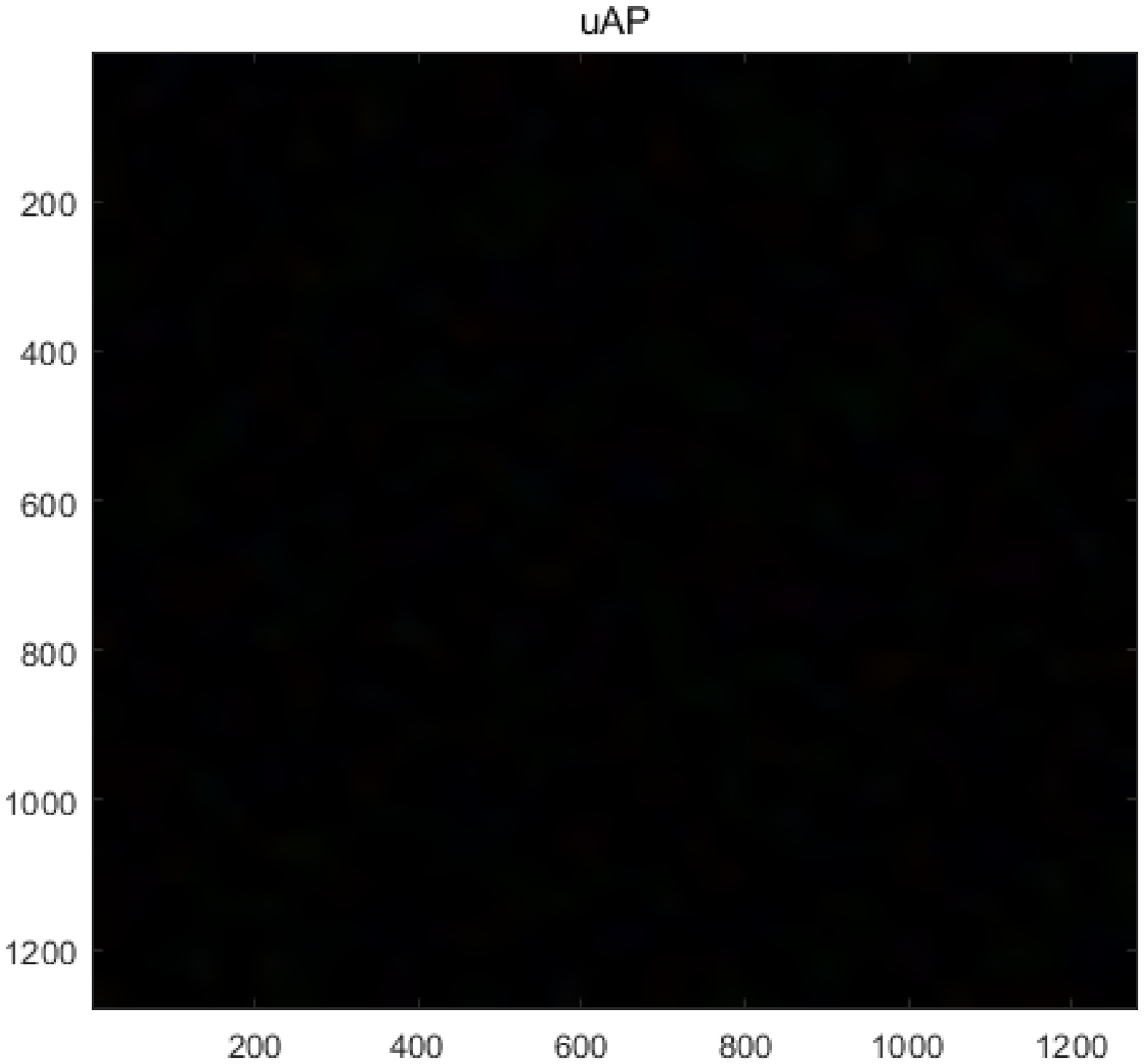}\vspace{20mm}
		\end{minipage}
	\label{label_newfigure6_2}
	}
	\subfloat[The  perturbed  images with class truck.]
	{
		\begin{minipage}[b]{.23\linewidth}
			\centering
			\includegraphics[scale=0.12]{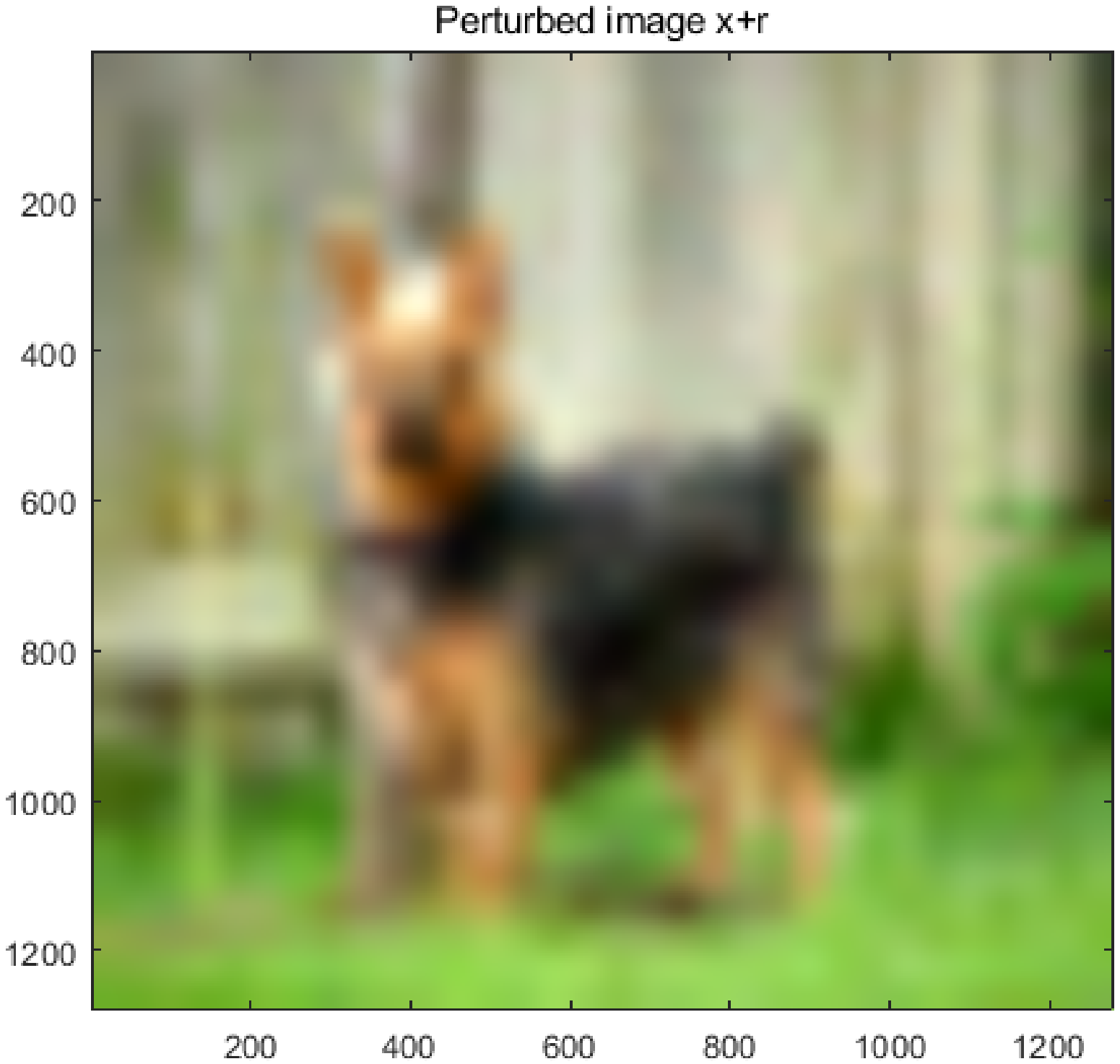} \\
			\includegraphics[scale=0.12]{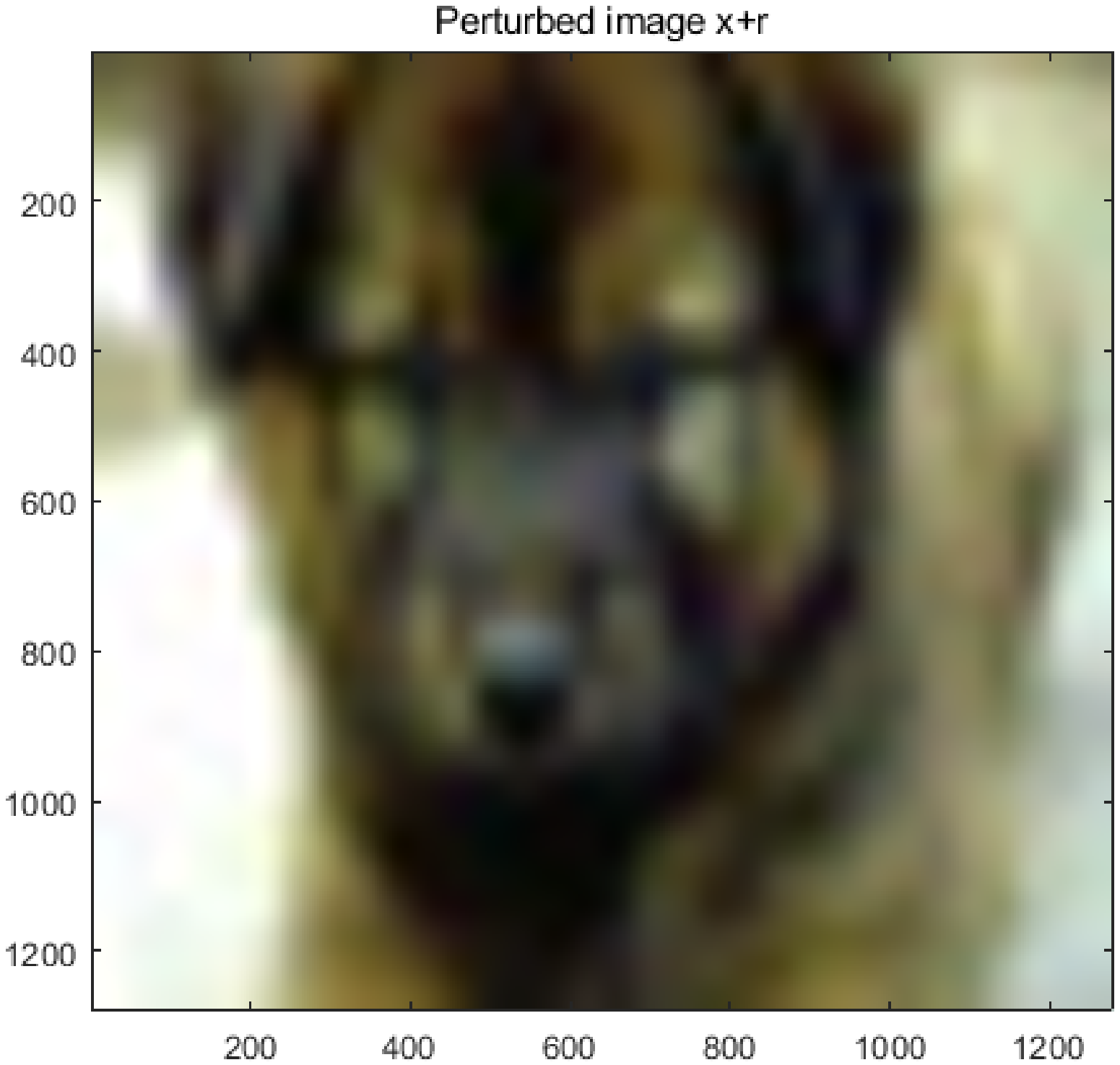} \\
			\includegraphics[scale=0.12]{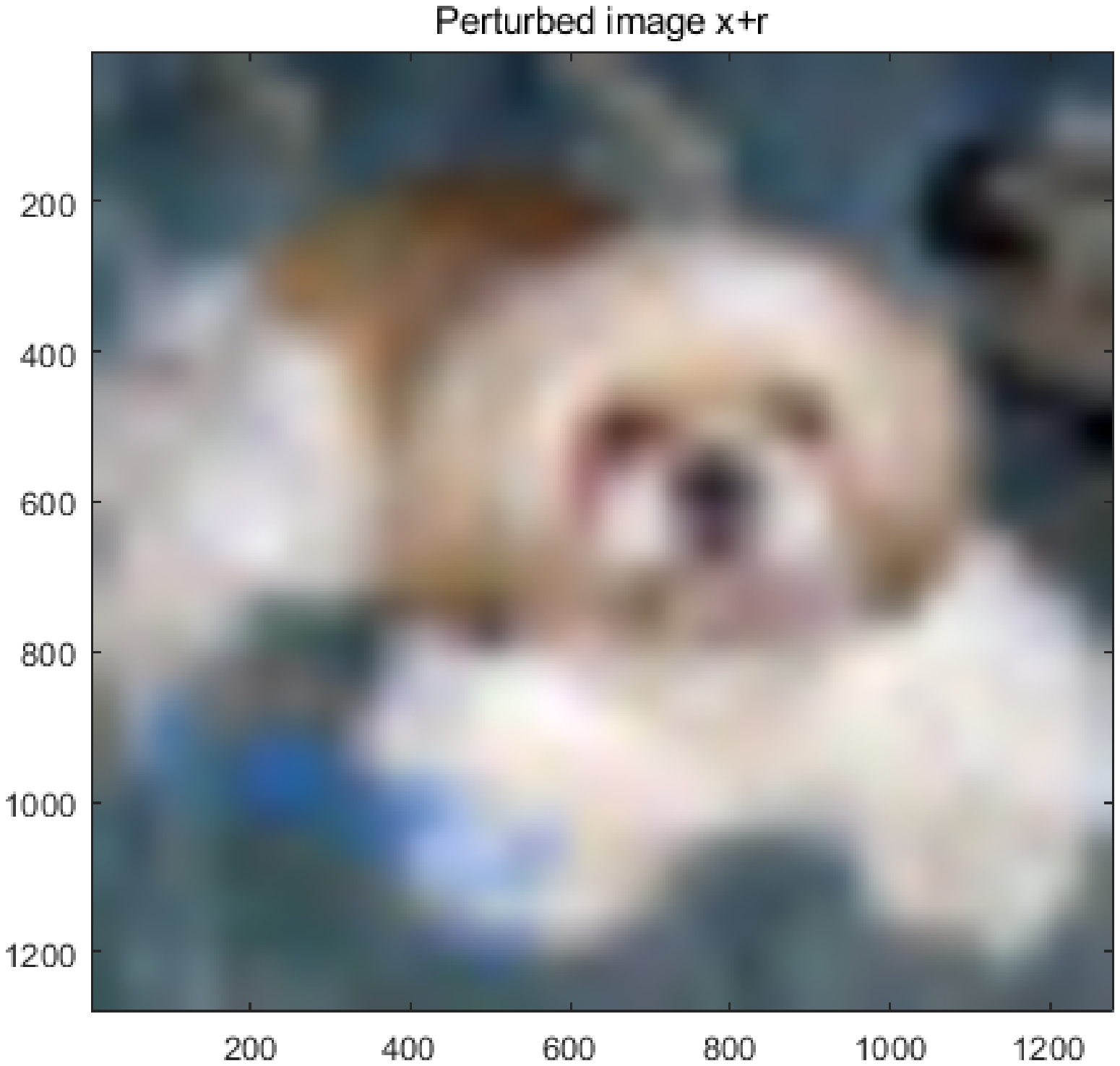}
		\end{minipage}
	}
	\caption{The original images, the  image of uAP, and the images that have been misclassified after being attacked, when $ \xi=0.5 $ on the CIFAR-10 dataset.}
	\label{label_newfigure6}
\end{figure}

However, in reality, we should also consider that our uAP should not be observed by human beings, that is, the norm  should be small enough. So we may not choose the uAP with the maximum fooling rate, but choose an appropriate size of uAP. In \cref{label_newfigure5}, we give an example to compare the original images, the image of uAP, and the  images of that have been misclassified after being attacked, when $ \xi=2 $ on the MNIST dataset. In this case,  SNR is 12.57. By selecting the average of 10 repeated experiments on the same MNIST test dataset, we get that the  CPU  time to generate uAP is only $2.05 \times 10^{-5}s $. The process of generating uAP does not need iteration.
In \cref{label_newfigure5_1}, the original images class are 1 and in \cref{label_newfigure5_3}, the perturbed  images class are 0. In MNIST dataset,  because the proportion of data predicted as class 1 is larger,  uAP  mainly  fools the data with class 1, as shown in \cref{label_newfigure5_2}.
The  image of uAP is small enough. 

Similarly, in \cref{label_newfigure6}, we give the original images with class dog, the uAP  and the perturbed image with class truck, when $ \xi=0.5 $ on the CIFAR-10  dataset. In this case,  SNR is 35.55. In CIFAR-10 dataset,  because the proportion of data predicted as class dog is larger,  uAP  mainly  fools the data with class dog, as shown in \cref{label_newfigure6_2}.
The CPU time to generate uAP is $5.65 \times 10^{-4}s $. 
Comparing \cref{label_newfigure5} and \cref{label_newfigure6},  uAP generated on the dataset (CIFAR-10) with a larger number of features is less likely to be observed.

\begin{figure}[H]
	\centering
	%	\subfigure{
	%		\begin{minipage}[b]{.3\linewidth}
	%			\centering
	%			\includegraphics[scale=0.25]{Figure/figure4_9.eps}
	%		\end{minipage}
	%	}
	\subfloat{
		\begin{minipage}[b]{.45\linewidth}
			\centering
			\includegraphics[width=1.05\textwidth,height=0.3\textwidth]{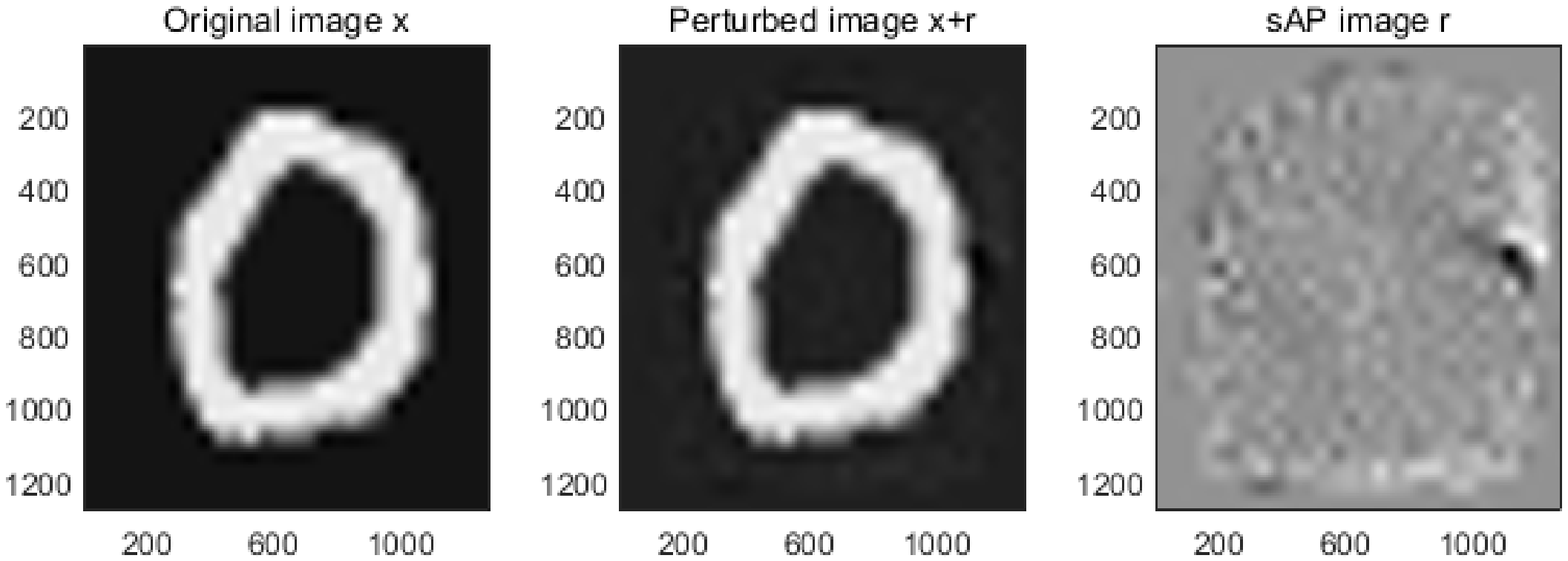}
		\end{minipage}
	}
\hfill 
	\subfloat{
		\begin{minipage}[b]{.45\linewidth}
			\centering
			\includegraphics[width=1.05\textwidth,height=0.3\textwidth]{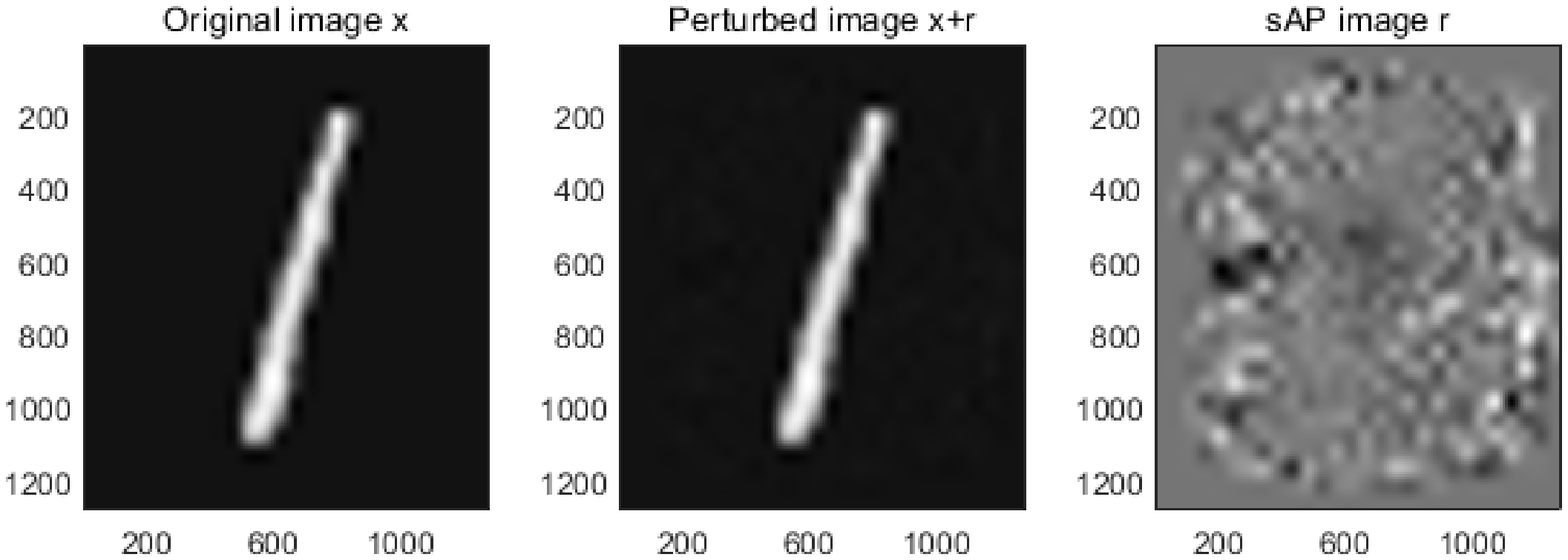}
		\end{minipage}
	}
\newline
	\subfloat{
		\begin{minipage}[b]{.45\linewidth}
			\centering
			\includegraphics[width=1.05\textwidth,height=0.3\textwidth]{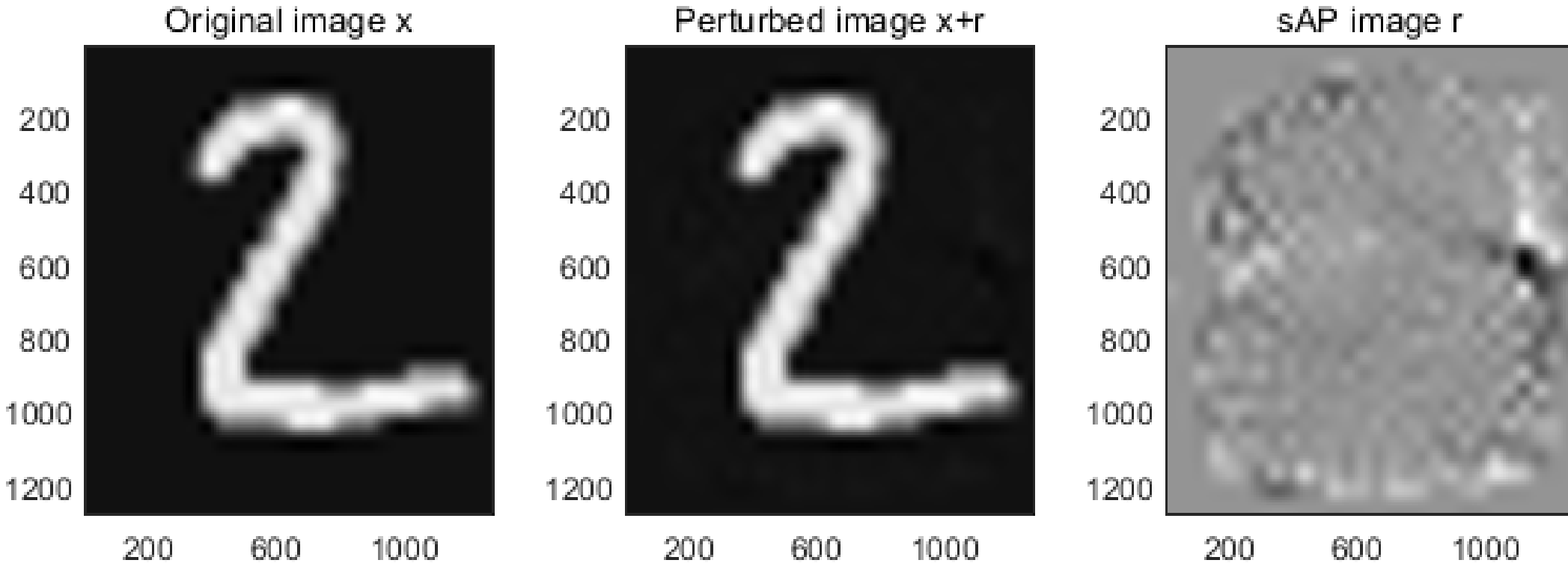}
		\end{minipage}
	}
\hfill 
	\subfloat{
		\begin{minipage}[b]{.45\linewidth}
			\centering
			\includegraphics[width=1.05\textwidth,height=0.3\textwidth]{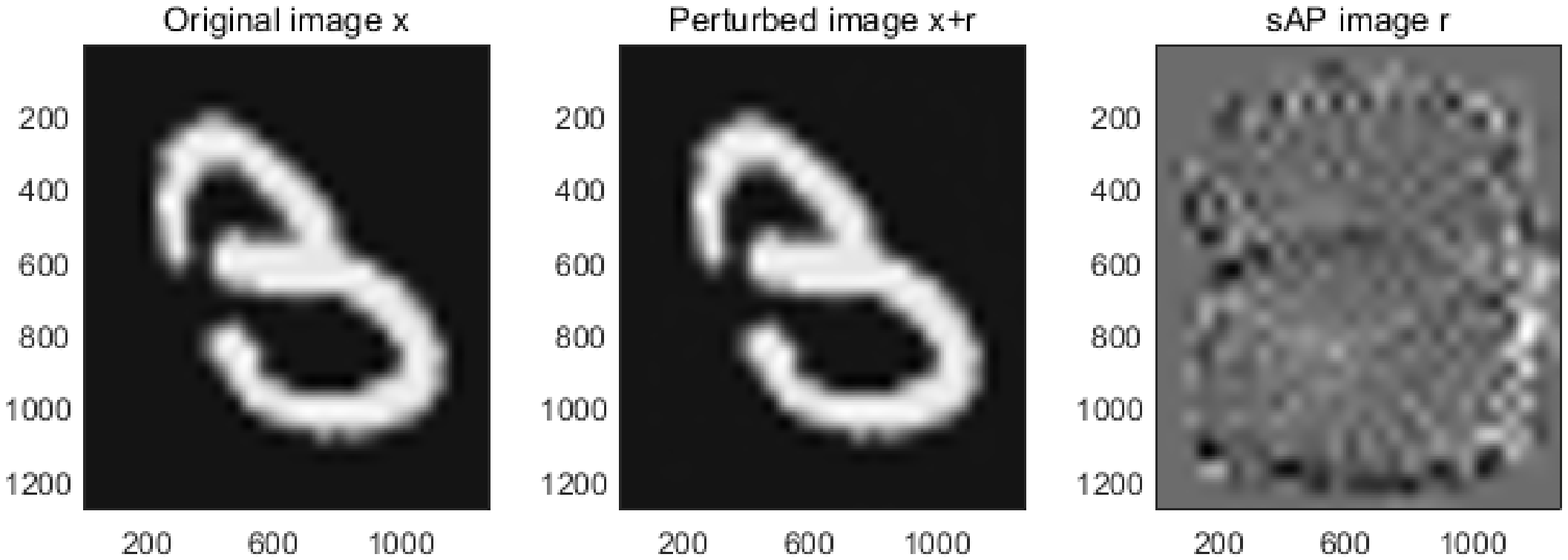}
		\end{minipage}
	}
\newline
	\subfloat{
		\begin{minipage}[b]{.45\linewidth}
			\centering
			\includegraphics[width=1.05\textwidth,height=0.3\textwidth]{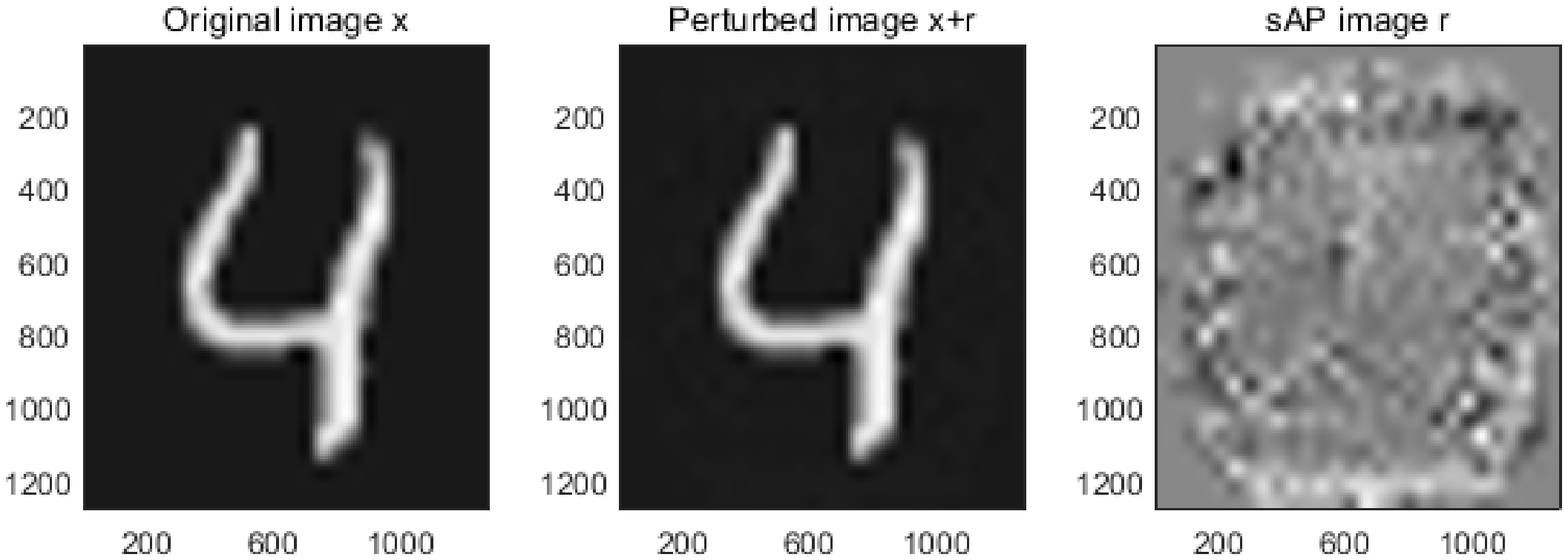}
		\end{minipage}
	}
\hfill 
	\subfloat{
		\begin{minipage}[b]{.45\linewidth}
			\centering
			\includegraphics[width=1.05\textwidth,height=0.3\textwidth]{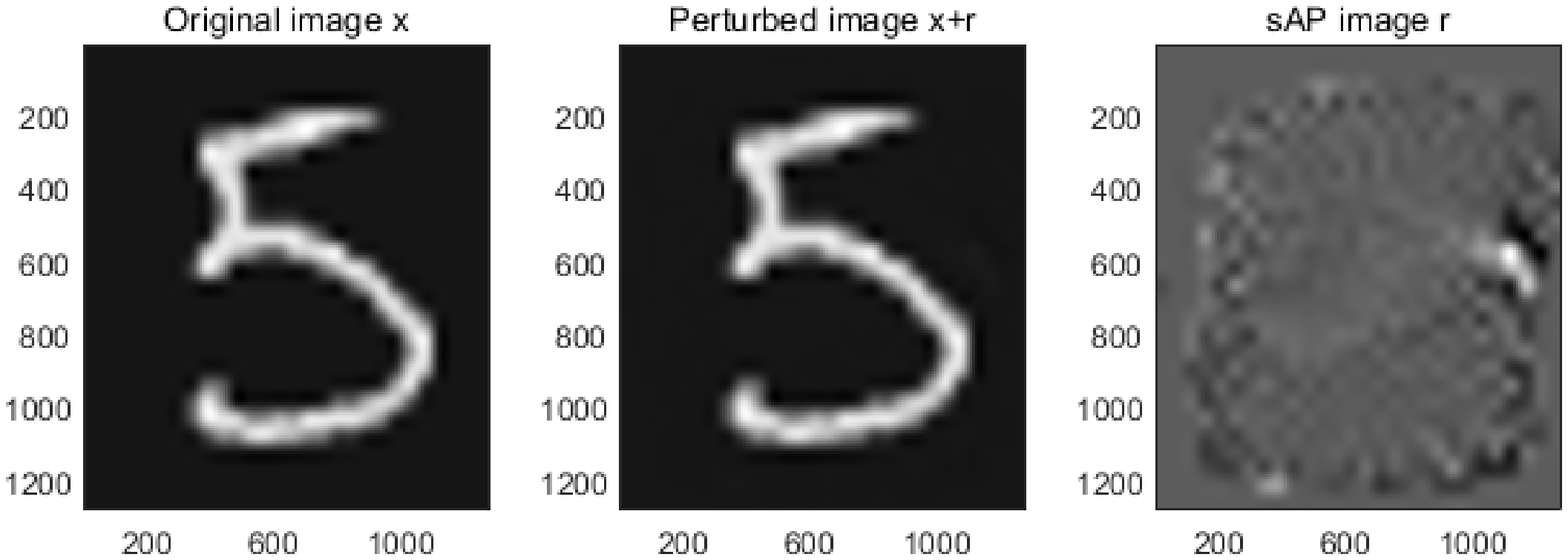}
		\end{minipage}
	}
\newline
	\subfloat{
		\begin{minipage}[b]{.45\linewidth}
			\centering
			\includegraphics[width=1.05\textwidth,height=0.3\textwidth]{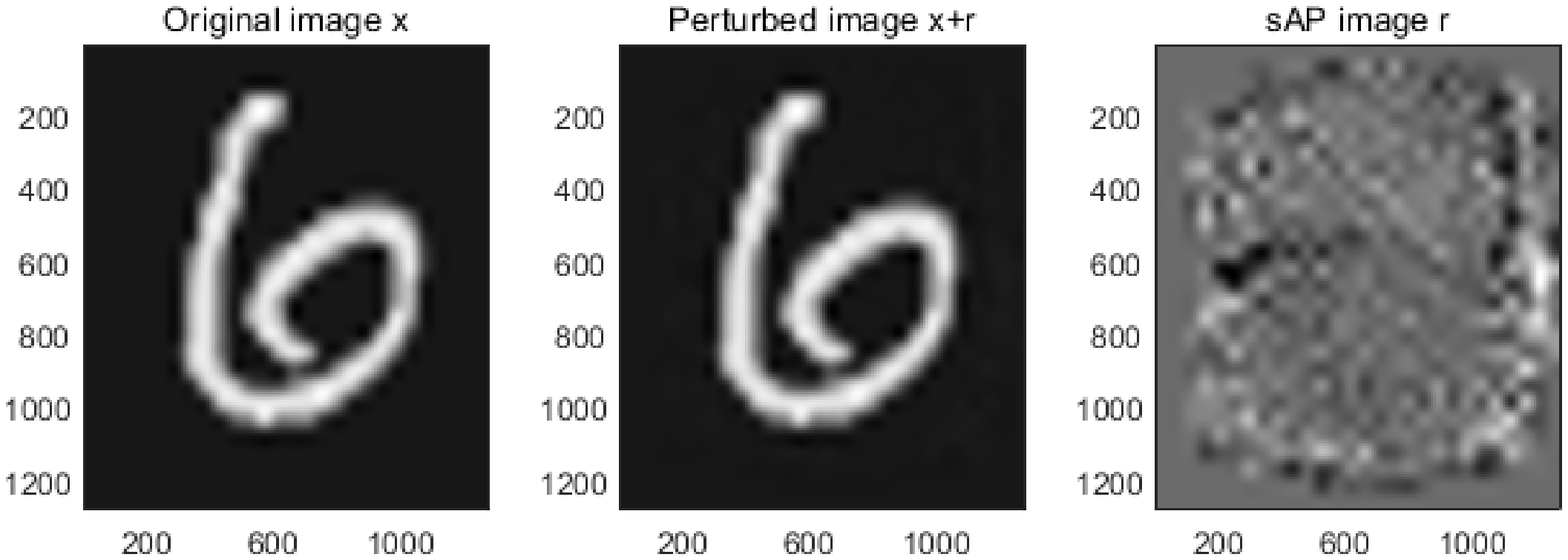}
		\end{minipage}
	}
\hfill 
	\subfloat{
		\begin{minipage}[b]{.45\linewidth}
			\centering
			\includegraphics[width=1.05\textwidth,height=0.3\textwidth]{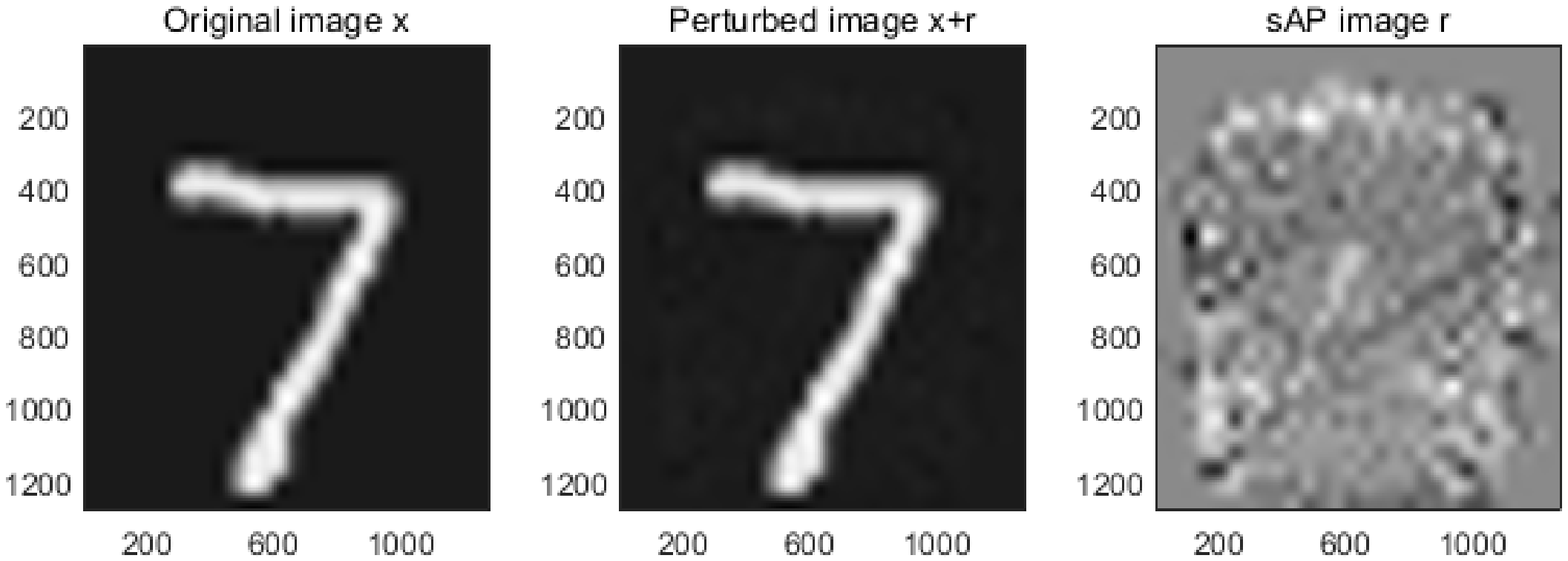}
		\end{minipage}
	}
\newline
	\subfloat{
		\begin{minipage}[b]{.45\linewidth}
			\centering
			\includegraphics[width=1.05\textwidth,height=0.3\textwidth]{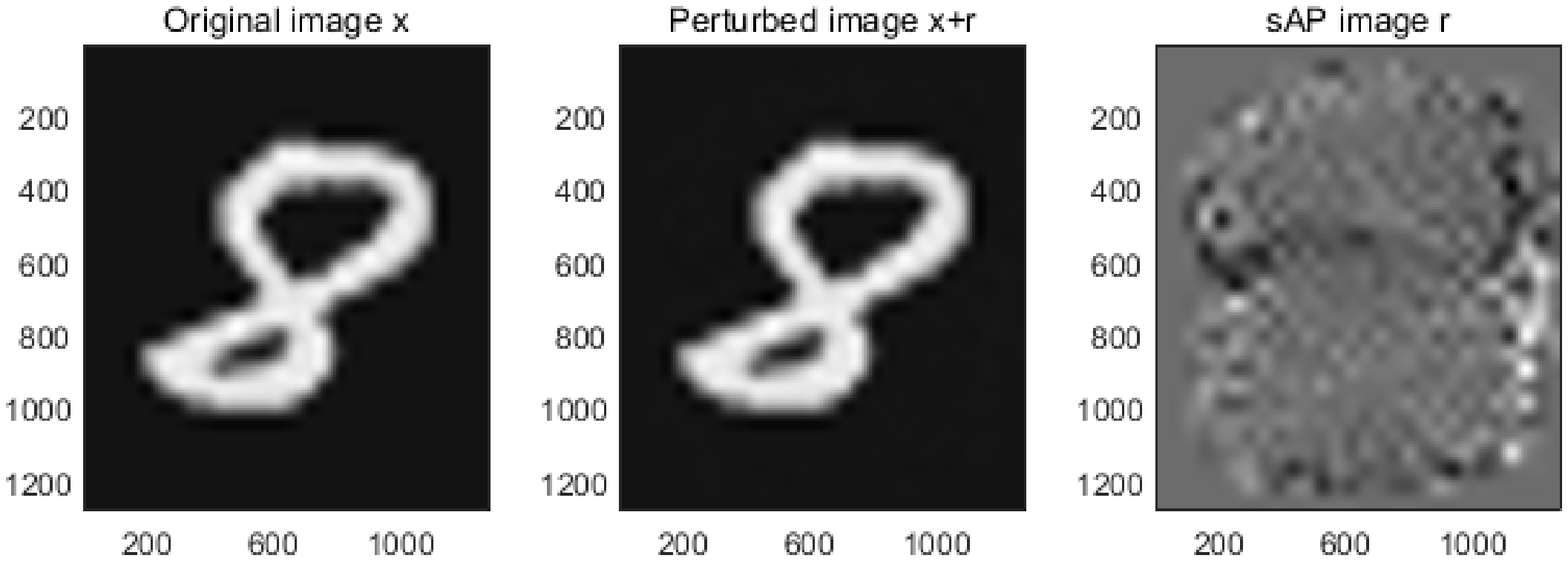}
		\end{minipage}
	}
\hfill 
	\subfloat{
		\begin{minipage}[b]{.45\linewidth}
			\centering
			\includegraphics[width=1.05\textwidth,height=0.3\textwidth]{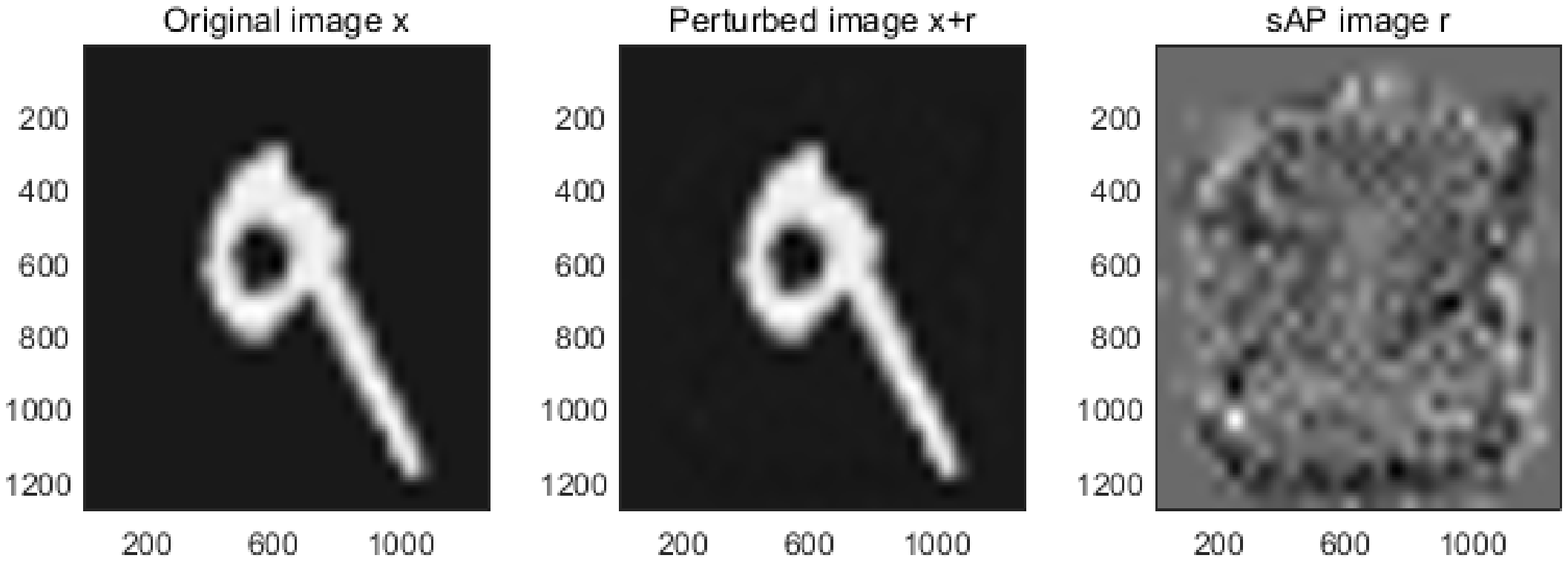}
		\end{minipage}
	}
	\caption{The original images with class $  0, 1, 2, 3, 4, 5, 6, 7, 8, 9 $, the images that has been misclassified after being attacked with class $5,2,5,2,9,8,2,3,2,1$, and the  images of sAP. }
	\label{label_newfigure7}
\end{figure}
\subsection{Numerical experiments of multiclass linear SVMs}  We present our experiments on the MNIST dataset and CIFAR-10 dataset on multiclass linear SVMs.   For CIFAR-10, due to the limitation of the computational complexity of training model, we only select 5,000 training samples and 1,000 test samples from the original dataset.
First, we use LIBLINEAR to build a multiclass linear  SVM on the training set, and obtain the parameters $ w_l$,  $l \in [c]$ of the classifier.
Then we use formulas \cref{equ5}, \cref{equ7} and \cref{the5} to generate sAP, cuAP and uAP  respectively.
Finally, we calculate $ G_{\Omega, \hat{k}} $ of the uAP.
\subsubsection{Numerical experiments of sAP}
In \cref{label_newfigure7}, we give an example to compare the original image, the image that has been misclassified after being attacked, and the  image of sAP on the MNIST dataset. By selecting the average of 10 repeated experiments on the same MNIST dataset, we get that   the  CPU  time to train the multiclass classifier model is $ 35.78s $, and  the average CPU  time to generate sAP is only $3.12 \times 10^{-3}s $.
The process of generating sAP does not need iteration.
In \cref{label_newfigure7}, the original images class are $  0, 1, 2, 3, 4, 5, 6, 7, 8, 9 $. When we add sAP to the original images, the perturbed image class are $5,2,5,2,9,8,2,3,2,1$, but in human eyes, the class of the  perturbed images have not changed.  The average norm of the data in the MNIST  dataset is 9.30, the  average norm of sAP is 0.20, and SNR is 35.19.
\begin{figure}[h]
	\centering
	%	\subfigure{
	%		\begin{minipage}[b]{.3\linewidth}
	%			\centering
	%			\includegraphics[scale=0.25]{Figure/figure4_9.eps}
	%		\end{minipage}
	%	}
	\subfloat{
		\begin{minipage}[b]{.45\linewidth}
			\centering
			\includegraphics[width=1.05\textwidth,height=0.3\textwidth]{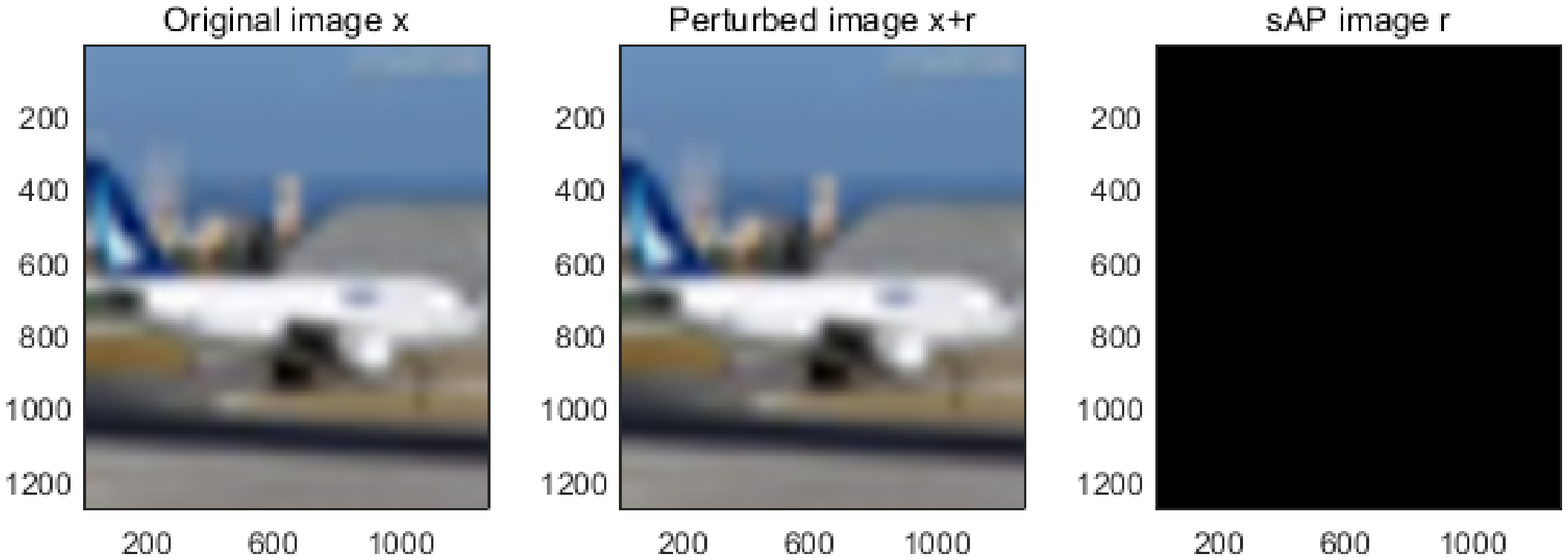}
		\end{minipage}
	}
\hfill 
	\subfloat{
		\begin{minipage}[b]{.45\linewidth}
			\centering
			\includegraphics[width=1.05\textwidth,height=0.3\textwidth]{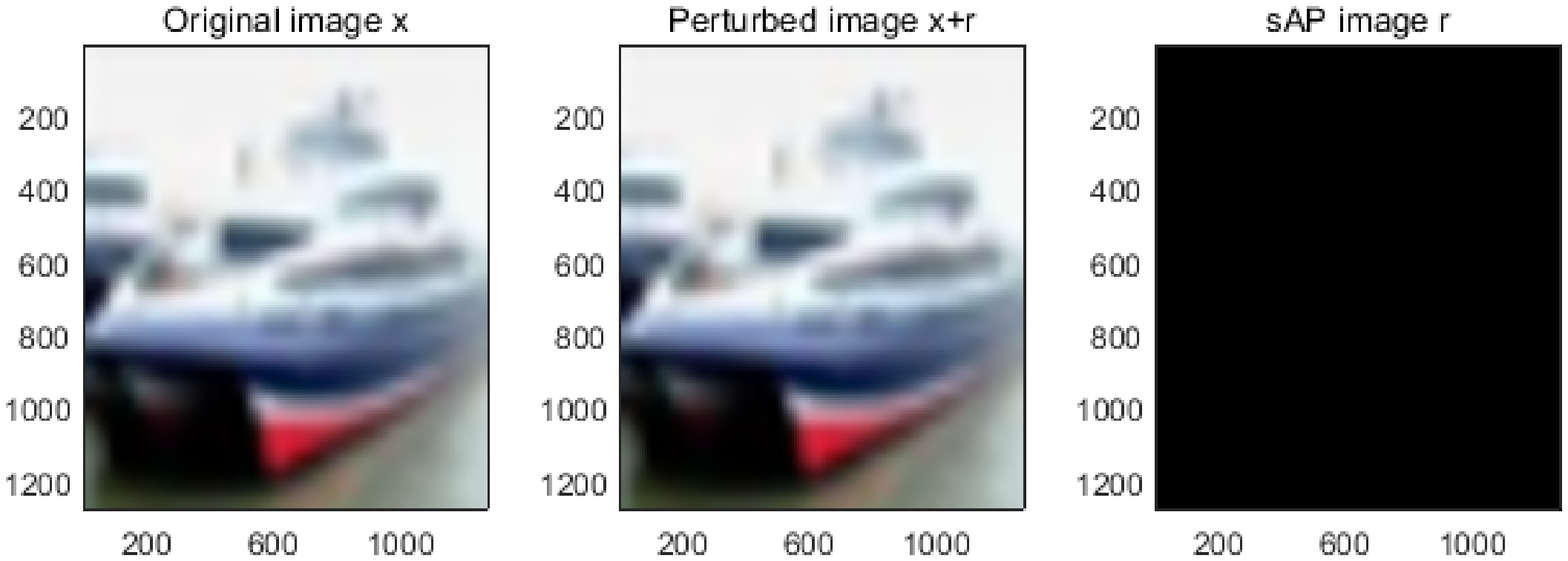}
		\end{minipage}
	}
\newline
	\subfloat{
		\begin{minipage}[b]{.45\linewidth}
			\centering
			\includegraphics[width=1.05\textwidth,height=0.3\textwidth]{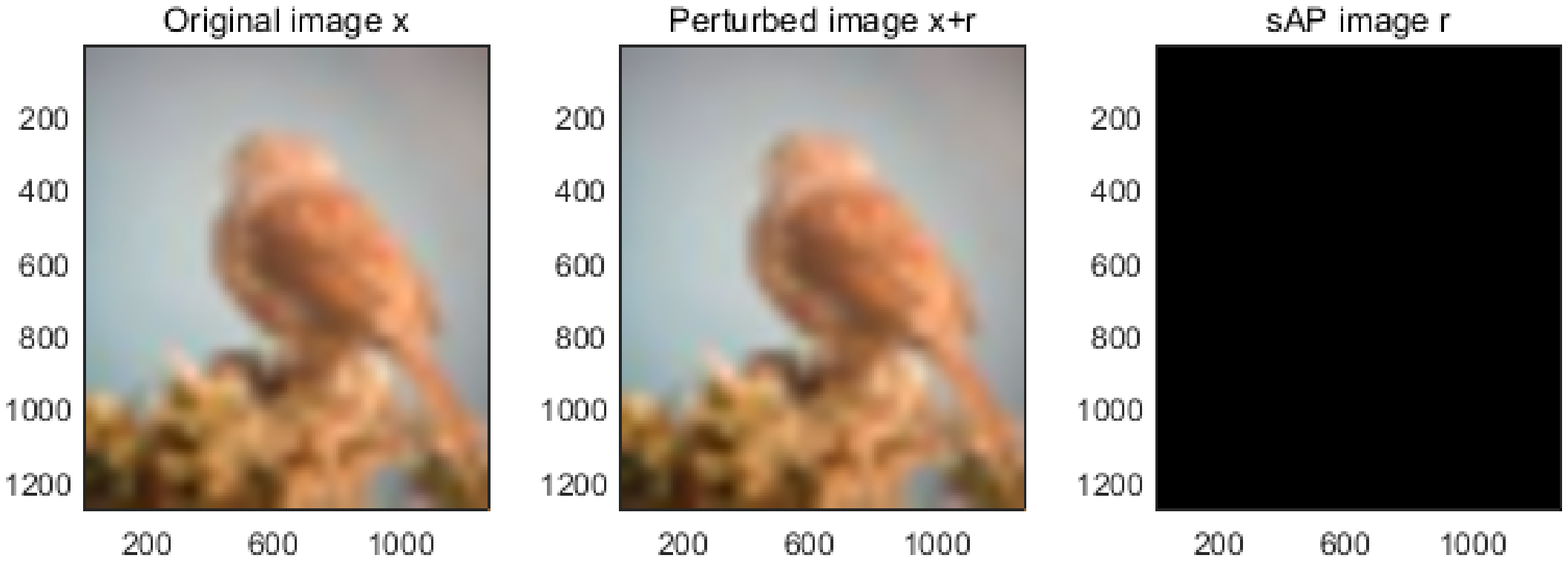}
		\end{minipage}
	}
\hfill 
	\subfloat{
		\begin{minipage}[b]{.45\linewidth}
			\centering
			\includegraphics[width=1.05\textwidth,height=0.3\textwidth]{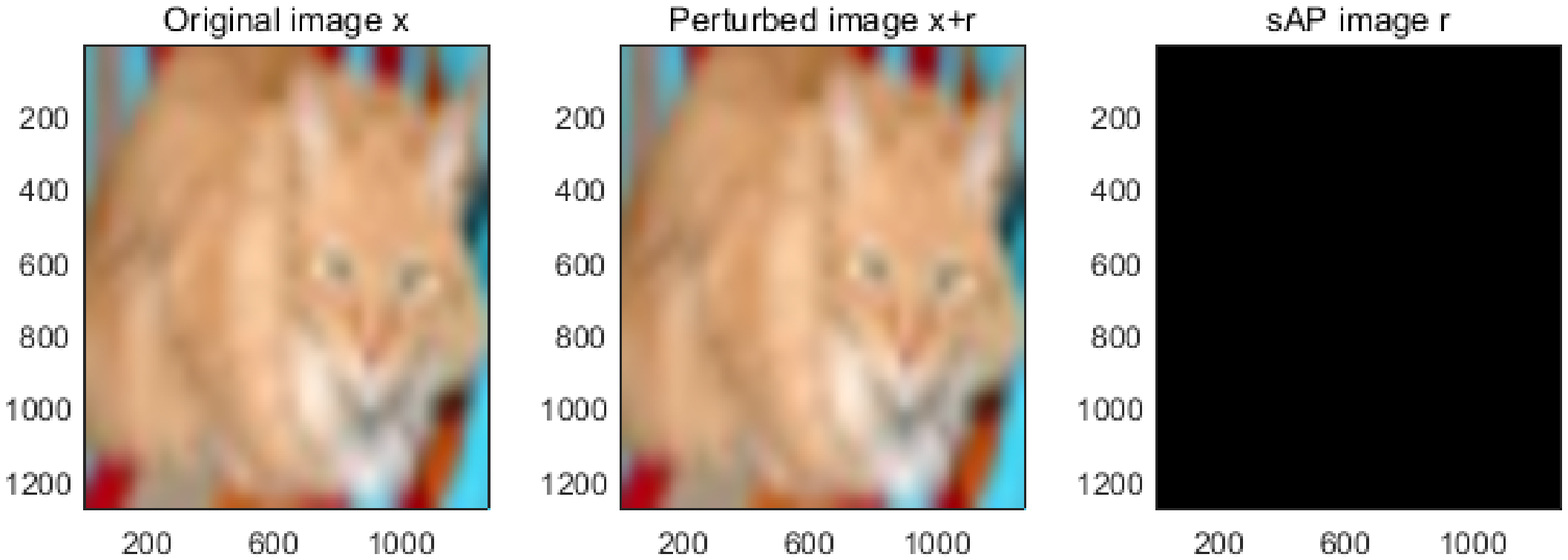}
		\end{minipage}
	}
\newline
	\subfloat{
		\begin{minipage}[b]{.45\linewidth}
			\centering
			\includegraphics[width=1.05\textwidth,height=0.3\textwidth]{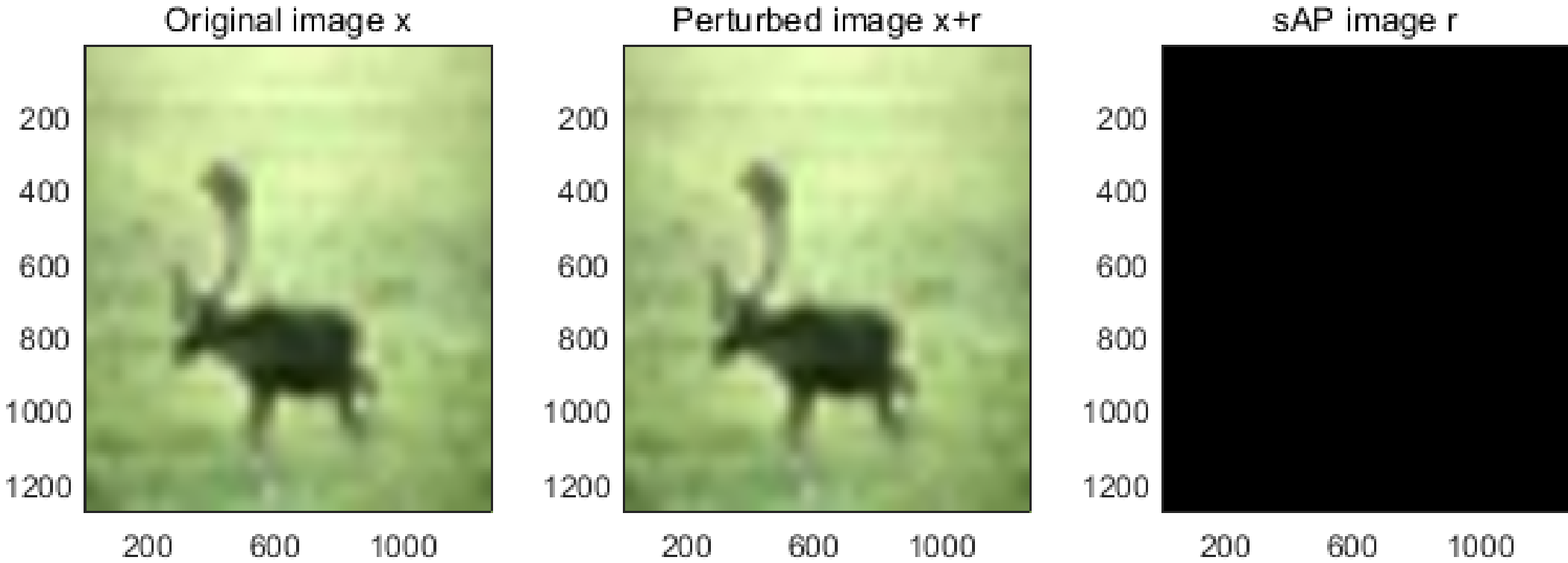}
		\end{minipage}
	}
\hfill 
	\subfloat{
		\begin{minipage}[b]{.45\linewidth}
			\centering
			\includegraphics[width=1.05\textwidth,height=0.3\textwidth]{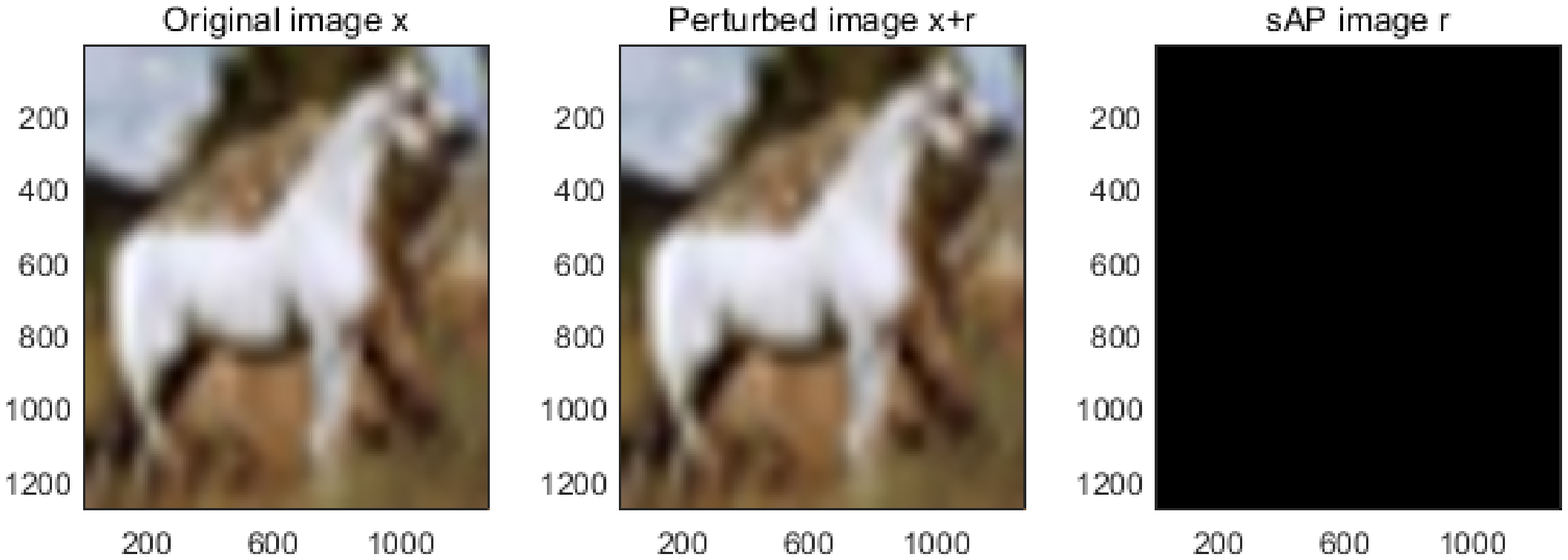}
		\end{minipage}
	}
	\caption{The original images with class airplane,  ship, bird, cat, deer and horse, the images that has been misclassified after being attacked with class ship, automobile, cat, horse, bird and deer, and the  images of sAP. }
	\label{label_newfigure8}
\end{figure}

Similarly, on the CIFAR-10 dataset, we get  that the CPU time to train the multiclass  classifier model is $ 2.01 \times 10^{3}s $, and    the average CPU  time to generate sAP is only $1.44 \times 10^{-3}s $.
In \cref{label_newfigure8}, the original images class are airplane,  ship, bird, cat, deer and horse. When we add the sAP to the original images, the perturbed image class are ship, automobile, cat, horse, bird and deer. It can be seen that animal images become another class of animal images after being attacked, and so are vehicle images.
This is because our generated sAP is based on the idea of attacking images to the most similar class, and our method generates the smallest sAP. We obtain that the average norm of the data in the CIFAR-10  dataset is 29.97, the average norm of sAP is 0.07, and SNR is 55.79.
\subsubsection{Numerical experiments of cuAP}

In \cref{label_newfigure9}, 
we take numerical experiments on the  MNIST dataset with class 6, and we get the CPU time to generate cuAP is only  $5.11s $.
\cref{label_newfigure9_1} show the original clean images with class 6. 
If they are attacked by the same cuAP shown in \cref{label_newfigure9_2}, the perturbed new images will be displayed in the corresponding position in  \cref{label_newfigure9_3}, and the classifier will misclassify them as the number 8. The norm of cuAP is 1, and SNR is 19.22.
\begin{figure}[H]
	\centering
	\subfloat[The original images with class  6.]
	{
		\begin{minipage}[b]{.23\linewidth}
			\centering
			\includegraphics[scale=0.11]{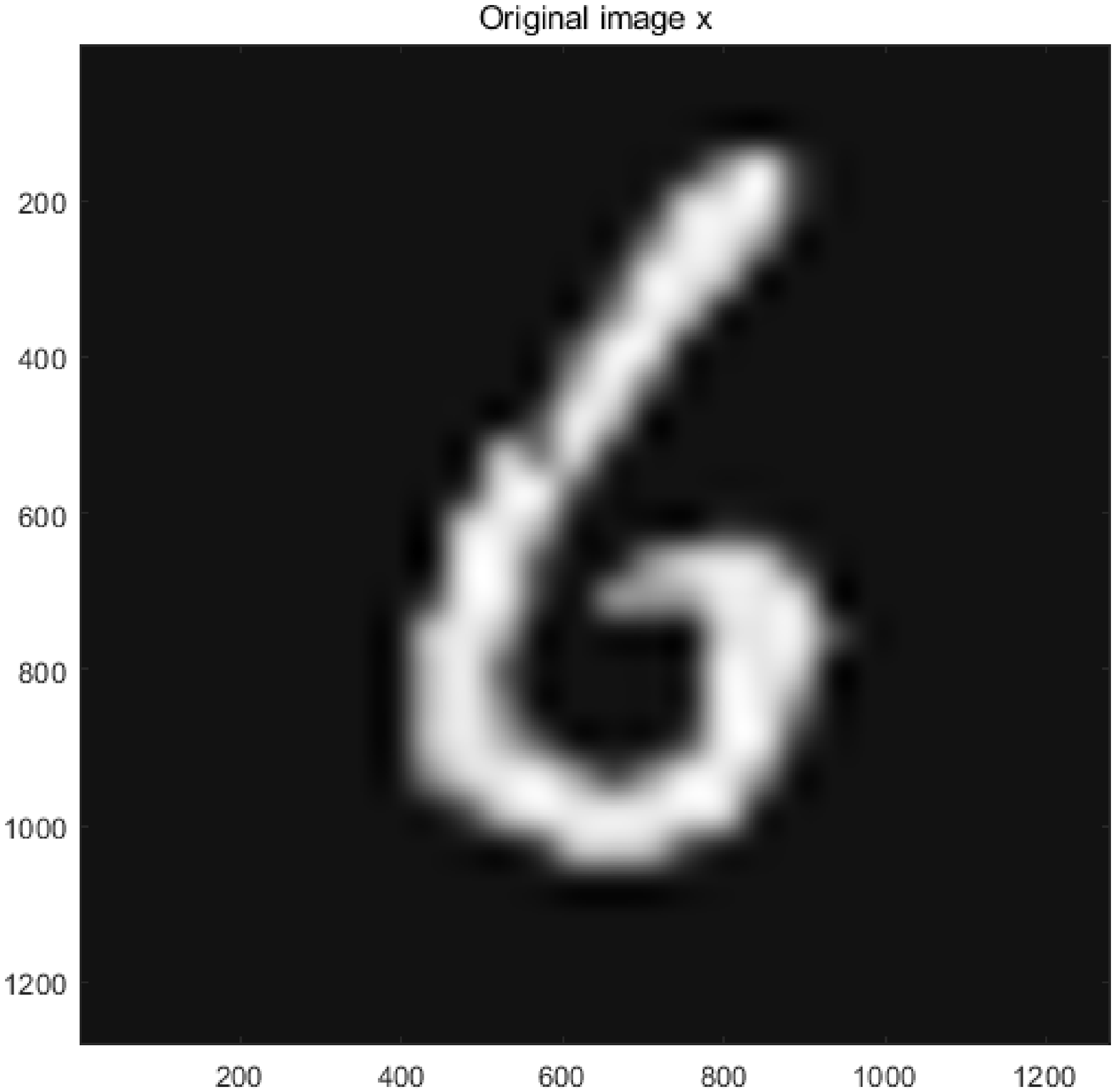} \\
			\includegraphics[scale=0.11]{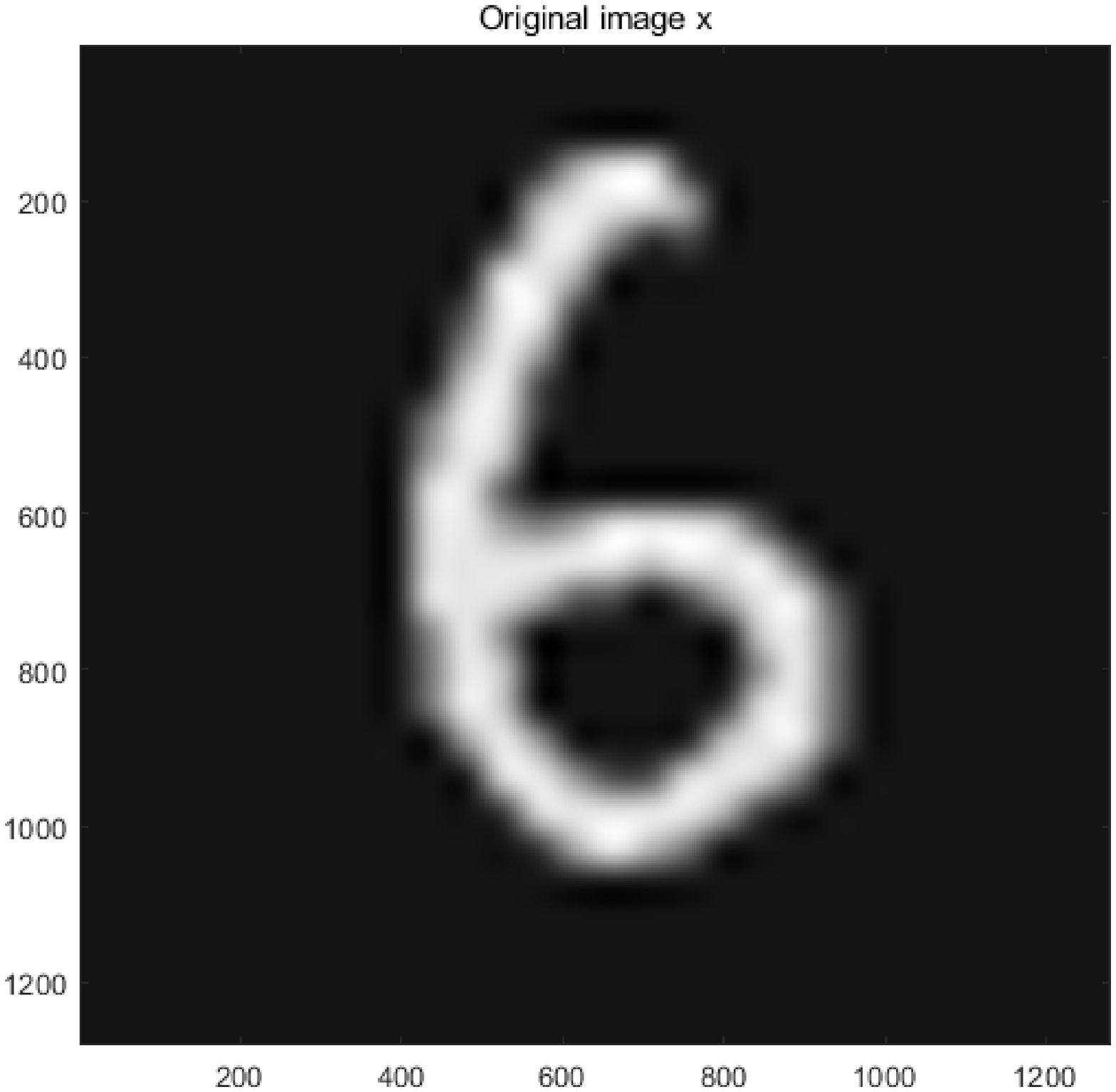} \\
			\includegraphics[scale=0.11]{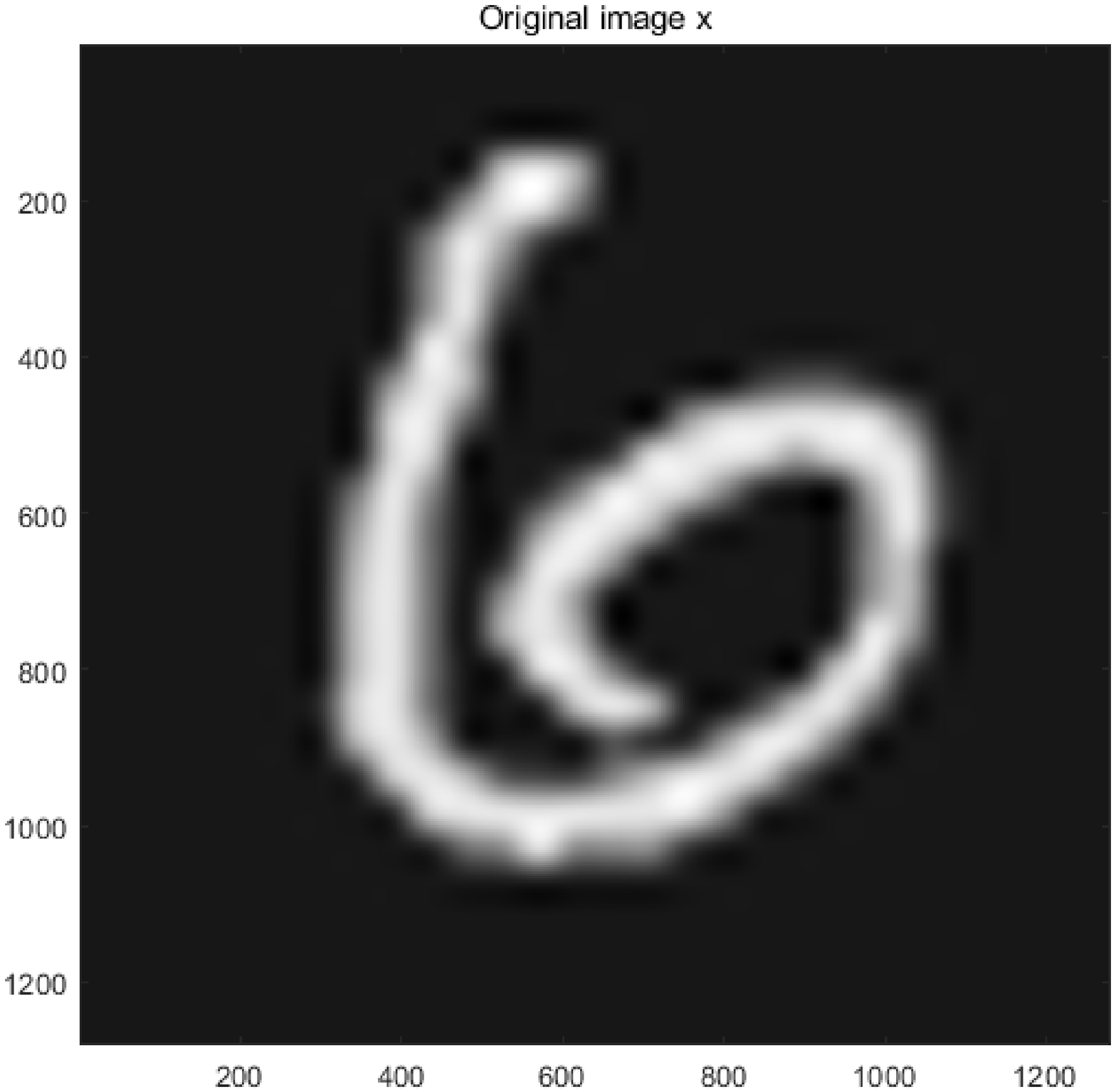}
		\end{minipage}
	\label{label_newfigure9_1}
	}
	\subfloat[The image of the unique cuAP on  the subset of MNIST dataset with class 6.]
	{
		\begin{minipage}[b]{.3\linewidth}
			\centering
			\includegraphics[scale=0.13]{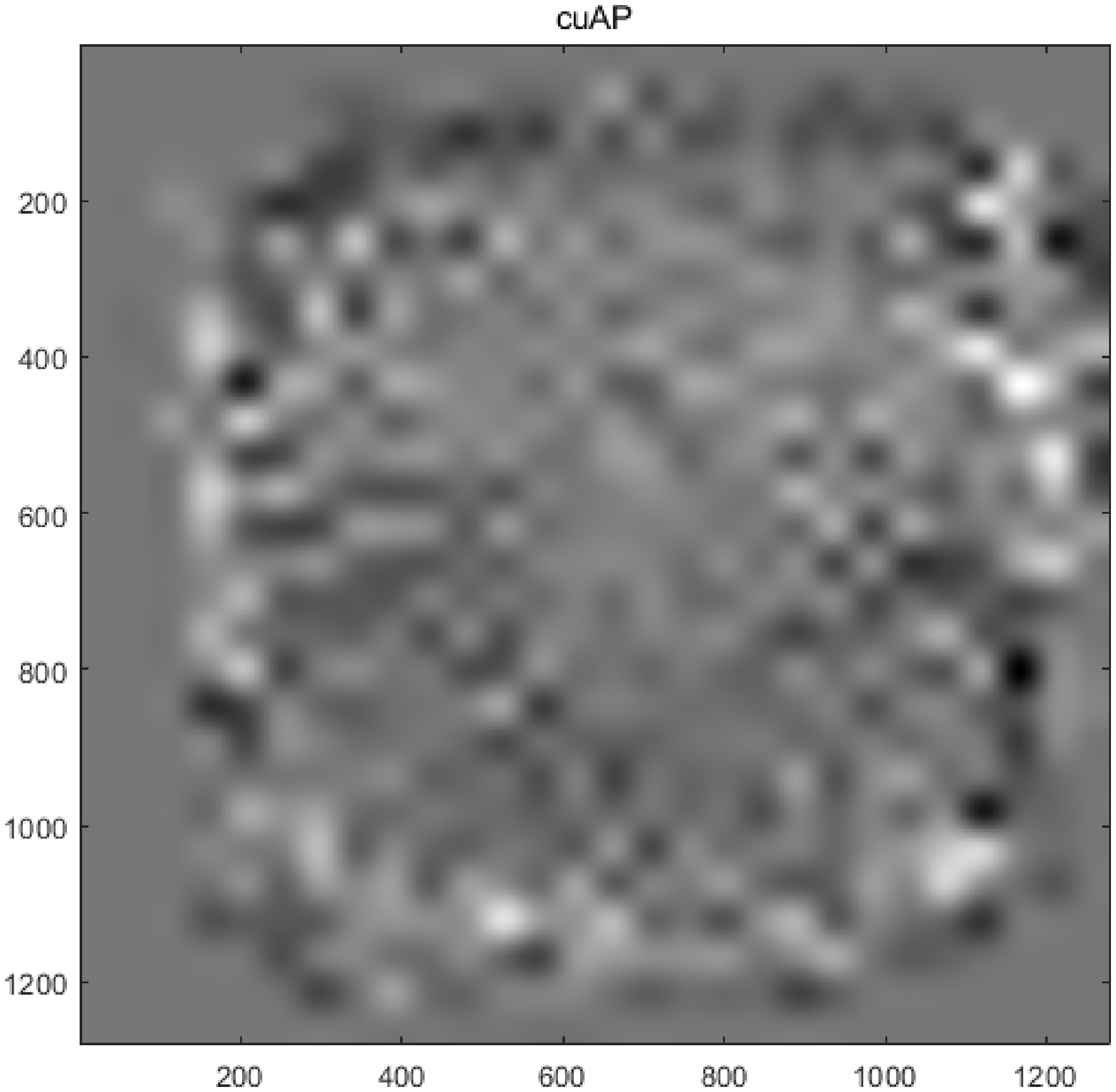}\vspace{20mm}
		\end{minipage}
	\label{label_newfigure9_2}
	}
	\subfloat[The  perturbed  images with class 8.]
	{
		\begin{minipage}[b]{.23\linewidth}
			\centering
			\includegraphics[scale=0.11]{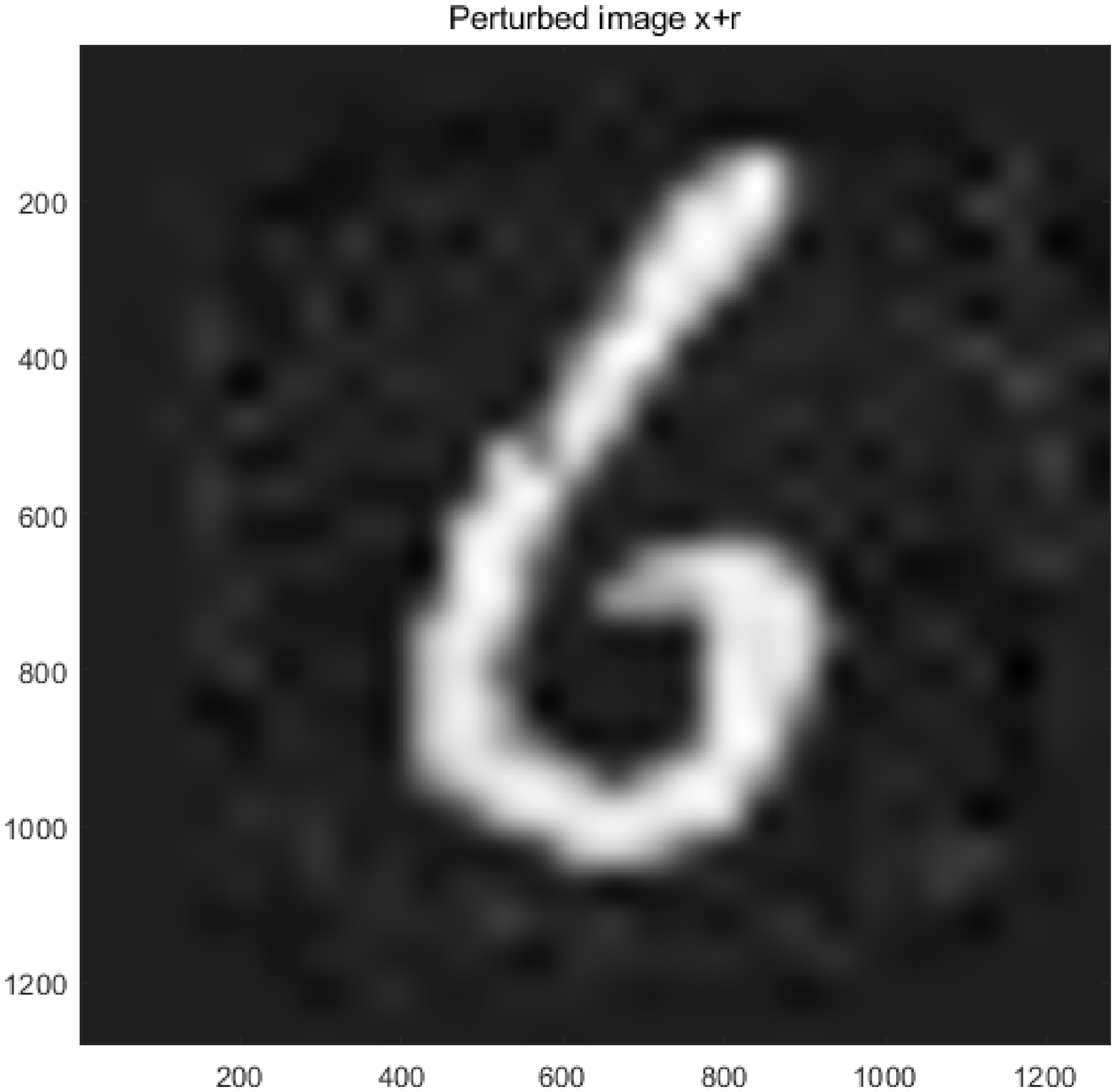} \\
			\includegraphics[scale=0.11]{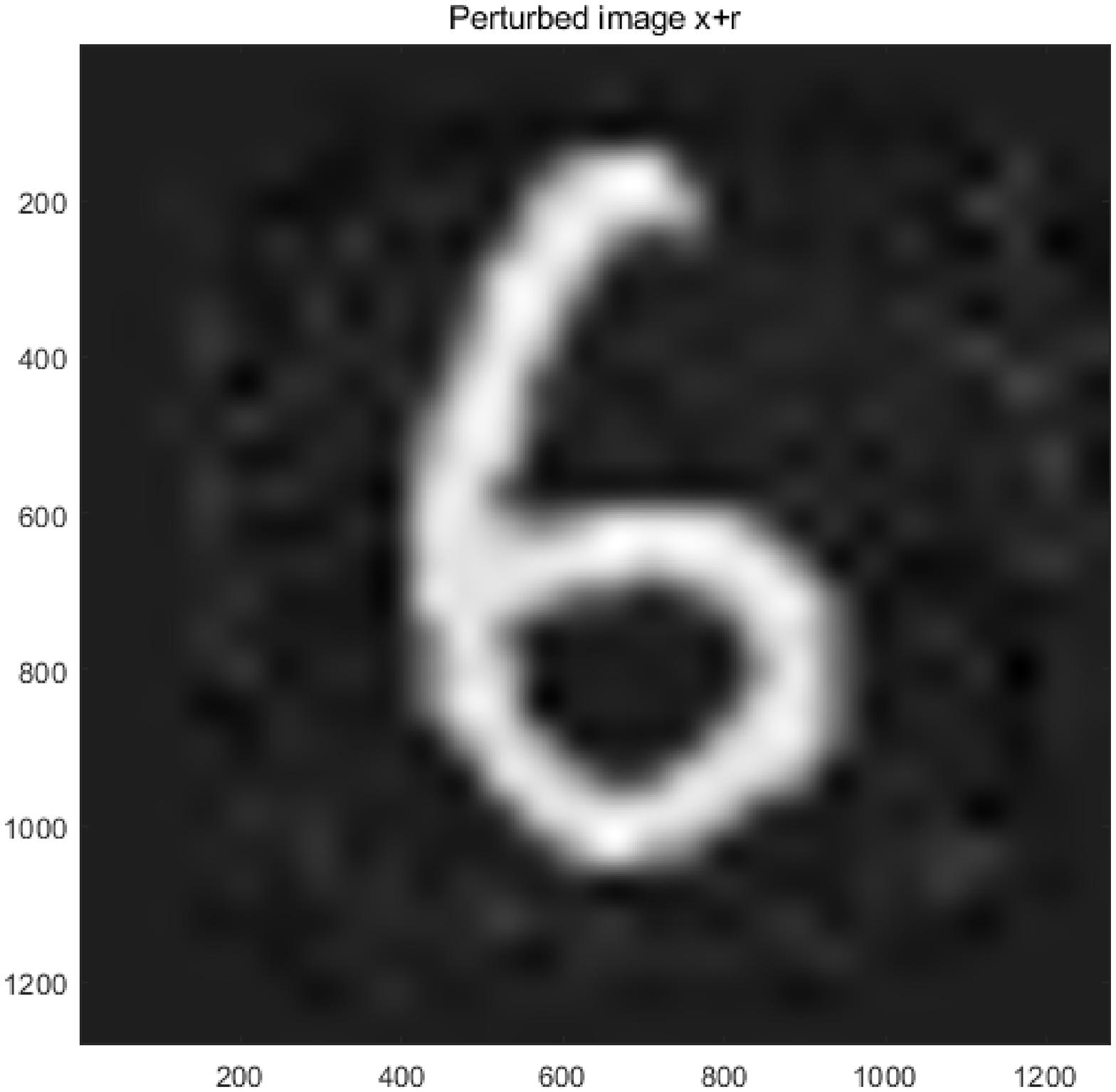} \\
			\includegraphics[scale=0.11]{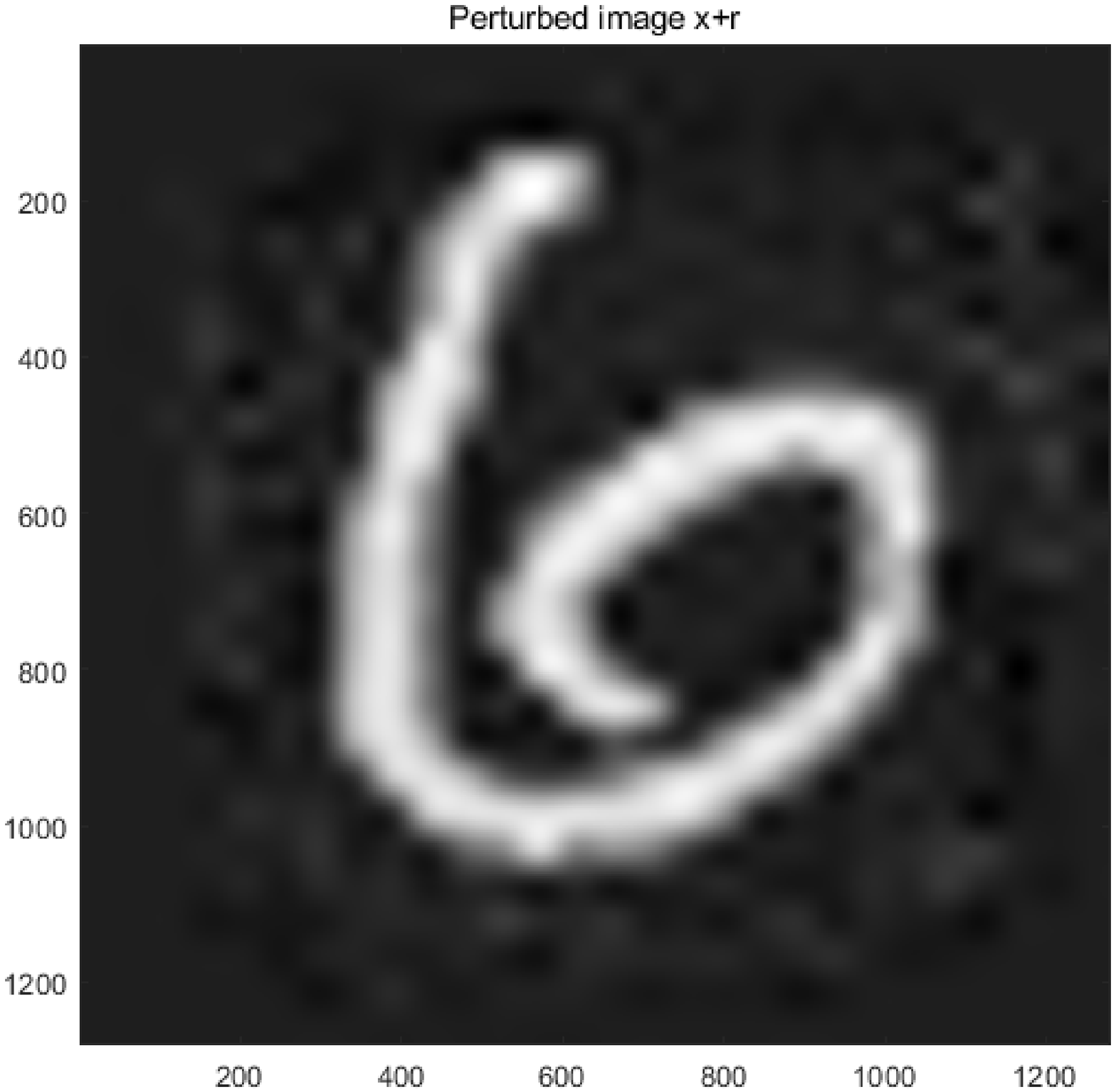}
		\end{minipage}
	\label{label_newfigure9_3}
	}
	\caption{The original images, the  image of cuAP, and the images that have been misclassified after being attacked, when $ \xi=1 $ on the MNIST dataset.}
	\label{label_newfigure9}
\end{figure}

\begin{figure}[h]
	\centering
	\subfloat[The original images with class horse.]
	{
		\begin{minipage}[b]{.23\linewidth}
			\centering
			\includegraphics[scale=0.11]{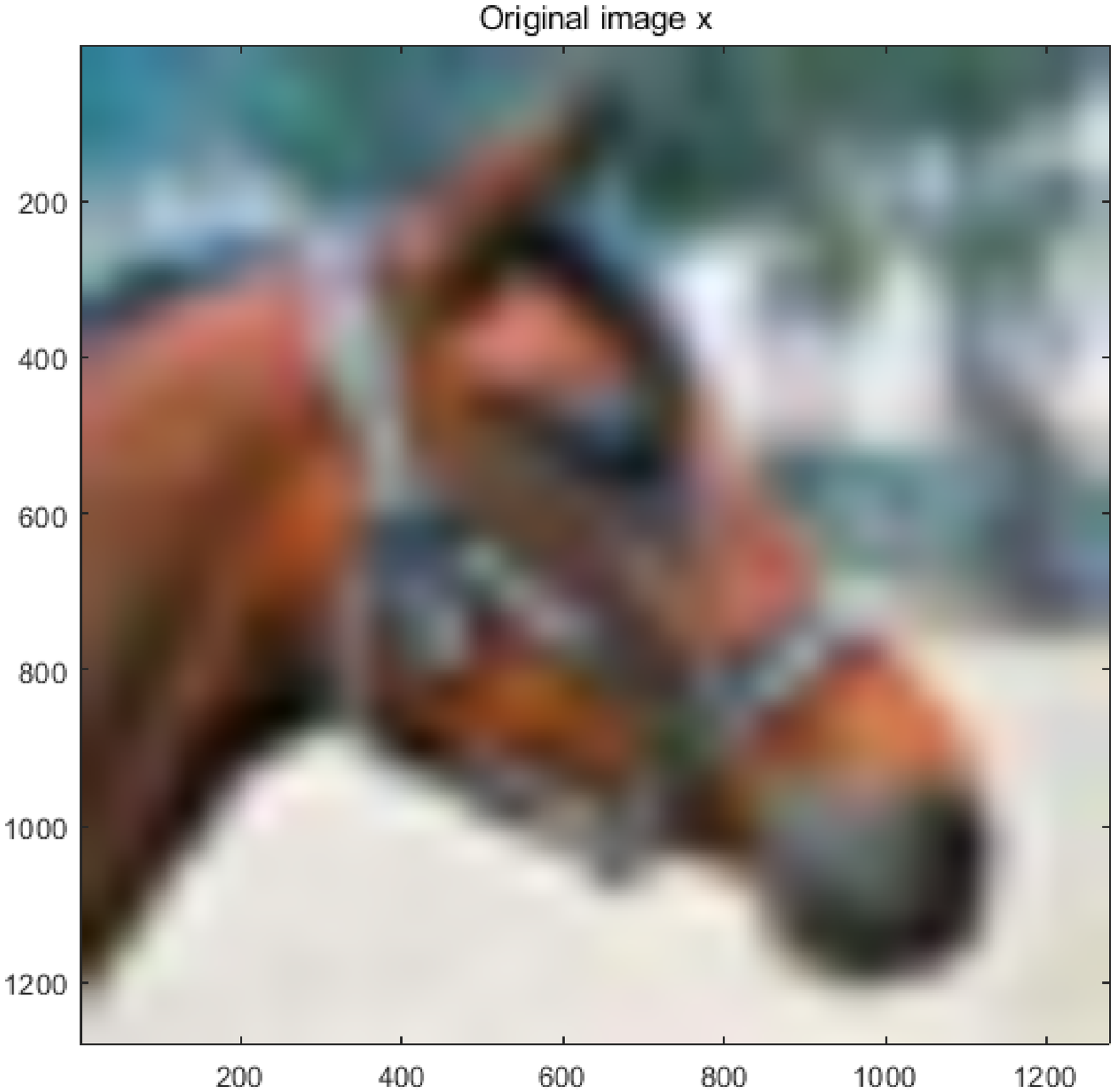} \\
			\includegraphics[scale=0.11]{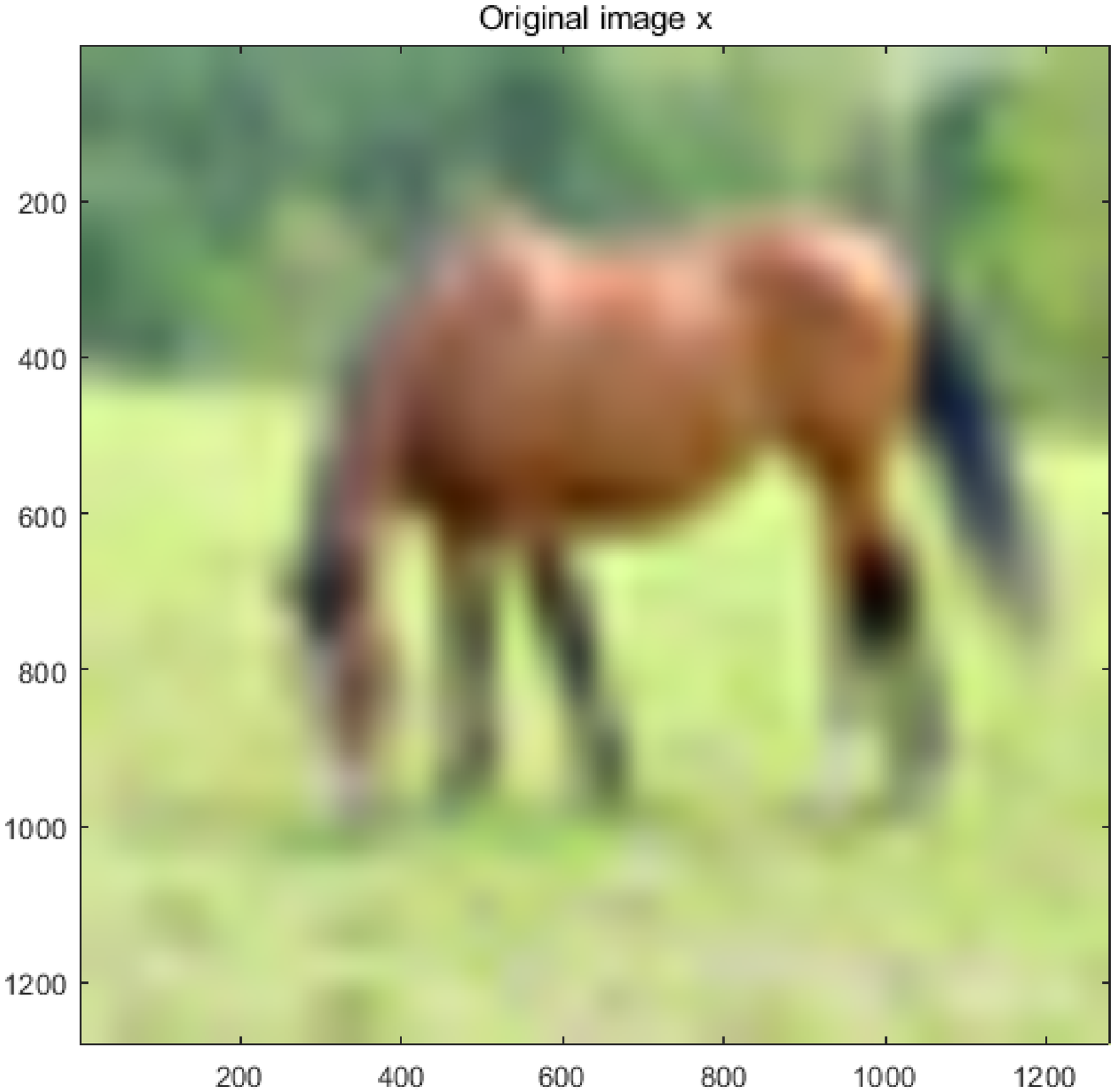} \\
			\includegraphics[scale=0.11]{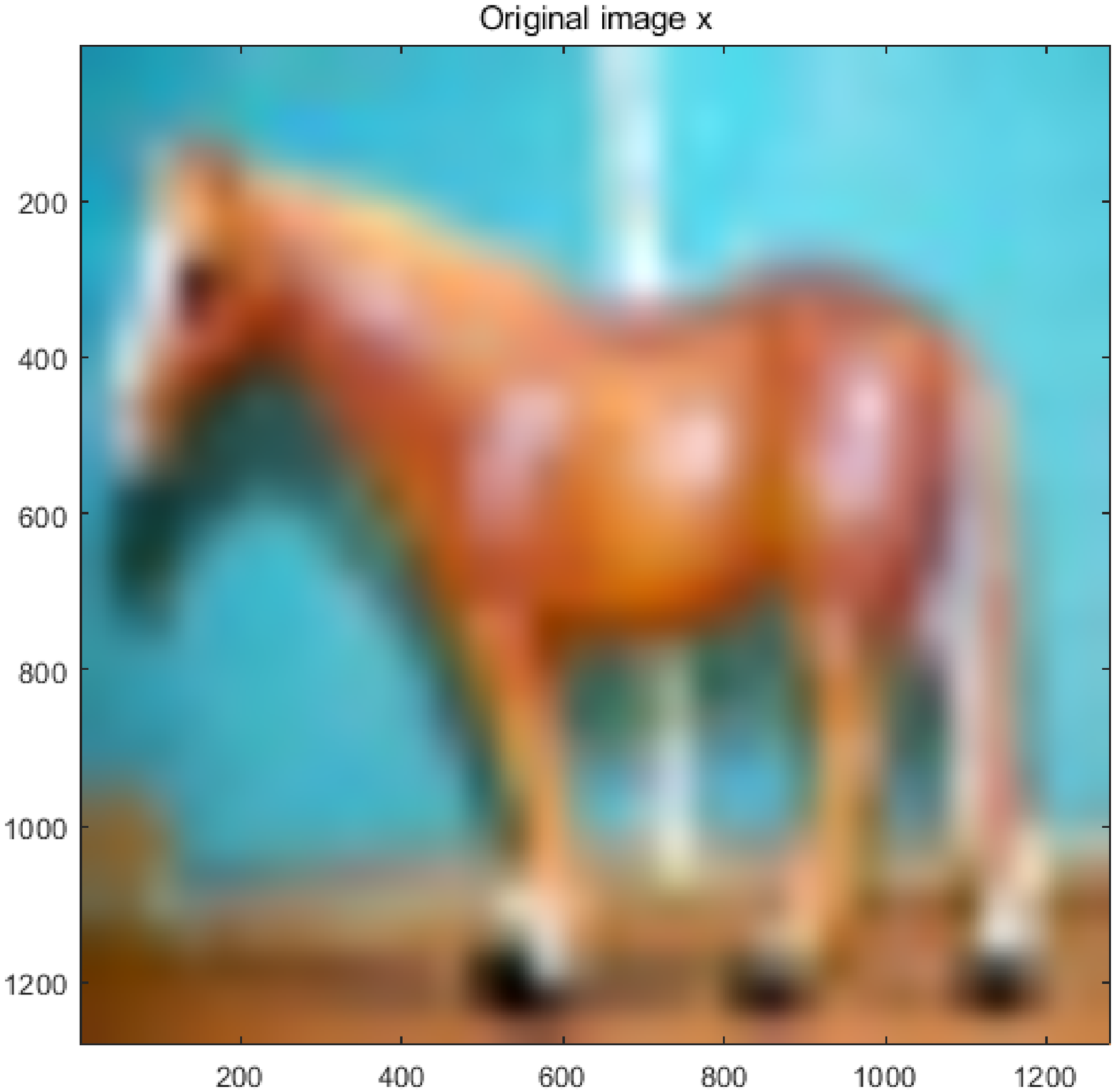}
		\end{minipage}
	\label{label_newfigure10_1}
	}
	\subfloat[The image of the unique cuAP on  the subset of CIFAR-10 dataset with class horse.]
	{
		\begin{minipage}[b]{.3\linewidth}
			\centering
			\includegraphics[scale=0.13]{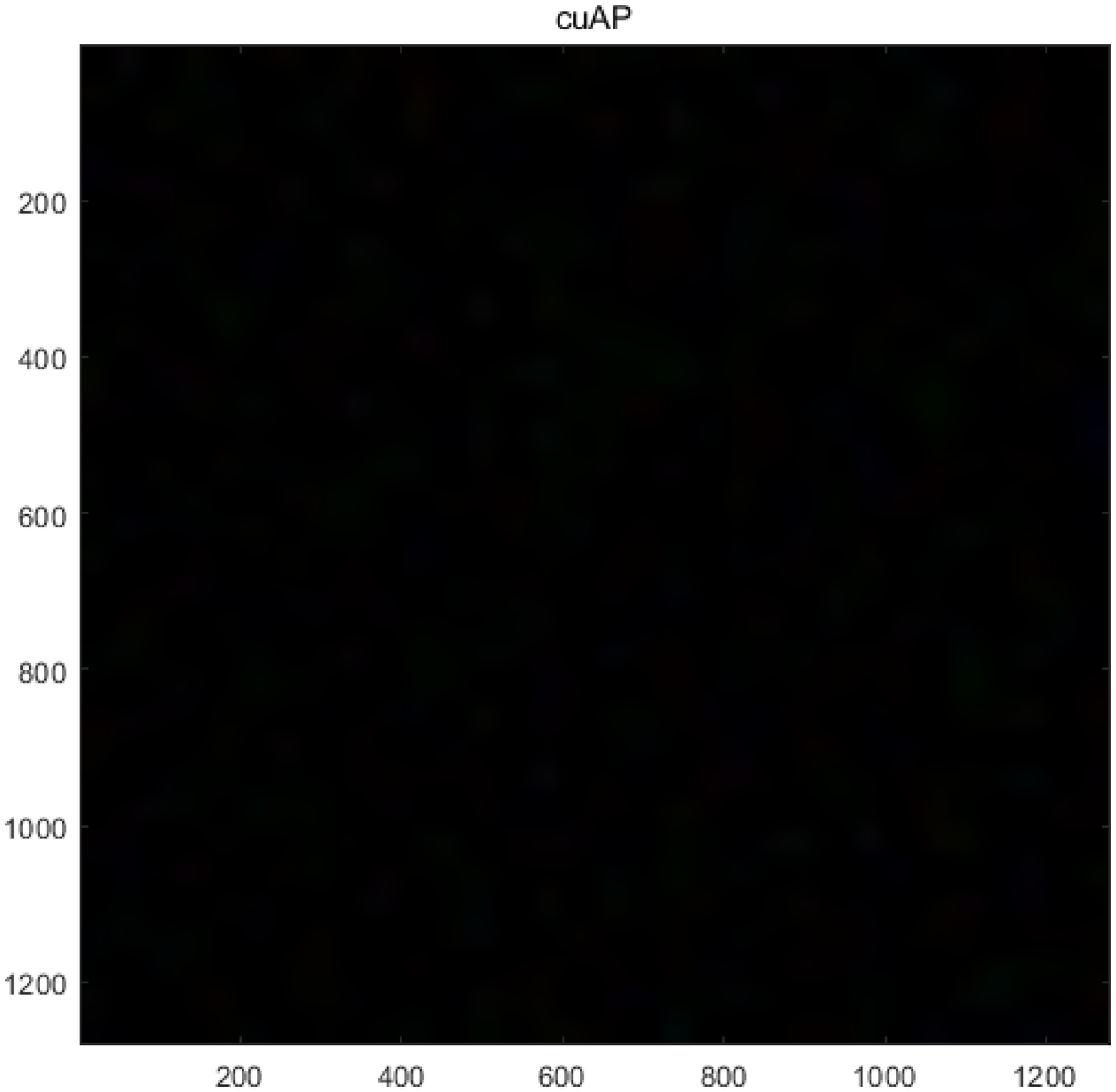}\vspace{20mm}
		\end{minipage}
	\label{label_newfigure10_2}
	}
	\subfloat[The  perturbed  images with class deer.]
	{
		\begin{minipage}[b]{.23\linewidth}
			\centering
			\includegraphics[scale=0.11]{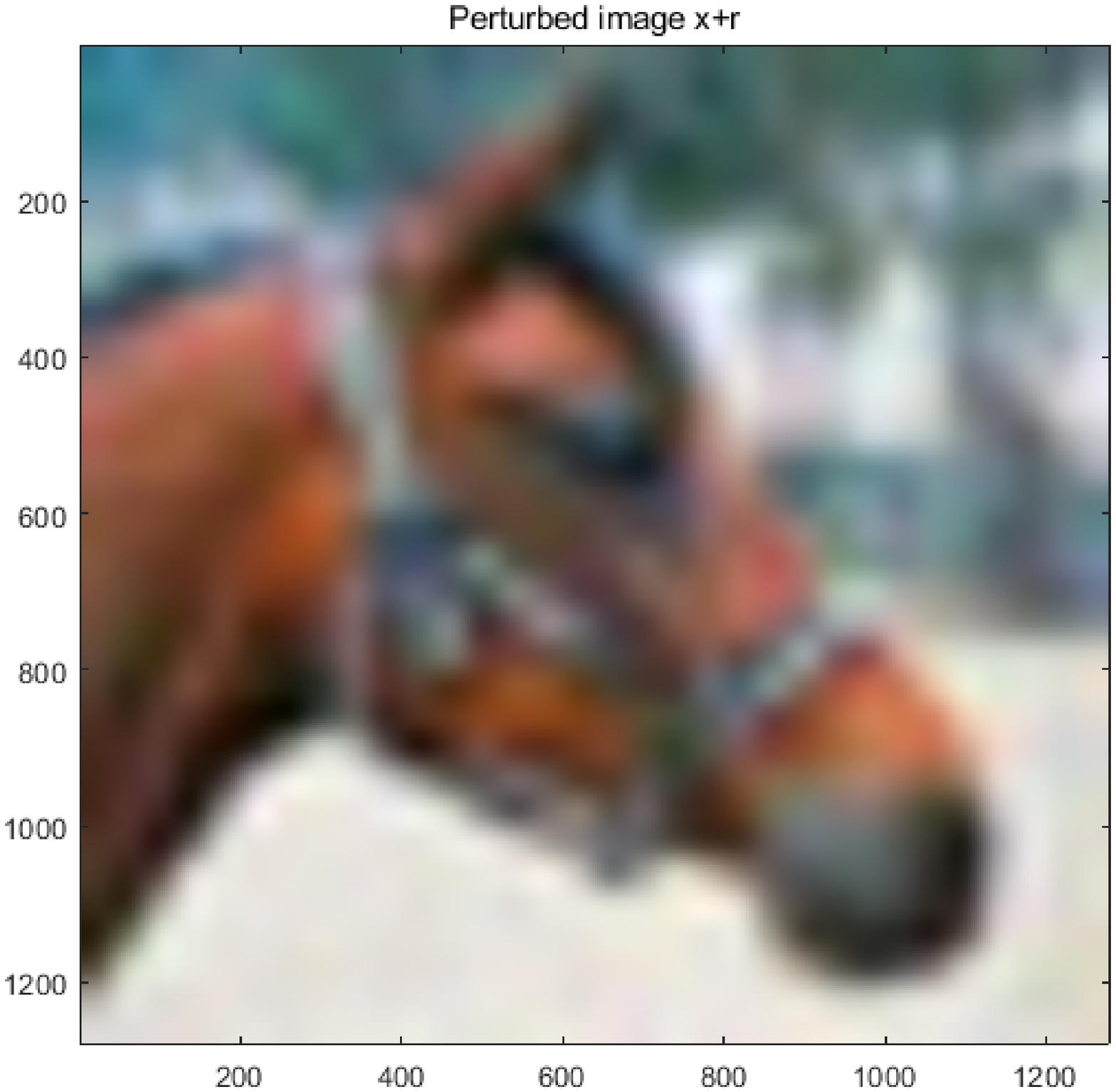} \\
			\includegraphics[scale=0.11]{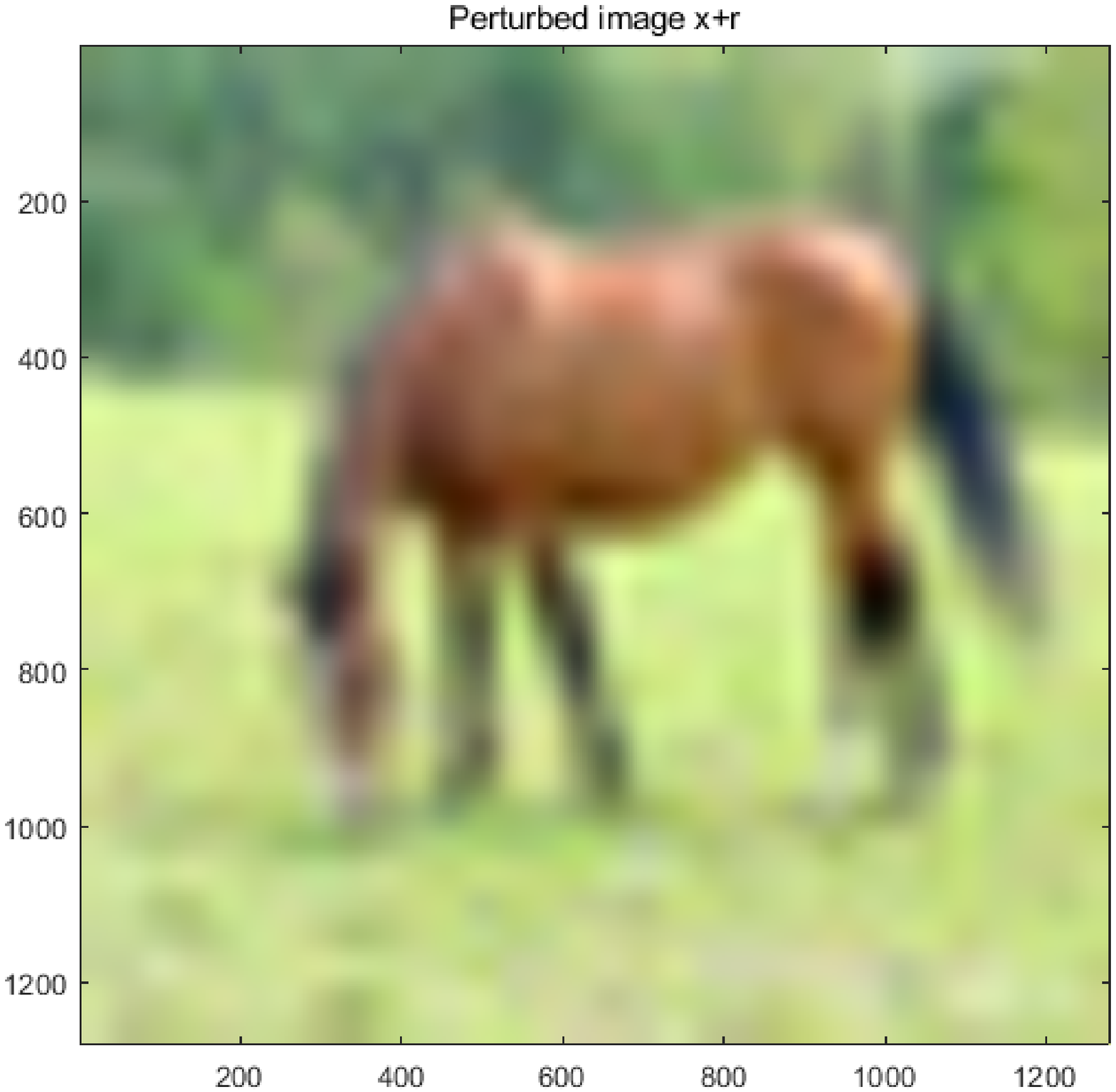} \\
			\includegraphics[scale=0.11]{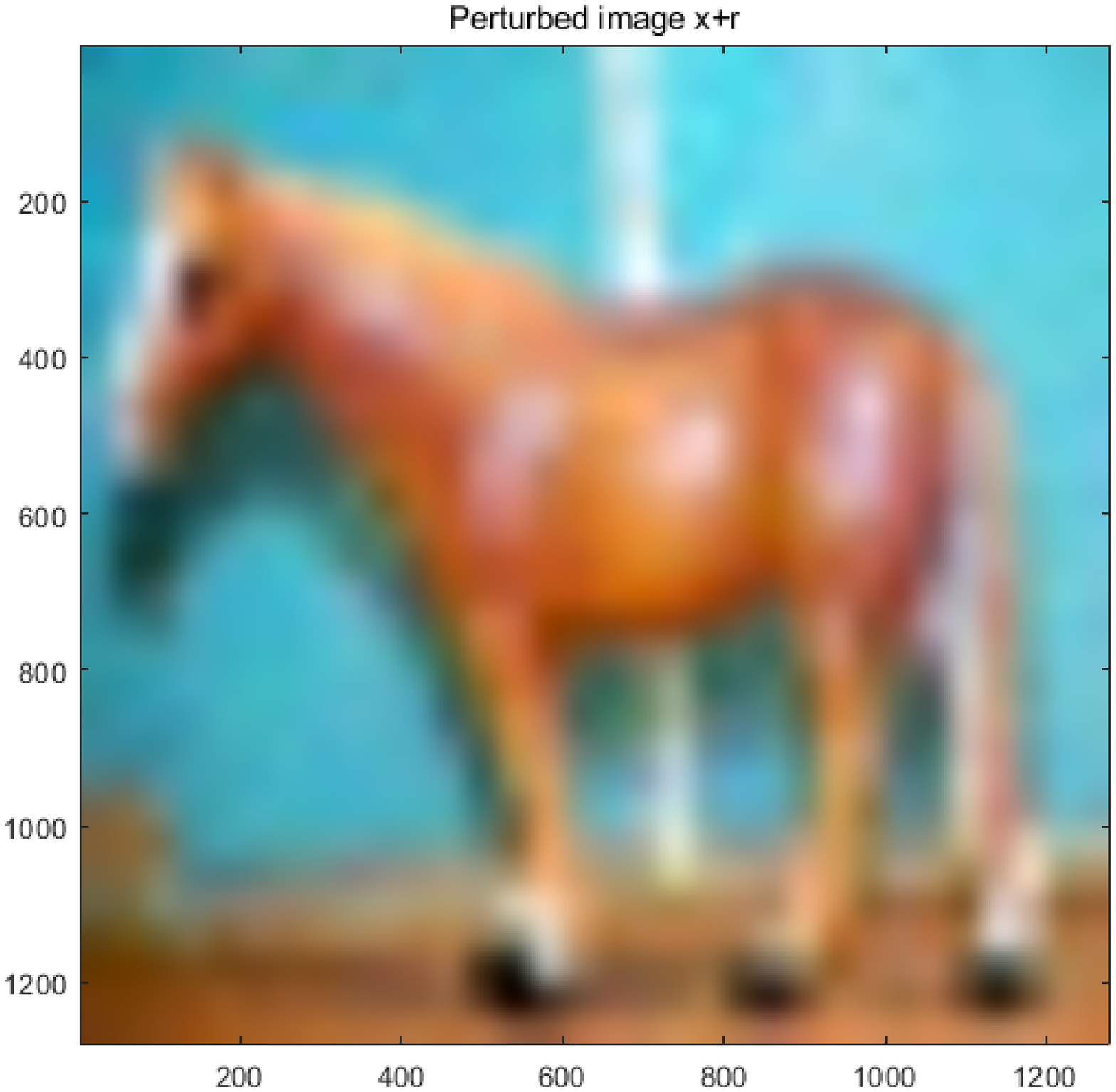}
		\end{minipage}
	\label{label_newfigure10_3}
	}
	\caption{The original images, the  image of cuAP, and the images that have been misclassified after being attacked, when $ \xi=0.5 $ on the CIFAR-10 dataset.}
	\label{label_newfigure10}
\end{figure}

Similarly, in \cref{label_newfigure10}, we give the original images with class horse, the image of  cuAP and the perturbed images with class deer, when $ \xi=0.5 $ on the CIFAR-10  dataset. In this case,  SNR is 35.35. 
And we get that the CPU time to generate cuAP is $2.71s $.
Comparing \cref{label_newfigure9} and \cref{label_newfigure10}, we find that  cuAP generated on the dataset (CIFAR-10) with a larger number of features is  more difficult to be observed. 
%
%Similarly, in Fig.~\ref{label_newfigure10}, we give the original images with class horse, the image of  cuAP and the perturbed images with class deer, when $ \xi=0.5 $ on the CIFAR-10  dataset. And we get that the time to generate cuAP is $2.71s $. 
%Comparing Fig.~\ref{label_newfigure9} and Fig.~\ref{label_newfigure10}, we find that  cuAP generated on the dataset (CIFAR-10) with a larger number of features is 
%more difficult to be observed.

\subsubsection{Numerical experiments of uAP}

%
%uAP扰动更高比例的
%与cuAP相同

\cref{label_newfigure11} shows the relationship that the fooling rate increases with the increase of the size of uAP on the MNIST and CIFAR-10  datasets. In \cref{fig:figure11_0}, we find that when the  norm of uAP on the MNIST reaches 1, the fooling rate is almost $ 90.48\% $ and in \cref{fig:figure11_1}, when the  norm of uAP on the CIFAR-10  reaches 0.5, the fooling rate is almost $ 89.60\% $.  We find that the dataset in \cref{fig:figure11_1}, which has more features, is easier  to be fooled by very small perturbation.
However, with the increase of the  norm  of uAP, the fooling rate will not increase to $ 100\% $.  In reality, we may choose an appropriate size of uAP, when the fooling rate is large enough.

\begin{figure}[bhtp]
	\centering
	\subfloat[The relationship  on the MNIST dataset.]{
		\label{fig:figure11_0} 
		\includegraphics[scale=0.21]{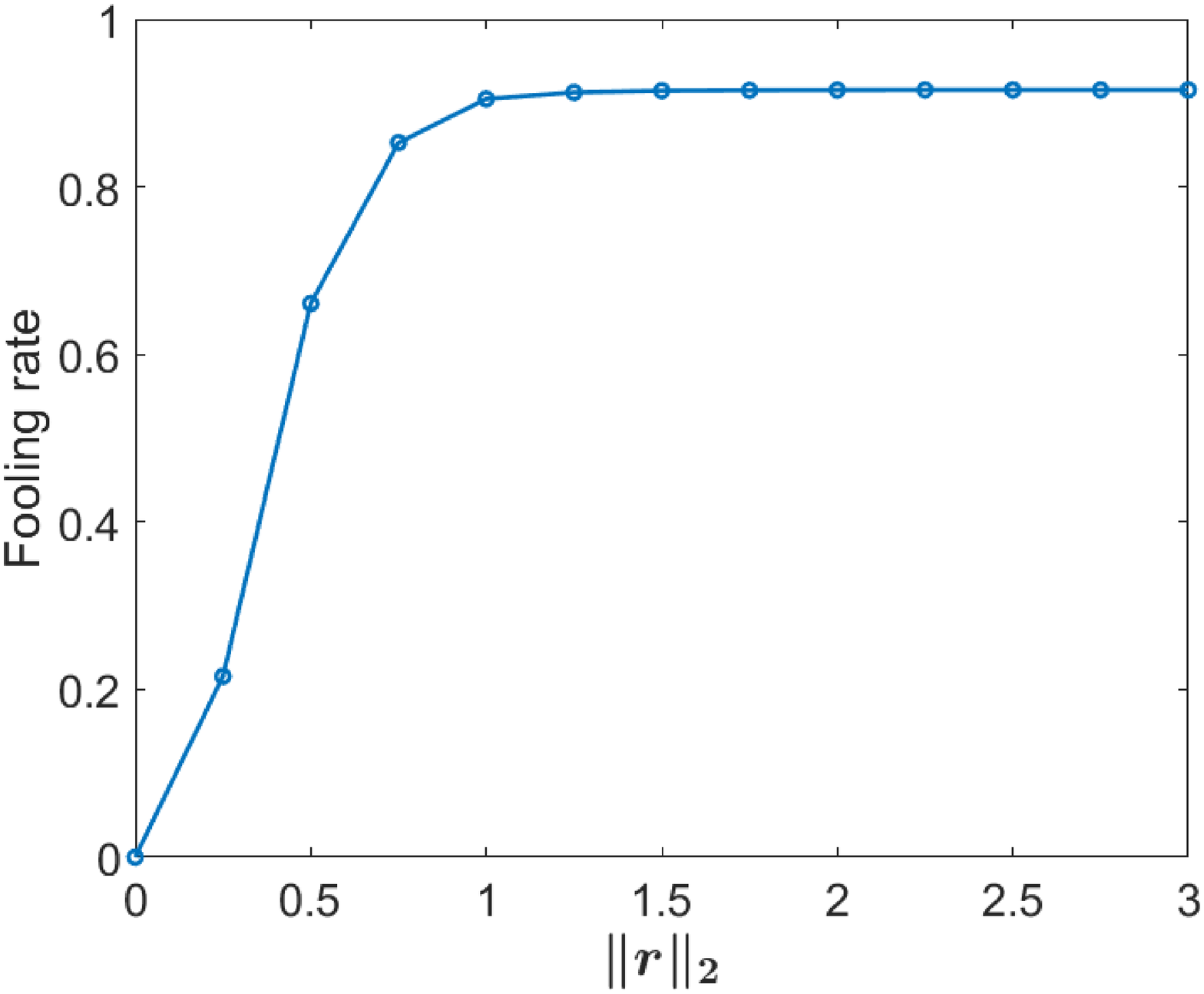}}
	\hspace{0.1in} % 两图片之间的距离
	\subfloat[The relationship  on the CIFAR-10 dataset.]{
		\label{fig:figure11_1} 
		\includegraphics[scale=0.21]{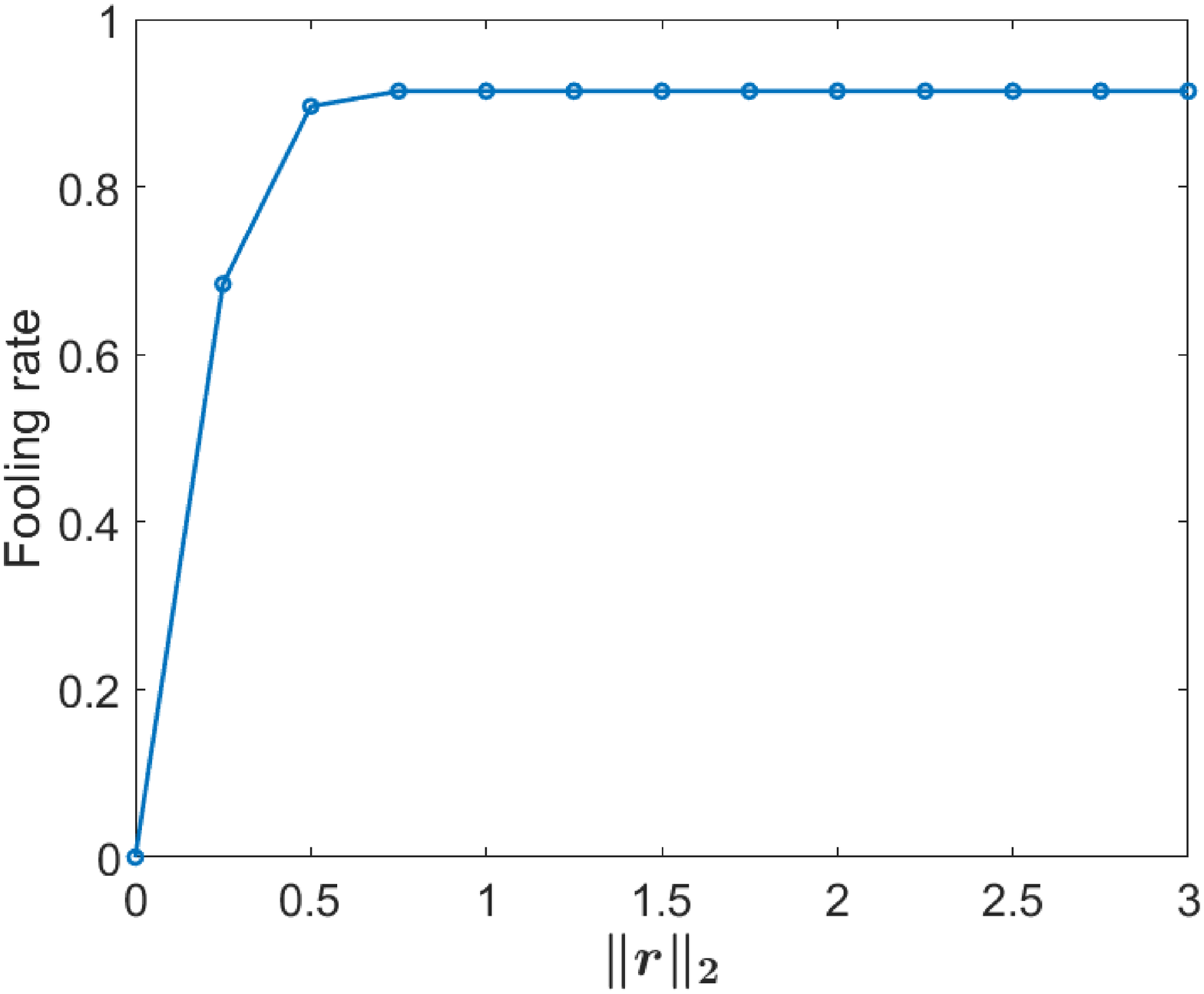}}
	\caption{The relationship between the fooling rate and the norm  of uAP  obtained on the linear multiclass SVMs on the MNIST and  CIFAR-10 datasets.}
	\label{label_newfigure11} 
\end{figure}

In \cref{label_newfigure12}, we give an example to compare the original images, the image of uAP, and the  images of that have been misclassified after being attacked, when $ \xi=1 $ on the MNIST dataset.
In this case,  SNR is 19.22.  
The CPU time to generate uAP is only $3.05 \times 10^{-3}s $.
In MNIST dataset,  because the proportion of data predicted as class 5 is the smallest,  the direction of uAP points to class 5.

\begin{figure}[h]
	\centering
	\includegraphics[scale=0.072]{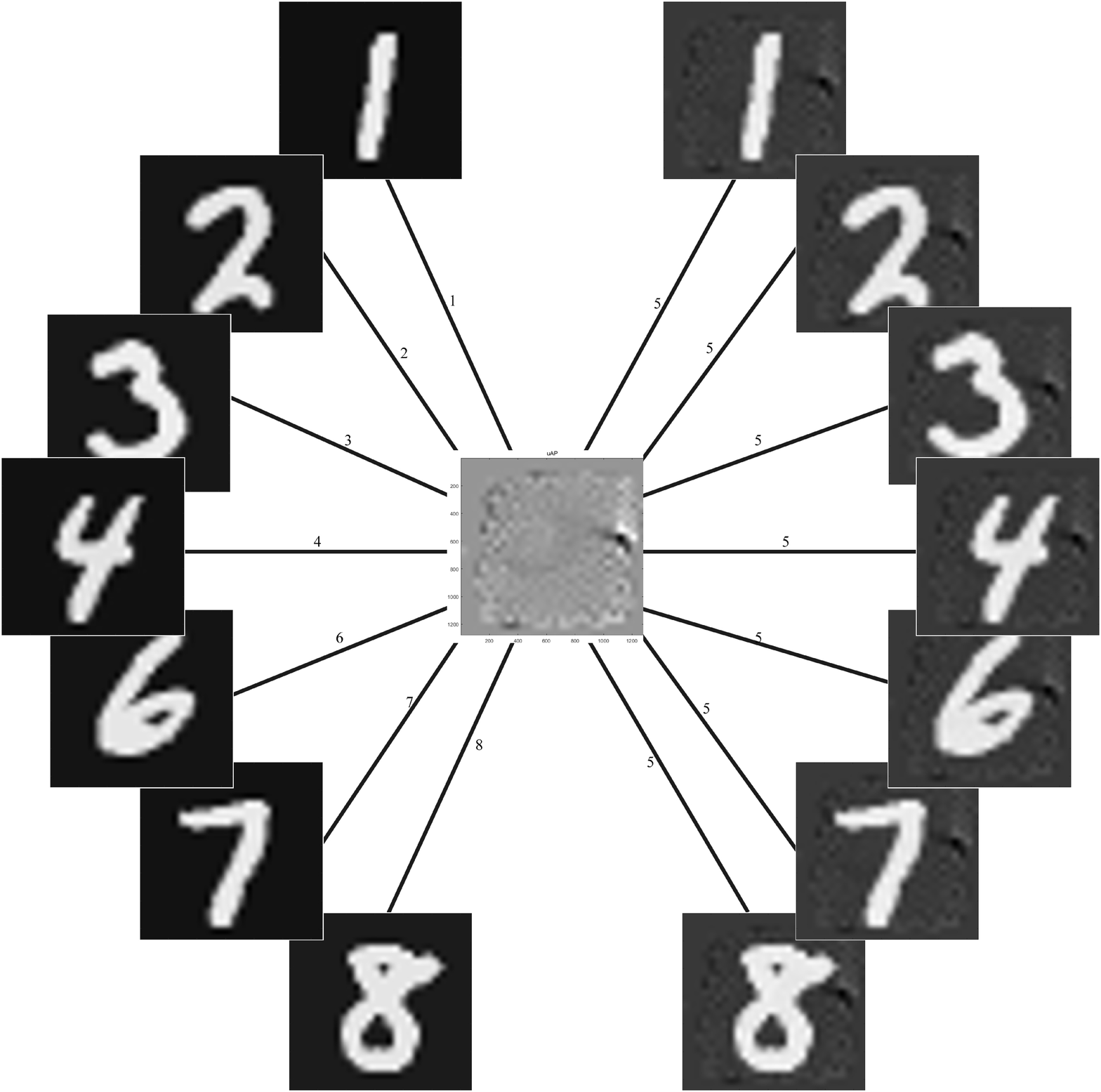}
	\caption{When added to a natural image, uAP causes the image to be misclassified by the linear multiclass SVMs  on the  MNIST dataset. Left images: Original  images. 
		The classes are shown at each arrow. Central image: uAP. Right images: Perturbed images. The estimated classes of the perturbed images are shown at each arrow, which are 5.}
	\label{label_newfigure12} 
\end{figure}

In \cref{label_newfigure13}, the original images are   bird, cat, deer, dog, frog, horse and ship on the CIFAR-10 dataset. When we add uAP ($ \xi=0.5 $) to the original images, the perturbed images are automobiles, but in human eyes, the class of the  perturbed images have not changed. In this case,  SNR is 35.35.  The CPU time to generate uAP is only $8.86 \times 10^{-4}s $.
Comparing \cref{label_newfigure12} and \cref{label_newfigure13}, we find that on the multi\mbox{-}classification model, the data with the same distribution are easy to be attacked by the uAP with unique direction and size, and the perturbation generated on the dataset (CIFAR-10) with a larger number of features is less likely to be observed.

\begin{figure}[H]
	\centering
	\includegraphics[scale=0.073]{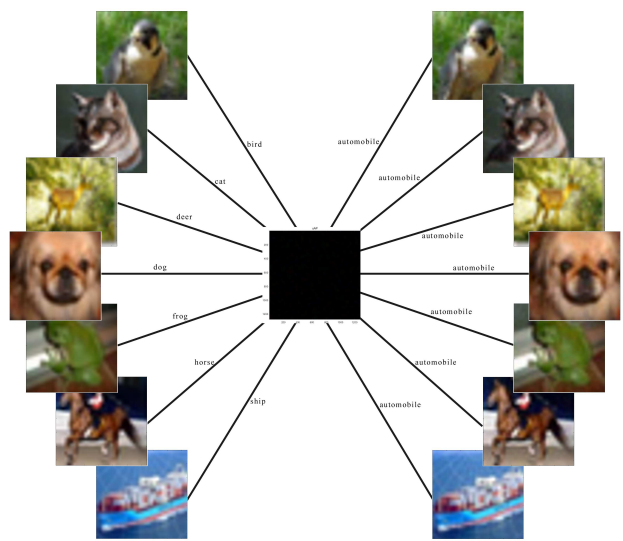}
	\caption{When added to a natural image, uAP causes the image to be misclassified by the linear multiclass SVMs  on the CIFAR-10 dataset. Left images: Original  images. 
		The classes are shown at each arrow. Central image: uAP. Right images: Perturbed images. The estimated classes of the perturbed images are shown at each arrow, which are automobiles.}
	\label{label_newfigure13} 
\end{figure}

Obviously, the multi\mbox{-}classification tasks are more likely to be fooled by smaller perturbation than binary classification tasks, because its classification boundary is closer. Therefore, we speculate that the reason for the existence of  adversarial perturbations is that the flexibility of the classifier is lower than the difficulty of the classification task. As we all know, SVM, neural network and most machine learning models are discriminant models, whose purpose is to find the optimal separating hyperplane  between different class.
Different from the generative model, it mainly focuses on the differences between different class of data, and cannot reflect the characteristics of the data itself, so it will produce some defects.

\section{Conclusions}\label{conclusion}

In this paper, we propose the optimization models for the  adversarial perturbations on classification based on SVMs. 
We derive the explicit solutions for sAP, cuAP and uAP (binary case), and approximate solution for uAP of multi\mbox{-}classification. We also provided the upper bound for the fooling rate of uAP. 
Moreover, we increase the interpretability of adversarial perturbations which shows that the directions of adversarial perturbations are related to the separating  hyperplanes of the model and the size is related to the fooling rate.  Numerical results demonstrate that our proposed approach can generate efficient  adversarial perturbations for models trained by SVMs. It is extremely fast for linear training models since one can avoid iteration process in this case. Numerical results  also show some insights on the potential security vulnerabilities of machine learning models.

\bibliographystyle{siamplain}
\bibliography{references8}

\end{document}